%% file: main.tex
\documentclass[a4paper,10pt,reqno]{article}

\input{parts/header.tex}

\usepackage[style=alphabetic-verb, 
			sorting=nyt,
			url = false, 
			giveninits=true, 
			eprint = true, 
			isbn = false,
			backend = biber,
			maxbibnames=99,
			useprefix=true]{biblatex}
\usepackage[a4paper,hmargin=3cm,bottom=4cm,top=3.5cm,footskip=3\baselineskip]{geometry}
\usepackage{setspace}
\allowdisplaybreaks		

\begin{document}
\doparttoc 
\faketableofcontents 

\input{parts/title_author_abstract_etc.tex}

\input{parts/introduction.tex}

\input{parts/main_results.tex}
\input{parts/related_work.tex}

\input{parts/simulation.tex}

\section{Proofs}\label{section:Proofs}
\input{parts/technical_preliminaries.tex}
\input{parts/proofs_Dtwo.tex}
\input{parts/proofs_Dgreatertwo.tex}

\section{Sharpness of the upper and lower bounds}\label{section:Sharpness}
\input{parts/sharpness_Dtwo.tex}

\input{parts/sharpness_Dgreatertwo.tex}

\input{parts/discussions.tex}

\section*{Methods}

AI and NLP were only used for checking spelling and grammatical errors.

\printbibliography

\newpage
\appendix

\renewcommand \thepart{}
\renewcommand \partname{}
\part{Appendix} 

\parttoc 

\input{parts/main_results_non-unique.tex}

\section{Proofs in the non-unique case}
\input{parts/proofs_Dtwo_non-unique.tex}
\input{parts/proofs_Dgreatertwo_non-unique.tex}

\section{Proofs of technical lemmas}
\input{parts/technical_preliminaries_non-unique.tex}
\input{parts/basic_properties_Dtwo.tex}
\input{parts/basic_properties_Dgreatertwo.tex}

\end{document}

%% file: parts/header.tex
\usepackage[utf8]{inputenc}
\usepackage[english]{babel}

\usepackage{amsmath}
\usepackage{amsthm}
\usepackage{multicol} 			   
\usepackage{multirow}
\usepackage{xcolor}	
\usepackage{cancel}
\usepackage{enumerate}
\usepackage[toc,page]{appendix}
\usepackage{todonotes}
\usepackage{aligned-overset} 
\usepackage{mathtools}
\usepackage{subcaption}
\usepackage{authblk}
\usepackage{minitoc}

\usepackage{hyperref}
\hypersetup{
	bookmarks=true,
	bookmarksnumbered,
	plainpages=false,
	colorlinks=true, 
	linktoc=all,     
	linktocpage,
	linkcolor=red!70!black,  
	citecolor=green!80!black,
	filecolor=magenta,
	hidelinks,
	urlcolor=magenta,
	breaklinks,
	unicode=true,
	hypertexnames=false,
	}

\usepackage[nameinlink,capitalize,noabbrev]{cleveref}

\usepackage{mathrsfs}
\usepackage{bm}					   
\usepackage{dsfont}

\usepackage{amssymb}
\usepackage{pifont}                
\usepackage{esint}				   
\usepackage{soul}

\usepackage{graphicx}
\usepackage{tikz}
\usetikzlibrary{calc,intersections,through,backgrounds}
\usetikzlibrary{arrows.meta}
\usepackage{booktabs}			   
\usepackage{subcaption}
\usepackage{pgfplots}
\usepackage{float}		

\theoremstyle{plain}
\newtheorem{theorem}{Theorem}[section]       

\newtheorem{lemma}[theorem]{Lemma}

\newtheorem{proposition}[theorem]{Proposition}

\theoremstyle{definition}
\newtheorem{definition}[theorem]{Definition}
\newtheorem{remark}[theorem]{Remark}
\newtheorem*{remark*}{Remark}

\newtheorem{assumption}[theorem]{Assumption}

\newcommand*{\tikzmres}[1][]{\tikz[x=1ex, y=1ex, line width = .1ex] \draw[line cap = round ] (0,2) -- (0,0) -- (.8,0); } 

\newcommand*{\fres}[2]{ {\left.\kern-\nulldelimiterspace #1 \vphantom{\big|} \right|_{\kern-1pt #2} }}

\newcommand{\supp}{\operatorname{supp}}
\newcommand{\eps}{\varepsilon}
\newcommand{\dif}{\mathrm{d}}

\newcommand{\abs}[1]{\left\lvert#1\right\rvert}

\newcommand{\norm}[1]{\left\lVert#1\right\rVert}
\newcommand{\vertiii}[1]{{\left\vert\kern-0.45ex\left\vert\kern-0.45ex\left\vert #1 \right\vert\kern-0.45ex\right\vert\kern-0.45ex\right\vert}}   

\DeclareMathOperator{\sign}{sign}
\newcommand{\spanV}{\operatorname{span}}

\DeclareMathOperator*{\argmin}{arg\,min}
\DeclareMathOperator*{\argmax}{arg\,max}

\makeatletter
\newcommand*{\transpose}{{\mathpalette\@transpose{}}}
\newcommand*{\@transpose}[2]{\raisebox{\depth}{$\m@th#1\intercal$}}
\makeatother

\newcommand{\hada}{\odot}

\newcommand{\xa}{x^{\ast}}

\DeclareMathOperator{\arsinh}{arsinh}
\newcommand{\xinf}{x^{\infty}}

\newcommand{\para}{\vert\vert}
\newcommand{\npara}{n^{\para}}
\newcommand{\nperp}{n^{\perp}}

\newcommand{\Lmin}{\Lscr_{\mathrm{min}}}

\newcommand{\Ha}{H_{\alpha}}
\newcommand{\Hb}{H_{\alpha}}

\newcommand{\na}{n^{\ast}}
\newcommand{\QQ}{Q^D_{\alpha}}
\newcommand{\qq}{q_D}
\newcommand{\hh}{h_D}
\newcommand{\ma}{m^{\ast}}
\newcommand{\ga}{g^{\ast}}
\newcommand{\gD}{g_D}
\newcommand{\GD}{G_{\alpha}^D}

\newcommand{\s}{\sigma}
\newcommand{\SP}{\mathcal{S}}
\newcommand{\SC}{\SP^c}
\newcommand{\Tan}{\Tcal}
\newcommand{\Nor}{\Ncal}
\newcommand{\mg}{\min_{i\in \SP} \abs{\ga_i}}
\newcommand{\overleq}[1]{\overset{#1}{\le}}
\newcommand{\overgeq}[1]{\overset{#1}{\ge}}
\newcommand{\overeq}[1]{\overset{#1}{=}}
\newcommand{\xinfalpha}{\xinf(\alpha)}

    \newcommand{\Tcal}{\mathcal{T}}                    \newcommand{\Ncal}{\mathcal{N}}

                  \newcommand{\Lscr}{\mathscr{L}}       

\newcommand{\R}{\mathbb{R}} \newcommand{\N}{\mathbb{N}}

%% file: parts/title_author_abstract_etc.tex

\title{Linear regression with overparameterized linear neural networks:
Tight upper and lower bounds for implicit 
$\ell^1$-regularization}

\author[1]{Hannes Matt}
\author[1]{Dominik St\"{o}ger}
\affil[1]{Mathematical Institute for Machine Learning and Data Science
\newline 
KU Eichstätt--Ingolstadt

}





\date{\today}

\maketitle

\begin{abstract}
    Modern machine learning models are often trained in a setting
    where the number of parameters exceeds the number of training samples.
    To understand the implicit bias of gradient descent in such overparameterized models,
    prior work has studied diagonal linear neural networks
    in the regression setting.
    These studies have demonstrated that gradient descent, 
    when initialized with small weights, 
    tends to favor solutions with minimal $\ell^1$-norm
    —a phenomenon referred to as implicit regularization.
    In this paper, we investigate implicit regularization in diagonal linear neural networks
    of depth $D\ge 2$ 
    for overparameterized linear regression problems. 
    We focus on analyzing the approximation error between the limit point of 
    gradient flow trajectories and the solution to the $\ell^1$-minimization problem.
    Our analysis precisely characterizes how the approximation error
    depends
    on the scale of initialization $\alpha$
    by establishing tight upper and lower bounds on the approximation error.
    Our results highlight a qualitative difference between
    networks of different depth $D$:
    for $D \ge 3$, 
    the error decreases linearly with $\alpha$,
    whereas for $D=2$,
    it decreases at rate $\alpha^{1-\varrho}$.
    Here, 
    the parameter $\varrho \in [0,1)$ can be explicitly characterized
    and is closely related
    to null space property constants
    studied in the sparse recovery literature. 
    We demonstrate the asymptotic tightness of our bounds through explicit examples. 
    Numerical experiments corroborate our theoretical findings
    and suggest that deeper networks, i.e., $D \ge 3$, 
    may lead to better generalization,
    particularly for realistic initialization scales
    and in noisy regimes.
\end{abstract}

%% file: parts/introduction.tex
\section{Introduction}

Modern neural networks are often trained in an overparameterized setting,
where the number of parameters significantly exceeds the number of data points.
Despite their complexity, these models exhibit strong generalization properties, 
even when the training data is perfectly interpolated 
and no regularization is applied \cite{zhang2021understanding}.
At first glance, this may seem to contradict conventional statistical wisdom, 
which suggests that overparameterized models are prone to overfitting.
Indeed, due to the high capacity of these overparameterized models
there are infinitely many minimizers of the risk function 
that perfectly interpolate the training data,
many of which may generalize poorly.
As a consequence, the performance of the trained model is not determined solely by the training risk 
but also depends on the choice of the training algorithm.
Implicit regularization refers to the hypothesis
that the training algorithm itself induces
a bias towards solutions that minimize a certain complexity parameter. 
Indeed, practitioners are well aware that the generalization error
of the trained model depends on the choice of the hyperparameters during training,
such as step size, batch size, choice of the optimizer, network architecture, 
or initialization. 

While in the context of neural networks 
the precise nature of the implicit regularization phenomenon remains to be 
fully understood,
significant progress has been made in recent years
toward understanding 
the effects of implicit regularization through gradient descent and related algorithms
in simplified models
such as diagonal linear neural networks or low-rank matrix recovery with factorized gradient descent.
For instance, in diagonal neural networks, 
gradient flow and gradient descent
with sufficiently small initialization
have been shown to bias the optimization process toward sparse solutions
\cite{vaskevicius2019implicit,woodworth2020kernel,amid2020reparameterizing,amid2020winnowing,yun2021a,azulay2021implicit,li2022implicit,chou2023less}.
In the context of low-rank matrix recovery,
factorized gradient descent with small random initialization
has been demonstrated to favor low-rank solutions
in overparameterized matrix recovery problems
\cite{gunasekar2017implicit,li2018algorithmic,arora2019implicit,litowards,razin2020implicit,stoger2021small,soltanolkotabi2023implicit,jin2023understanding,wind2023asymmetric,chou2024deep,ma2024convergence}.

In this paper, we focus on diagonal linear neural networks with Hadamard reparameterization.
Specifically, we consider the linear regression problem
\begin{equation}\label{equ:lossfunction}
	\mathcal{L} \left( x \right)
	=
	\big\Vert y - A x  \big\Vert_{\ell^2}^2,
\end{equation} 
where $y \in \mathbb{R}^N$ 
and $A \in \mathbb{R}^{N \times d}  $.
We assume the model is overparameterized, i.e., $ d \gg N$.
The vector $x \in \R^d$ is reparameterized as 
\begin{equation*}
x (u,v) := u^{\odot D} - v^{\odot D},
\end{equation*}
where $u,v \in \mathbb{R}^d$ and $D \ge 1$ is a natural number.
Here, $x^{\odot D}$ denotes the Hadamard (element-wise) product of $x$
with itself $D$-times,
i.e., $ \left(x^{\odot D}\right)_i:= (x_i)^D $ for each index $i \in [d]$.
The function $x(u,v)$ is referred to as a diagonal neural network with $D$ layers,
as discussed in more detail in \cite{chou2023less}.
By substituting $x(u,v)$ into the original objective function,
we obtain the reparameterized objective function
\begin{equation}\label{equ:reparameterized}
	\widetilde{\mathcal{L}} \left( u , v \right)
	:=
	\Big\Vert y - A \left( u^{\odot D}-v^{\odot D} \right) \Big\Vert_{\ell^2}^2.
\end{equation}
It has been shown for $D=1$
that gradient descent on the reparameterized objective function \eqref{equ:reparameterized}
converges towards the solution
which is closest to the initialization with respect to the $\ell^2$-norm.
In contrast, for $D \ge 2$ 
it has been demonstrated that
gradient descent has an implicit bias towards the solution with smallest $\ell^1$-norm.

What makes the diagonal linear network model appealing for theoretical studies
is that the implicit regularization effect can be rigorously expressed in terms of a Bregman
divergence,
where we recall that for a strictly convex function $F: \R^d \rightarrow \R$
the Bregman divergence with potential function $F$ is defined as
\begin{equation*}
	D_F (p,q)
	:=
	F(p) - F(q) - \langle \nabla F(q), p-q \rangle.
\end{equation*}
To formalize the connection between diagonal networks
and the Bregman divergence,
we consider the idealized scenario of gradient flow,
i.e., the continuous-time limit of gradient descent when the step size approaches zero. 
In this setting,
the gradient flow trajectories $u,v : [0,\infty) \rightarrow \R^d $ are 
defined as the solutions of the following ordinary differential equations:
\begin{equation*}
	\frac{d}{dt} u (t) = - \left( \nabla_{u} \widetilde{\mathcal{L}} \right) \left( u (t), v (t) \right), \quad
	\frac{d}{dt} v (t) = - \left( \nabla_{v} \widetilde{\mathcal{L}} \right) \left( u (t), v (t) \right),
    \quad u(0)=u_0, \quad v(0)=v_0
\end{equation*}
with the initial conditions $u(0)=u_0$ and $ v(0)= v_0$
for a given initialization $u_0,v_0 \in \R^d$.

This setup allows us to define $x: [0, \infty) \rightarrow \R^d$ as 
\begin{equation}\label{equ:trajectory}
x(t) 
:= 
x\left(u \left(t \right),v \left(t \right)\right)
=
u(t)^{\odot D} - v(t)^{\odot D}.
\end{equation}
To keep the presentation concise, 
we now assume that $u(0)=v(0)= \alpha^{1/D} \cdot \mathbf{1}$.
Here $\alpha >0$ is referred to as the \textit{scale of initialization}
and $\mathbf{1} \in \R^{d}$ denotes the vector 
in which each entry is equal to one.
This assumption implies that $x(0)=0$.
The following result then characterizes the limit point of the gradient flow trajectory
as a minimizer of a constrained optimization problem involving the Bregman divergence.
\begin{proposition}[see, e.g., Theorem 3.8 and Theorem 4.8 in \cite{li2022implicit}]\label[proposition]{thm:introduction}
	Let $D\ge 2$ be an integer 
	and let $\alpha >0$ represent the scale of initialization.
	Assume that $u_0=v_0= \alpha^{1/D} \cdot \mathbf{1}$.
	Furthermore, assume that the gradient flow trajectory $x: [0,\infty) \rightarrow \R$,
	as defined in \eqref{equ:trajectory},
	converges to a limit point $\xinfalpha = \lim_{t \rightarrow \infty} x(t) $ 
	with $A \xinfalpha =y$.
	The limit point $\xinfalpha$ can then be uniquely characterized as
	\begin{equation}\label{bregmanminimizers}
		\xinfalpha 
		= \underset{x \in \R^d: A x = y}{\text{arg min}} \ D_{F_{\alpha,D}} \left(x, 0 \right),
	\end{equation}
	where the potential function $F_{\alpha,D}$ of the Bregman divergence $D_{F_{\alpha,D}}$
	depends only on the depth $D$ and on the scale of initialization $\alpha$.
We note that
the potential function $F_{\alpha,D}$ can be expressed analytically, see \cref{section:MainResults}.
\end{proposition}

Although the Bregman divergence $D_{F_{\alpha,D}}$ is in general not a metric,
it can be interpreted as a measure of the distance between two points.
Thus, \cref{thm:introduction} shows 
that the limit point of the gradient flow trajectory
can be characterized as the solution $x$ to the equation $Ax=y$ 
that is closest to the 
initialization $x(0)$
with respect to the Bregman divergence.
It is important to note
that this relationship can be extended to general initializations
$u_0 , v_0 \in \R^d$ as well.

Moreover, while this paper does not focus on convergence properties of gradient flow 
we note that the convergence of the gradient flow trajectory to a minimizer of
$\mathcal{L}$, or equivalently $\widetilde{\mathcal{L}}$,
was proven in \cite{chou2023less} 
under the assumption that a solution of the equation
$Ax=y$ exists.

Equation \eqref{bregmanminimizers}
can be used as a foundation
for analyzing the implicit regularization effect of gradient flow
towards sparse solutions.
Indeed, using this equation 
previous work 
\cite{chou2023less,wind2023implicit}
established 
that for $D \ge 2$ it holds that
\begin{equation*}
	\lim_{\alpha \rightarrow 0} \norm{ \xinf (\alpha)}_{\ell^1}
	= 
	\underset{x: Ax=y}{\text{min}} \ \Vert x \Vert_{\ell^1}.
\end{equation*}
This result justifies the implicit bias of gradient flow towards sparse solutions.
While the assumption of gradient flow simplifies the problem
and is unrealistic in practice,
recent results have extended these findings to gradient descent \cite{wang2023linear}
showing that the limit point can also be connected to the 
Bregman divergence. 
Moreover, several algorithmic modifications
inspired by deep learning practice
--
such as weight normalization \cite{chou2023robust},
stochastic label noise \cite{pesme2021implicit},
and large step size combined with stochastic gradient descent (SGD) \cite{even2023s}--
have been proposed and studied for the diagonal linear neural network model \cref{equ:reparameterized}.
In particular, 
it has been shown that the implicit regularization effect
of these algorithmic modifications can also be characterized using the Bregman divergence,
and that they lead to a \textit{smaller effective initialization}.
For instance,
in the case of weight normalization \cite{chou2023less}
the parameter $\alpha$ in the Bregman divergence in \cref{bregmanminimizers}
is replaced by a new parameter $\widetilde{\alpha}$
with $ \widetilde{\alpha} \ll \alpha $.
As a result, these modifications further
strengthen the implicit bias towards the $\ell^1$-minimizer.

In this paper, we aim to understand
how the approximation error $ \norm{\xinf(\alpha) - \ga}_{\ell^p} $,
where $\ga$ denotes a solution of the $\ell^1$-optimization problem 
$\underset{x: Ax=y}{\text{min}} \ \Vert x \Vert_{\ell^1}$ 
for $p \in \left\{ 1; \infty \right\}$,
depends on the scale of initialization $\alpha$.
Although previous works have established upper bounds for the approximation error 
for $D=2$ \cite{woodworth2020kernel,chou2023less,wind2023implicit} and $D \ge 3$ \cite{chou2023less,wind2023implicit}
these bounds are often either pessimistic when compared to numerical evidence 
or involve unspecified constants.
Furthermore, to the best of our knowledge,
no lower bounds have been established in previous work.
As a result,
it was unclear before this paper
how different depths $D$ of the diagonal linear network
precisely influence the implicit regularization.
Moreover, 
the absence of lower bounds makes it hard to compare the impact
of various algorithmic modifications,
such as weight normalization or stochastic label noise,
on the implicit regularization 
towards the $\ell^1$-minimizer.

\paragraph{Our contribution:}
In this paper, 
we prove tight upper and lower bounds on the approximation error $\norm{\xinf (\alpha) - \ga}_{\ell^1}$,
assuming that the $\ell^1$-minimization problem $\underset{x: Ax=y}{\min} \norm{x}_{\ell^1}$ admits 
a unique solution $\ga$.
While our upper bounds improve upon previous work,
no lower bounds have been established in the literature thus far.
Using these bounds,
we precisely characterize the convergence rate of $\norm{\xinf (\alpha) - \ga}_{\ell^1}$ 
as the scale of initialization $\alpha$ approaches zero.
In particular, we show that for $D \ge 3$ the convergence rate is proportional to $\alpha$,
whereas for $D=2$, the convergence rate is proportional to $\alpha^{1-\varrho}$.
We explicitly characterize the constant $\varrho \in (0,1)$, which depends on $A$ and $y$,
and show that it is closely related to \textit{null space property constants} 
studied in the sparse recovery literature \cite{CohenNullSpace2009}, see also \cite{foucart_MathematicalIntroductionCompressive_2013}.
Furthermore,
by constructing explicit examples,
 we demonstrate 
that our upper and lower bounds 
are optimal in an asymptotic sense
and thus cannot be improved.

Inspired by our theoretical findings,
we conduct numerical experiments in a sparse recovery setting,
both with and without noise.
In the noiseless scenario,
we observe that
the approximation error decreases
with rate $\alpha^{1-\varrho}$ in the case $D=2$
and with rate $\alpha$ in the case $D\ge 3$
as the scale of initialization $\alpha$ approaches zero
as predicted by our theory.
In the noisy scenario,
we observe that the null space constant $\varrho$
is very close to $1$.
For this reason,
in the case $D=2$ 
the approximation error converges only slowly towards the 
$\ell^1$-minimizer,
whereas for $D\ge 3$ we observe better behavior.
This might indicate an advantage of deeper nets 
over shallow nets.

\paragraph*{Outline of this paper:}
The remainder of this paper is structured as follows.
In \cref{section:MainResults}, we present the main theoretical findings of our work.
In \cref{section:related_work}, we discuss further some related work.
In \cref{section:Simulations}, we conduct numerical experiments to validate our theoretical results.
These experiments also demonstrate that,
particularly in the presence of noise, 
deeper nets with $D \ge 3$ 
have a significant advantage in terms of implicit regularization over shallow nets with $D=2$.
In \cref{section:Proofs}, we provide the proofs of the upper and lower bounds on the approximation error.
In \cref{section:Sharpness}, we construct explicit examples
to show that our upper and lower bounds are tight in an asymptotic sense.
Finally, in \cref{section:outlook},
we discuss interesting directions for future research.

\paragraph*{Notation:}
For an integer $d\in \N$, we define $[d]:=\{1,\dots,d\}$.
For $x,y\in \R^d$, 
we denote the standard inner product by $\langle x,y \rangle := \sum_{i=1}^d x_iy_i$.
Given a subset $S\subset [d]$, 
we write $x_S:= (x_i)_{i\in S}$ and $\langle x_S,y_S\rangle := \sum_{i\in S} x_i y_i$.
Moreover, we denote by $x\hada y$ the Hadamard product of $x$ and $y$ given by $(x\hada y)_i := x_iy_i$ for $i\in [d]$. We define $\abs{x}\in\R^d$ to be the vector with entries $\abs{x}_i :=\abs{x_i}$ for $i\in [d]$.
For a vector $x \in \R^d$ and a subset $L \subset \R^d$, we define:
\begin{equation*}
	S(x):=\supp(x):= \{ i\in [d] : x_i \ne 0\},
	\qquad \text{and} \qquad
	\supp(L) := \bigcup_{x\in L} \supp(x).
\end{equation*}
Furthermore, we set $S^c(x) := [d] \setminus \supp(x)$.

%% file: parts/main_results.tex
\section{Main results}\label{section:MainResults}

\subsection{Our setting}

Before we state the main results of this paper
we introduce our main assumptions
in Section \ref{section:MainResults}.

\begin{assumption} \label[assumption]{assumption_on_A_y}
Let $A\in \R^{N\times d}$ and $y\in \R^N$. 
We assume that:
\begin{itemize}
    \item[(a)] there exists $x\in \R^d$ such that $Ax=y$,
    \item[(b)] $y\ne 0$,
    \item[(c)] $\ker(A)\ne \{0\}$,
	\item[(d)] and there is a unique minimizer $\ga$
    of the minimization problem $\min_{x:Ax=y} \norm{x}_{\ell^1}$.
\end{itemize}
\end{assumption}
\begin{remark}
The most important assumption we have made is Assumption $(d)$.
While this assumption is satisfied in most scenarios of interest,
our theory can be extended to the non-unique scenario.
For the sake of completeness, we have 
included the non-unique case in the appendix,
see Appendix \ref{section:main_results_non-unique}.
The reason why we have chosen to 
focus on the unique case in the main part of the paper
is that it is easiest to present prove our results in this case.

Assumptions (a), (b), and (c) are standard in the literature
and are not restrictive.
If Assumption $(a)$ does not hold 
gradient flow will converge to a limit point $\xinf (\alpha)$ 
which can be characterized as
$ \xinf (\alpha) = \underset{x \in \mathcal{T}}{\arg \min} \ D_{F_{\alpha,D}} (x, 0) $.
Here, 
$\mathcal{T} \subset \R^d$ denotes the affine subspace 
$\mathcal{T}:= \underset{x \in \R^d }{\arg \min} \norm{Ax-y}_{\ell^2}  $,
see, e.g., \cite[Theorem 3.8]{jin2023understanding}.
Our theory can be extended verbatim to this scenario.
However, to keep the presentation simple,
we will consider the case when Assumption $(a)$ holds.
If Assumption $(b)$ does not hold, i.e., we have that $y=0$,
then the gradient flow initialization $x(0)$ is already a global minimizer of the
loss function $\mathcal{L}$, and we will have $\xinf (\alpha)=0$ as well.
Note that if Assumption $(c)$ does not hold, i.e., we have that $ \ker (A)=\{0\}$,
then the equation $Ax=y$ has a unique solution $x$
and the function $\mathcal{L}$ in \cref{equ:lossfunction}
has a unique global minimizer.
In this scenario, the question of implicit regularization becomes meaningless.
\end{remark}

With these assumptions in place, we can define the following constants.
These are reminiscent of the null space constants 
studied in Compressed Sensing \cite{CohenNullSpace2009}.
It has been shown that these constants characterize 
the success of $\ell^1$-minimization 
and other methods such as Iteratively Reweighted Least Squares for sparse recovery problems
see \cite{foucart_MathematicalIntroductionCompressive_2013}.
For this reason, we will refer to them as null space property constants as well in this paper.
\begin{definition}[Null space property constants]
Assume that $A$ and $y$ fulfill \cref{assumption_on_A_y}
with a unique minimizer $\ga$.
Denote by $\SP:=\supp \left( \ga \right)$ the support of $\ga$. 
The null space property constants 
$\varrho$, $\varrho^{-}$, and $\tilde{\varrho}$ are defined as
\begin{equation}\label{nullspaceconstants}
	\begin{aligned}
		\varrho&:=\sup_{0\ne n\in \ker (A)} 
		\frac{-\sum_{i\in \SP} \text{sign}\left(\ga_i \right) n_i}{\norm{n_{\SC}}_{\ell^1}},\\
		\varrho^{-}&:= \sup_{0\ne n\in \ker (A)}
		\frac{\sum_{i \in \SP: \text{sign} (\ga_i) n_i <0} \abs{n_i}}{\norm{n_{\SC}}_{\ell^1}},\\
		\tilde{\varrho}&:=\sup_{0\ne n\in \ker (A)}
		\frac{\norm{n_{\SP}}_{\ell^1}}{\norm{n_{\SC}}_{\ell^1}}.
	\end{aligned}
\end{equation}
\end{definition}

The following proposition 
ensures that the constants $\varrho$, $\varrho^{-}$, and $\tilde{\varrho}$
are well-defined and states their main properties.
For the sake of completeness, 
we provide the straightforward proof of this result in 
Appendix \ref{section:proof_of_normal_conelemma}.
\begin{proposition} \label[proposition]{lemma:normal_coneSimple}
	Assume that $A$ and $y$ fulfill \cref{assumption_on_A_y}.
	Then the following statements hold
	for the null space property constants
	$\varrho, \varrho^{-},$ and $\tilde{\varrho}$
	introduced above.
	\begin{enumerate}
	\item 
	It holds that $\SC \ne	\emptyset$ and $\SP\ne \emptyset$.
	Moreover,
	for every $n\in \ker (A) \setminus \{0\}$ with $n\ne 0$
	it holds that $n_{\SC}\ne 0$,
	where $\SC = \{ 1;2; \ldots d \} \setminus \mathcal{S}$.
	In particular, the null space constants in \eqref{nullspaceconstants} are well-defined.
	\item 	
	It holds that $0 \le \varrho<1$ and 
	$0 \le \varrho^{-} <\infty$,
	$0 \le  \tilde{\varrho} <\infty$.
	Moreover, the suprema in \eqref{nullspaceconstants} are attained.
	\end{enumerate}
\end{proposition}
Finally, to formulate our main results,
we will also need to introduce the condition number of the unique minimizer $\ga$.
\begin{definition}[Condition number]
Assume that $A$ and $y$ fulfill \cref{assumption_on_A_y}
with a unique minimizer $\ga$.
Then the condition number $\kappa_{\ast}$ of $\ga$ is defined as
\begin{equation}\label{constants_general_case_unique}
	\kappa_{\ast}
    := 
    \frac{ \max\limits_{i\in \SP} \abs{\ga_i}}{\min\limits_{i\in \SP}\abs{\ga_i}}.
\end{equation}
\end{definition}

\subsection{Shallow case ($D=2$)}

In the shallow case, i.e., when $D=2$,
the potential function $F_{\alpha,D}$ of the Bregman divergence is given by 
$F_{\alpha,2}=\Hb$, where
\begin{equation*}
	\Hb (z)
	:=
	\sum_{i=1}^d \left(z_i \arsinh \left(\frac{z_i}{2\alpha}\right) -\sqrt{z_i^2+4\alpha^2}\right)
\end{equation*}
for all $z\in \R^d$, see, e.g., \cite[Theorem 1]{woodworth2020kernel}.
Our main result in the shallow case reads as follows.
\begin{theorem}\label[theorem]{theorem:upper_bound_+-}
	Let $A \in \R^{N \times d}$ and $y \in \R^d$.
    Assume that $A$ and $y$ fulfill \cref{assumption_on_A_y}
    with a unique minimizer $\ga$
	and corresponding support $\mathcal{S}$.
	The null space constants $\varrho$, $\varrho^{-}$, and $\tilde{\varrho}$ are as
	defined in \eqref{nullspaceconstants}.
	Let $\alpha>0$ be the scale of initialization
	and consider
	\begin{equation*}
		\xinf \in \argmin_{x:Ax=y} D_{\Hb}\left( x,0 \right).
	\end{equation*}
	Then the following two statements hold.
    \begin{enumerate}
	\item 
    \textbf{Upper bound:}
    It holds that 
	\begin{equation}\label{eqn:error_bound_sparse_hyperbolic}
		\frac{\norm{\xinf - \ga}_{\ell^1}}
		{\alpha^{1-\varrho}}
		\le 
		\abs{\SC}
		(1+\tilde{\varrho})
		\cdot
		\left( \min_{i\in \SP}\abs{\ga_i} \right)^{\varrho}
		\kappa_{\ast}^{\varrho^{-}}
		\cdot \left(
            1 +
            \frac{\alpha^2}{ \left( \min_{i\in \SP}\abs{\ga_i} \right)^2} 
            \right)^\varrho.
	\end{equation}
    \item
    \textbf{Lower bound: }
    Assume in addition that 
	\begin{equation*}
		\frac{\alpha}{\min_{i\in \SP}\abs{\ga_i}} 
		\le \left(\frac{1}{4\tilde{\varrho}
			\kappa_{\ast}^{\varrho^{-}}
			\abs{\SC}}
            \right)^{\frac{1}{1-\varrho}}.
	\end{equation*}
	Then it holds that
    \begin{equation*}
	\begin{split}
	\frac{\norm{\xinf_{\SC}- \ga_{\SC} }_{\ell^{\infty}}}
	{\alpha^{1-\varrho}}
	\ge 
	\frac
	{ \norm{\ga}_{\ell^\infty}^{\varrho}}
	{\kappa_{\ast}^{\varrho^-}}
	\left(
		1-
		8\tilde{\varrho}^2
        \abs{\SC}
	    \kappa_{\ast}^{\varrho^{-}} 
		\left(
		\frac{\alpha}
		{  \underset{i \in \mathcal{S}}{\min} \ \vert \ga_{i} \vert }
		\right)^{1-\varrho}
		-
		\kappa_{\ast}^{2\varrho^{-} -2\varrho}
		\left(
		\frac{\alpha}
		{\min_{i \in \mathcal{S}} \abs{\ga_i}}
		\right)^{2\varrho}
	\right).
	\end{split}
    \end{equation*}
\end{enumerate}
\end{theorem}
The proof of \cref{theorem:upper_bound_+-}
is deferred to \cref{section:Proofs}.

We observe that since the $\ell^1$-norm and the $\ell^{\infty}$-norm are equivalent,
\cref{theorem:upper_bound_+-} implies that
for fixed $A$ and $y$, we have
\begin{equation*}
	\frac
	{ \norm{\ga}_{\ell^\infty}^{\varrho}}
	{\kappa_{\ast}^{\varrho^-}}
	+
	o(1)
	\le
	\frac{\norm{\xinf (\alpha) - \ga }_{\ell^1}}{\alpha^{1-\varrho}}
	\le
	\abs{\SC}
	(1+\tilde{\varrho})
	\cdot
	\left( \min_{i\in \SP}\abs{\ga_i} \right)^{\varrho}
	\kappa_{\ast}^{\varrho^{-}}
	+
	o(1)
	\quad
	\text{ as }
	\alpha \downarrow 0.
\end{equation*}
In particular, the convergence rate
is proportional to $\alpha^{1-\varrho}$
and
is completely determined by 
the null space parameter $\varrho$.

The upper bound in \cref{theorem:upper_bound_+-}
improves over previous work in the literature.
In \cite{woodworth2020kernel,chou2023less} it was shown
that the approximation error decays with $O \left( \log(1/\alpha) \right)$
as the scale of initialization $\alpha$ approaches $0$.
These results were improved in \cite{wind2023implicit} to $O \left( \alpha^c \right)$.
However, the constant $c \in (0,1)$ was not further determined.
In contrast, our result determines the constant $c \in (0,1)$ precisely.
Moreover, \cref{theorem:upper_bound_+-} is the first result in the literature
which complements the upper bound with a lower bound
which consequently shows that $\alpha^{1-\varrho}$
is the correct rate of convergence.

One may ask whether the upper and lower bounds in \cref{theorem:upper_bound_+-}
can be further improved. 
The following result states
that our bounds are tight in an asymptotic sense
as $\alpha \downarrow 0$.
Thus, at least in an asymptotic sense,
there is no room for further refinement.
\begin{proposition}\label[proposition]{proposition:SharpnessDtwo}
	For given $A \in \R^{N \times d}$ and $y \in \R^N$
	denote 
	for any $\alpha>0$
	by $\xinf (\alpha)$
	the unique minimizer of
	\begin{equation}\label{equ:internDominik3}
		\min_{x:Ax=y} D_{\Hb}\left(x,0 \right).
	\end{equation}
	Now let $d \in \N$ with $d \ge 3$. 
	Choose any null space constants
	$\varrho \in [0,1)$,
	$\tilde{\varrho}>0$,
	and
	$\varrho^{-}>0$
	which satisfy the relations
	$\varrho^{-}\ge \varrho$ and 
	$ 2 \varrho^{-} - \varrho = \tilde{\varrho} $.
	Then there exists a matrix $A \in \R^{N \times d}$ 
	such that the following two statements hold
	for any $\kappa_{\ast} \ge 1$.
	\begin{enumerate}
		\item 
		There exists $y \in \R^N$ 
		such that 
		there is a unique minimizer $\ga$ 
		of 
		$\underset{x:Ax=y}{\min} \norm{x}_{\ell^1} $
		with condition number $\kappa_{\ast}$
		and such that the corresponding null space constant
		are $\varrho$, $\tilde{\varrho}$, and $\varrho^{-}$
		as chosen above.
		Moreover, the minimizers $\left(\xinf (\alpha)\right)_{\alpha >0}$ 
		of \cref{equ:internDominik3}
		satisfy
		\begin{equation}\label{equ:internDominik6}
			\lim_{\alpha \downarrow 0}
			\frac{\norm{ \xinf \left(\alpha \right)  - \ga}_{\ell^1}}
			{\alpha^{1-\varrho}  }
			=
			\abs{\SC} 
			(1+\tilde{\varrho})
			\cdot 
			\left( \min_{i\in \SP}\abs{\ga_i} \right)^{\varrho}
			\kappa_{\ast}^{\varrho^{-}}.
		\end{equation}
		\item
		There exists $y \in \R^N$ 
		such that 
		there is a unique minimizer $\ga$ 
		with condition number $\kappa_{\ast}$
		of the optimization problem
		$\underset{x:Ax=y}{\min} \norm{x}_{\ell_1} $
		and such that
		the minimizers $\left(\xinf (\alpha)\right)_{\alpha >0}$ 
		of \cref{equ:internDominik3}
		satisfy
		\begin{equation}\label{equ:internDominik5}
			\lim_{\alpha \downarrow 0}
			\frac{\norm{ \left(\xinf \left(\alpha \right) \right)_{\SC} 
			- \ga_{\SC} }_{\ell^{\infty}}}
			{\alpha^{1-\varrho} }
			=
			\frac{\norm{ \ga }_{\ell^\infty}^{\varrho}}
			{\kappa_{\ast}^{\varrho^{-}}}.
		\end{equation}
	\end{enumerate}
\end{proposition}
The proof of \cref{proposition:SharpnessDtwo} is deferred to \cref{section:Sharpness}.
We note that $\varrho^{-} \ge \varrho$ is a direct consequence
of the definition of these two constants.
It remains an open problem whether the 
condition $2\varrho^{-} - \varrho = \tilde{\varrho}$
can be relaxed.

\subsection{Deep case ($D\ge 3$)}

In the deep case, i.e., when $D \ge 3$,
the potential function $F_{\alpha,D}$ of the Bregman divergence is given by 
$F_{\alpha,D}=\QQ$, see \cite[Theorem 3]{woodworth2020kernel}.
To define this function $\QQ$, we 
first introduce
the function
$h_D: (-1,1) \rightarrow \R$
defined by
\begin{equation*}
h_D (z)
:=
\frac{1}{(1-z)^{D/(D-2)}} -\frac{1}{(1+z)^{D/(D-2)}}
\end{equation*}
for all $z\in \R$.
Next, 
denote by $h_D^{-1}$ the inverse of $h_D$
and define
$q_D(u):=\int_{0}^{u} h^{-1}_D(v) dv$.
With these definitions in place,
we can finally define the function $\QQ$ as
\begin{equation*}
	\QQ (z)
	:=
	\sum_{i=1}^{d} \alpha q_D \left( z_i/\alpha \right)
\end{equation*}
for all $z \in \R^d$.

Our main result in the deep case reads then as follows.

\begin{theorem}\label[theorem]{theorem:upper_bound_deep}
	Let $A \in \R^{N \times d}$ and $y \in \R^d$.
    Assume that $A$ and $y$ fulfill \cref{assumption_on_A_y}
    with a unique minimizer $\ga$ and corresponding support $\mathcal{S}$.
	The null space constants $\varrho$, $\varrho^{-}$, and $\tilde{\varrho}$ are as
	defined in \cref{nullspaceconstants}.
    Assume that $D \in \N$ with $D\ge 3$, let $\gamma :=\frac{D-2}{D}$, and let $\alpha>0$.
	Let
	\begin{equation*}
		\xinf \in \argmin_{x:Ax=y} D_{\QQ}(x,0).
	\end{equation*}
    Then the following statements hold.
    \begin{enumerate}
    \item
    \textbf{Upper bound:}
    Assume that
	\begin{equation} \label{eqn1001}
			\frac{\alpha}{\mg} 
			<  \left( 
            \frac{(1-\varrho)\gamma}{4\varrho^{-}} \right)^{\frac{1}{\gamma}}.
	\end{equation}
	Then
	\begin{equation} \label{eqn1002}
		\frac{\norm{\xinf -\ga}_{\ell^1}}{\alpha}
			\le \abs{\SC} (1+\tilde{\varrho})
			\bigg[ 
			\hh(\varrho) 
			+\frac{4\varrho^{-}}{\gamma (1-\varrho)^{\frac{1}{\gamma}+1}} 
			\cdot \Big(\frac{\alpha}{\mg}\Big)^{\gamma}
			\bigg].
	\end{equation}
    \item
    \textbf{Lower bound: }
	Assume in addition that
	\begin{equation} \label{eqn1006}
		\frac{\alpha}{\min_{i\in \SP}\abs{\ga_i}} 
		\le 
		\min\bigg\{ 
		\left( \frac{(1-\varrho) \gamma}{4\varrho^{-}} \right)^{\frac{1}{\gamma}},
		\quad
		\frac{(1-\varrho)^{\frac{1}{\gamma}}}{4(1+\tilde{\varrho})\abs{\SC}},
		\quad
		\left(
			\frac{\varrho}{\varrho  + 2^{2+\gamma} \tilde{\varrho}  \gamma  \kappa_{\star}}
		\right)^{\frac{1}{\gamma}} 
		\bigg\}.
	\end{equation}
	Then
	\begin{equation} \label{lower_bound_deepd_unique}
		\frac{\norm{\xinf_{\SC} -\ga_{\SC}}_{\ell^{\infty}}}{\alpha}
		\ge  
			\hh(\varrho) 
			- 
			\frac{2(\varrho  + 2^{2+\gamma}\tilde{\varrho} \gamma \kappa_{\star} )}
			{\gamma  (1-\varrho)^{\frac{1}{\gamma}+1}} 
			\cdot \Big(\frac{\alpha}{\min_{i\in \SP}\abs{\ga_i}}\Big)^{\gamma}.
	\end{equation}
    \end{enumerate}
\end{theorem}
The proof of \cref{theorem:upper_bound_deep}
is deferred to \cref{section:Proofs}.

Since the $\ell^{\infty}$-norm
is smaller than the $\ell^{1}$-norm, 
\cref{theorem:upper_bound_+-} implies
that for fixed $A$ and $y$
we have that
\begin{equation*}
\hh(\varrho) + o(1) 	
\le
\frac{\norm{\xinf (\alpha) - \ga }_{\ell^1}}{\alpha}
\le
\abs{\SC}
\left(1+\tilde{\varrho} \right)
\hh (\varrho)
+ o(1)	
\quad
\text{ as }
\alpha \downarrow 0.
\end{equation*}
In particular, the convergence rate of the approximation error is proportional to $\alpha$.
Thus, the convergence rate is faster than in the shallow case,
where the rate is given by $\alpha^{1-\varrho}$.
As we will see in our numerical experiments in \cref{section:Simulations},
the constant $\varrho \in [0,1)$ is typically smaller in noisy settings.
Thus, this result indicates that the advantage of deeper networks 
is especially pronounced in noisy settings.

We note that the upper bound in \cref{theorem:upper_bound_+-}
improves over previous work in the literature.
In \cite{chou2023less} it was shown
that the approximation error is bounded from above by $O \left( \alpha^{1-2/D} \right)$.
These results were improved in \cite{wind2023implicit}
to a bound of the form
\begin{equation*}
\norm{\xinf (\alpha) - \ga}_{\ell^2}
\le
C_A \alpha,
\end{equation*}
where $C_A$ denotes an absolute constant which depends only on $A$.
However, the absolute constant $A$ was not determined.
In contrast, our constant is determined precisely.
Moreover, \cref{theorem:upper_bound_deep} is the first result in the literature
which shows a lower bound for the case $D\ge 3$.

The following result shows that our upper and lower bounds are sharp in an asymptotic sense.
\begin{proposition}\label[proposition]{proposition:SharpnessDlargertwo}
	For given $A \in \R^{N \times d}$ and $y \in \R^N$
	denote 
	for any $\alpha>0$
	by $\xinf (\alpha)$
	the unique minimizer of
	\begin{equation}\label{equ:internDominik10}
		\min_{x:Ax=y} D_{\QQ}\left(x,0 \right).
	\end{equation}
	Let $d \in \N$ with $d \ge 3$
	and let $\varrho \in [0,1)$
	be arbitrary.
	Then there exists a matrix $A \in \R^{N \times d}$ 
	and $y \in \R^d$ 
	such that the following holds.
	\begin{enumerate}
		\item There exists a unique minimizer $\ga \in \R^N$ of the optimization problem
        $ \underset{x: Ax=y}{\min} \norm{x}_{\ell_1} $
		such that the null space constant corresponding to $A$ and $\ga$
		are equal to $\varrho$ and $\tilde{\varrho}$
		as chosen above.
		\item Moreover, it holds that
		\begin{equation}\label{equ:internDominik13}
		    \lim_{\alpha \downarrow 0}
			\frac{\norm{\xinf \left(\alpha \right) - \ga}_{\ell^1}}
			{\alpha
			\abs{\SC}
			\left(1 + \tilde{\varrho}\right)}
			=
			\hh (\varrho)
		\end{equation}
		as well as
		\begin{equation}\label{equ:internDominik14}
		    \lim_{\alpha \downarrow 0}
			\frac{\norm{ \left( \xinf \left(\alpha \right) \right)_{\SC} - \ga_{\SC}}_{\ell^{\infty}}}
			{\alpha}
			=
			\hh (\varrho).
		\end{equation}
	\end{enumerate}
\end{proposition}
The proof of \cref{proposition:SharpnessDlargertwo} is deferred to \cref{section:Sharpness}.

%% file: parts/related_work.tex
\section{Related work}\label{section:related_work}

As mentioned in the introduction,
diagonal linear neural networks
in a regression context
have been studied extensively
\cite{vaskevicius2019implicit,woodworth2020kernel,amid2020reparameterizing,amid2020winnowing,yun2021a,azulay2021implicit,li2022implicit,chou2023less},
showing that this architecture can implicitly regularize towards sparsity.
Also, the training dynamics were rigorously studied in \cite{pesme2023saddle},
where a \textit{saddle-to-saddle} dynamics was established.
Additionally, 
the implicit bias of momentum-based optimization algorithms
in the context of diagonal linear networks was analyzed in \cite{papazov2024leveraging}.
The authors of the paper at hand also have published a 
short note where they study a simplified version of the problem considered in this paper.
Namely they consider positively quadratically reparameterizations linear regression,
i.e., $v=0$ and $D=2$, see \cite{mattimplicit}.
In this note, they establish 
analogous similar upper and lower bounds 
as in the present paper.

A key insight in this line of research is that gradient flow with Hadamard
reparameterization is equivalent to the mirror descent/flow algorithm 
on the original parameter space with a suitable potential function $F_{\alpha,D}$.
In \cite{li2022implicit}, 
conditions were examined when gradient flow 
on a reparameterized loss function can be equivalently 
understood as a mirror flow.
This connection between gradient flow and mirror flow 
has been further explored to determine whether implicit regularization 
towards other minimizers can be induced by different reparameterizations.
For example, in \cite{chou2023induce},
implicit regularization towards the $\ell^p$-norm with $p \in (1,2)$
has been studied for certain reparameterizations,
whereas \cite{kolb2023smoothing}
studied reparameterizations that 
induce an implicit bias towards solutions with minimal $\ell_{p,q}$-norm.
In the context of sparse phase retrieval, 
mirror flow with the hypentropy mirror map and 
the closely related quadratically reparameterized Wirtinger flow 
were studied in \cite{wu2020continuous,wu2023nearly}.

Implicit regularization has also been examined in classification tasks
with linear classifiers,
see e.g., 
\cite{soudry2018,nacson2019convergence,moroshko2020implicit,ji2019implicit,ji2021characterizing}.
It has been observed that,
in certain cases,
gradient descent converges to certain max-margin classifiers.
These observations have been extended to more general reparameterizations
in \cite{SunJMIR2023,pesme2024implicit}.

Beyond models related to diagonal reparameterizations, 
implicit regularization has been studied in the context of
linear convolution neural networks \cite{gunasekar2017implicit},
low-tubal tensor recovery 
\cite{karnik2024implicitregularizationtubaltensor},
low-rank tensor completion \cite{razin2021implicitTensor},
and low-rank matrix recovery via factorized gradient descent
\cite{gunasekar2017implicit,li2018algorithmic,arora2019implicit,
litowards,razin2020implicit,stoger2021small,soltanolkotabi2023implicit,
jin2023understanding,wind2023asymmetric,chou2024deep,ma2024convergence,
RauhutWestdicke2022,nguegnang2024convergence}.
In the latter, a bias towards low-rank matrices has been established.
Also in the context of low-rank matrix recovery, 
in \cite{wu2021implicit},
mirror descent with a matrix version of the hypentropy mirror map was studied
and implicit regularization towards the nuclear norm minimizer for small initialization was 
established.
However, 
as \cite{li2022implicit} points out,
the 
connection between mirror descent and factorized gradient descent
in the case 
of low-rank matrix recovery
remains unclear,
since the equivalence between gradient flow on the factorized objective function 
and mirror flow does not hold in general.

In \cite{yaras2023lawparsimonygradientdescent,SooKwon2024,laus2025solving},
deep linear neural networks of the form 
$ x \mapsto W_1\cdot W_2 \ldots \cdot W_L \cdot x$
where studied.
In particular, it was shown that these models
exhibit an implicit bias towards low-rank weight matrices
which allows them to learn the underlying low-dimensional structure of the data.
Another line of research also studied implicit bias in neural networks 
by adding additional linear layers to a ReLU network.
Namely, in \cite{parkinson2025reluneuralnetworkslinear},
it was shown that adding linear layers to a ReLU network
induces a bias towards functions with low mixed variation,
i.e., functions that vary only in a few directions.
Experimentally, it was observed that this bias can lead to improved generalization performance.


%% file: parts/simulation.tex
\section{Simulations}\label{section:Simulations}

In this section,
we conduct numerical experiments to support our theoretical findings.

\paragraph*{Experimental setup:}
We pick a random matrix $A \in \mathbb{R}^{N \times d}$
with $d=300$ and $N=60$.
The entries of the matrix $A$ are chosen to be i.i.d.~with standard Gaussian distribution $\mathcal{N} (0,1)$.
We choose a ground truth vector $x_{0}$ with sparsity $s=5$.
Then we define $y_0 := A x_{0} $.
Next, we pick a noise vector $n \in \mathbb{R}^N$ from the unit sphere with uniform distribution.
Then we set
\begin{equation*}
	y
	:=
	y_0
	+
	\eta \cdot \norm{y_0}_{\ell_2}\cdot  n,
\end{equation*}
where we refer to  $\eta>0$ as the noise level.
In our experiments, we compute the $\ell^1$-minimizer
\begin{equation*}
	\ga
	:=
	\underset{x: Ax=y}{\text{arg min}} \ \Vert x \Vert_{\ell^1}
\end{equation*}
using solvers from the \textit{splitting conic solver} package \cite{scs}.
Moreover, we compute minimizers
\begin{equation*}
	\xinf (\alpha)
	:=
	\underset{x: Ax=y}{\text{arg min}} \ D_F (x, 0)
\end{equation*}
for different values of $\alpha$,
where $D_F$ is the Bregman divergence with potential function $F=\Hb$
in the case $D=2$ and $F=\QQ$ in the case $D\ge3$.
In our experiments, we use mirror descent \cite{nemivorvskij1979problem} to solve this constrained optimization problem.
More precisely, we minimize the objective function $\mathcal{L}$,
see \cref{equ:lossfunction},
using the mirror descent algorithm
with potential function $F$ and initialization at zero.
It has been established that in this case mirror descent converges to $\xinf (\alpha)$,
see, e.g., \cite{gunasekar2018characterizing}.
We run the mirror descent algorithm
until the value of the loss function $\mathcal{L}$ is less than $10^{-5}$.

\paragraph*{Experiment 1: The scenario $D=2$  with different levels of noise}
In our first experiment, we set $D=2$
and we consider different noise levels $\eta=0,0.03,0.1,0.4$.
We compute minimizers $\xinf(\alpha)$ for
different scales of initialization $\alpha = 10^{-i}$ and $i=0,1,\ldots,11$.
The experimental results are depicted in Figure \ref{fig:experiment1}.
\begin{figure}[!h]
	\centering
	\includegraphics[width=0.5\textwidth]{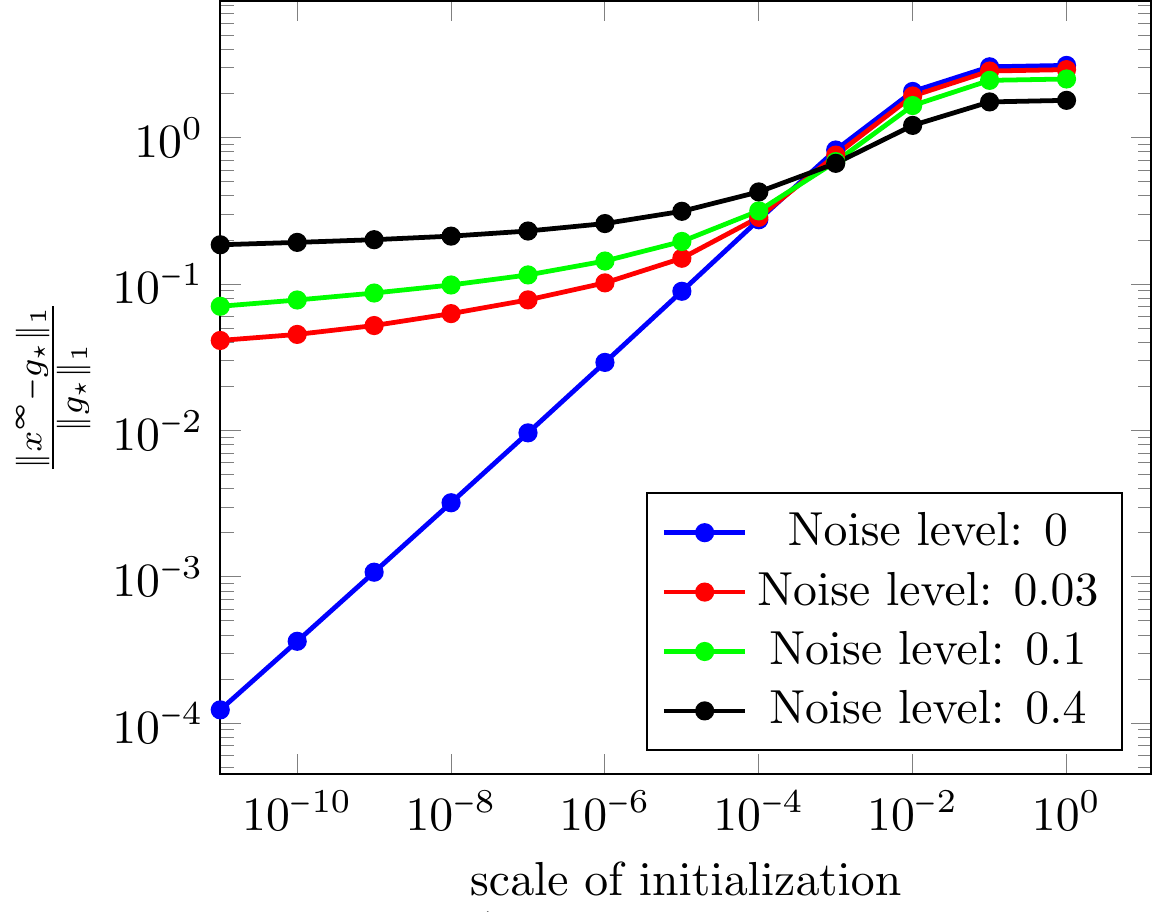}
	\caption{\textbf{Impact of different noise levels on the approximation error
			as the scale of initialization goes to zero (Experiment 1)}
		We consider the case $D=2$ and different noise levels $\eta=0,0.03,0.1,0.4$.
		We observe that in the noisy scenario the approximation error converges much slower to zero
		than in the noiseless scenario.}
	\label{fig:experiment1}
\end{figure}

In all four cases
we observe that
for sufficiently small $\alpha$,
the approximation error converges
to zero with a \textit{polynomial rate} of $ \alpha^c $ for some $c \in (0,1)$
as the scale of initialization $\alpha$ approaches zero.
This is in line with the predictions by \cref{theorem:upper_bound_+-}.
We observe that the slope of the curve in the noiseless scenario is larger than
the slopes
of the three curves corresponding to the noisy scenarios.
This indicates that the implicit regularization effect
is stronger in the noiseless scenario compared to the noisy ones.
According to Theorem \ref{theorem:upper_bound_+-} the slopes of the four curves are characterized by $c=1-\varrho$,
where $\varrho$ is the null space property constant corresponding to the $\ell^1$-minimizer $\ga$,
see \cref{nullspaceconstants}.
Thus, our experiments show that
this null space property constant $\varrho$ is smaller in the noiseless scenario
than in the noisy ones.
In particular, we observe
that in the noisy scenario the null space property constant $\varrho$ is quite close to $1$
and the convergence to the $\ell^1$-minimizer $\ga$ is slow.

\begin{remark}
	Our experimental findings can be explained as follows.
	Note that the null space property constant $\varrho$ depends
	on the alignment between the null space of the matrix $A$ and the
	descent cone of the $\ell^1$-norm at the point $\ga$.
	Here, descent cone refers to the set of all directions in which the $\ell^1$-norm decreases.
	In particular, a sparser signal $\ga$
	leads to a smaller descent cone, see, e.g., \cite{Chandrasekaran2012,Amelunxen2014}.
	Since the null space of $A$ is randomly chosen,
	one expects for a sparser signal
	that that the null space is less aligned with this smaller descent cone.
	Consequently, a sparser signal $\ga$ should lead to a smaller null space property constant $\varrho$.

	Now note that in the noiseless case, the $\ell^1$-minimizer is $s$-sparse
	and we have $ x^0 = \ga $.
    (We have verified this numerically in our experiments
		but this can also be explained
		using standard Compressed Sensing theory,
		see, e.g, \cite{foucart_MathematicalIntroductionCompressive_2013}.)
	However, as soon as we add noise to the signal
	the $\ell^1$-minimizer $\ga$
	recovers the ground truth $x_0$ no longer exactly.
	In particular,
	$\ga$ has much larger support in the noisy case.
	We have verified also this numerically in our experiments.
	Thus, with the reasoning above,
	we expect that in the noisy case the null space property constant $\varrho$
	is larger than in the sparse case.
\end{remark}

\paragraph*{Experiment 2: Different choices of $D$ in the noiseless and noisy scenarios}
In our second experiment, we vary the number of layers $D $.
We consider two cases.
In the first case, we set $\eta=0$, i.e., we consider the noiseless scenario.
In the second case, we set $\eta=0.1$, i.e., we consider a noisy scenario.
Again, we vary the scale of initialization $\alpha$.
(In the noisy scenario with $D=6$, we did not compute $\xinf(\alpha)$ for $\alpha < 10^{-9}$
as the optimization problem became too computationally expensive to solve.)
The results of this experiment are depicted in Figure \ref{fig:DifferentD}.
\begin{figure}[!h]
	\begin{subfigure}[b]{0.5\textwidth}
		\includegraphics[width=0.95\textwidth]{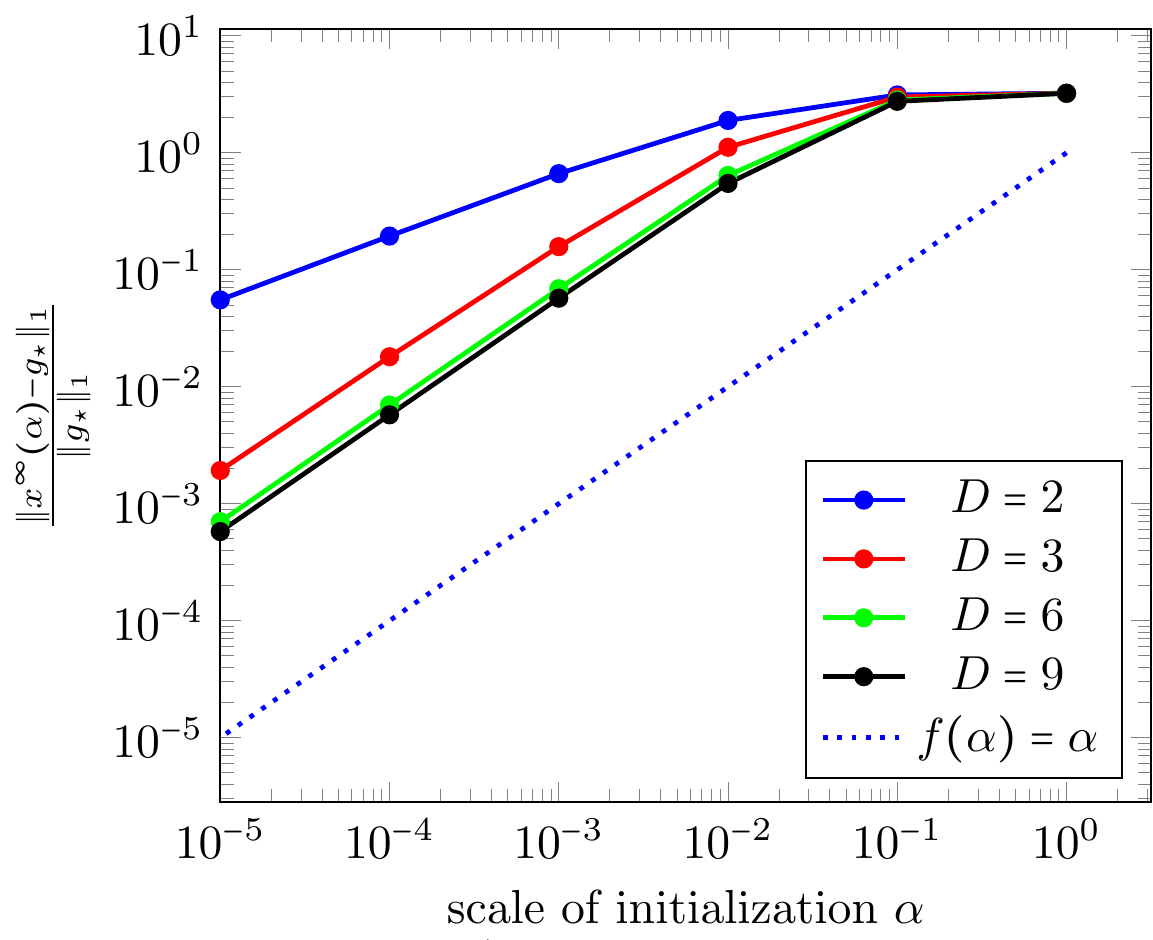}
		\caption{Noiseless scenario}
		\label{fig:DifferentDnoisefree}
	\end{subfigure}
	\begin{subfigure}[b]{0.5\textwidth}
		\includegraphics[width=0.95\textwidth]{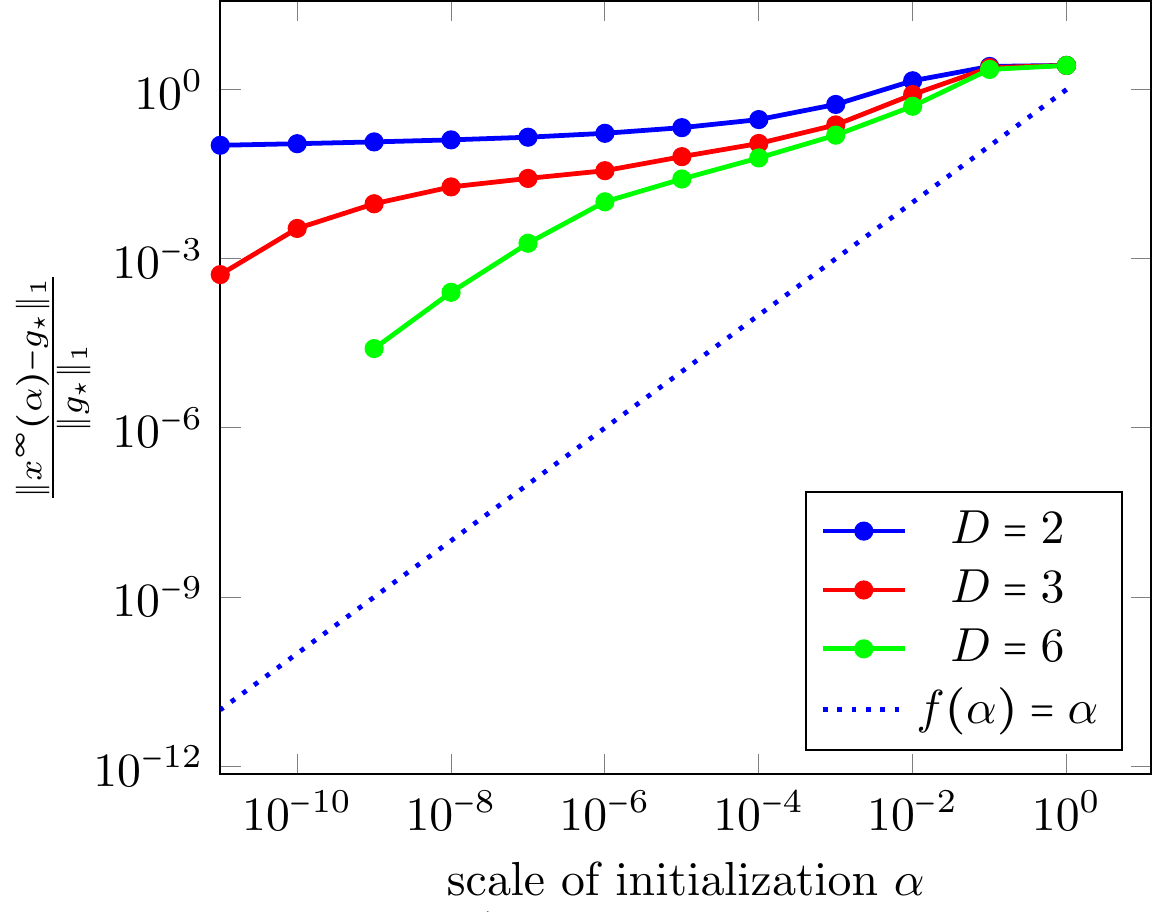}
		\caption{Noisy scenario}
		\label{fig:DifferentDnoisy}
	\end{subfigure}
	\caption{\textbf{Impact of different choices of $D$ on the approximation error
			as the scale of initialization goes to zero (Experiment 2)}
		We consider a noiseless scenario with $\eta=0$ and a noisy scenario with $\eta=0.1$,
		for different values of the number of layers $D$.
		We observe that in the noiseless scenario the $\ell^1$-approximation error converges
		to zero faster than in the noisy scenario.}
	\label{fig:DifferentD}
\end{figure}

In the noiseless scenario,
we observe that for $D\ge 3$ the approximation error converges to zero
with a linear rate proportional to $ \alpha $.
This is in line with the predictions from \cref{theorem:upper_bound_deep}.
In the noisy scenario, we observe a slower convergence compared to the noiseless scenario
for all choices of $D$.
Moreover, for a fixed scale of initialization $\alpha$, we observe that adding more layers, i.e.~increasing the number $D$, significantly improves the approximation error.

The experiment in the noisy case shows that the linear decay of the approximation error only manifests for sufficiently small $\alpha$.
Hence the assumptions \eqref{eqn1001} and \eqref{eqn1006} in \cref{theorem:upper_bound_deep} are necessary.
For, $D=6$, we observe for $\alpha \le 10^{-7}$ 
a linear convergence rate proportional to $\alpha$.
For $D=3$, we observe that this linear convergence rate occurs for $\alpha \le 10^{-9}$.
This indicates that with larger depth $D$, 
the linear convergence regime holds for larger values of $\alpha$.
This is in line with the above mentioned assumptions on $\alpha$, as the exponent $\frac{1}{\gamma} = \frac{D}{D-2}$ decreases as $D$ increases.

\paragraph{Experiment 3: Impact of depth and scale of initialization on the estimation/generalization error.}
    In this paper, we focused in our theoretical analysis and in the experiments so far on the
    approximation error $\norm{\xinf (\alpha) - \ga}_{\ell^p}$
    for $ p \in \{1; \infty \} $.
    However, in applications
    the goal is typically to estimate a sparse signal $\xa$ 
    from noisy measurements $y = A \xa + z$.
    While in this setting
    the $\ell^1$-minimizer $\ga$ is often a good estimator for $\xa$,
    in applications we are interested
    in the estimation error $\norm{\xinf (\alpha) - \xa}_{\ell^2}$
	directly
    instead of the approximation error.
    \Cref{fig:experiment4} shows an experiment 
    on how this estimation error depends on different depths $D$
    and on the scale of initialization $\alpha$.
    We observe that with depth $D \ge 3$ a comparable estimation error 
    can be achieved
    as with $\ell^1$-minimization
    while using a larger initialization.
    In contrast,
    for $D=2$,
    even with $\alpha =10^{-7}$
    we do not achieve comparable performance.
    This
    indicates 
    that in noisy scenarios
    only with depth $D \ge 3$
    we can achieve comparable performances
    to $\ell^1$-minimization while using a
    practical scale of initialization.
    \begin{figure}[!h]
        \centering
        \includegraphics[width=0.55\textwidth]{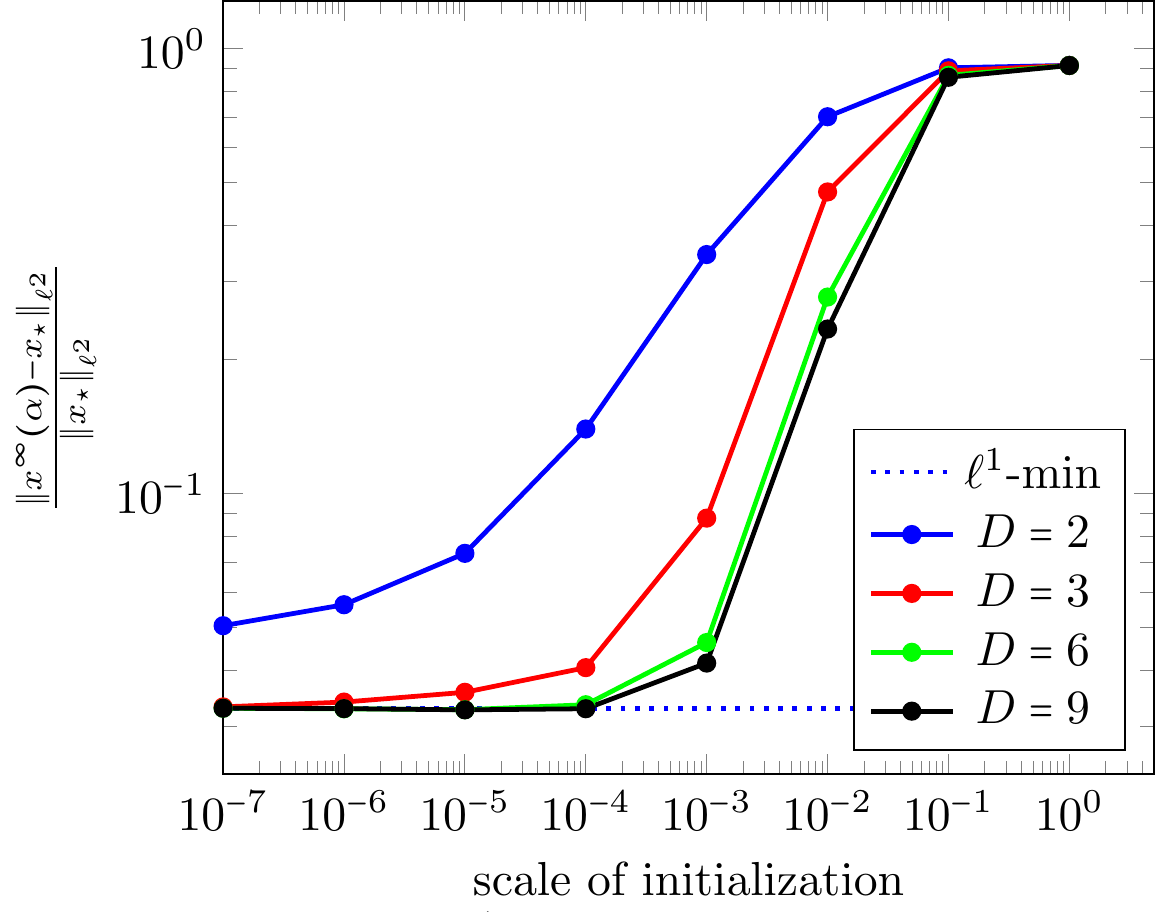}
        \caption{\textbf{Estimation error for different network depths.}
        We consider a noisy scenario with noise level equal to $\eta=0.03$.
        The dotted blue line denotes the estimation error of the $\ell^1$-minimizer $\ga$,
        i.e.,
        $\frac{\norm{ \ga - \xa  }_{\ell^2}}{\norm{ \xa }_{\ell^2}}$.
        We observe that with larger depth $D$, the same estimation error as the $\ell^1$-minimizer
        can be achieved while using a larger initialization.}
        \label{fig:experiment4}
    \end{figure}
\paragraph*{Summary}
Our numerical experiments show that in terms of implicit regularization,
there is a significant difference between the noiseless and noisy scenarios.
In the noiseless scenario, the approximation error $\norm{\xinfalpha - \ga}_{\ell^1}$
converges to zero fast for both shallow and deep nets.
This can be attributed to the fact that the null space property constant $\varrho$ is small,
which is a consequence of the sparsity of the $\ell^1$-minimizer.
In the noisy case, however, we observe that deeper nets achieve a
significantly better approximation error.

%% file: parts/technical_preliminaries.tex
The goal of this section is to prove the upper and lower bounds 
in \cref{theorem:upper_bound_+-} and \cref{theorem:upper_bound_deep}.
In  \cref{section:proofs_Dtwo}, 
we will prove the upper and lower bound in the case $D=2$.
In  \cref{section:Proofs_Dgreatertwo}, 
we will prove the corresponding bounds in the case $D\ge 3$.
Before that, we 
outline our proof strategy
and explain the main technical novelties of our proof approach.

\subsection{Proof ideas}

The main conceptual ideas in our proofs
are similar both in the shallow case, $D=2$, 
and in the deep case, $D \ge 3$.
Recall that we consider the potential function 
$F_{\alpha,D}$,
which
is given by $F_{\alpha,D}= \Hb $ in the case $D=2$
and by $F_{\alpha,D}= \QQ$ in the case $D \ge 3$.
First, we compute that 
\begin{equation*}
	\nabla_x D_{F_{\alpha,D}} (x,0)
	=
	\nabla F_{\alpha,D} (x) - \nabla F_{\alpha,D} (0).
\end{equation*}
Then, since $\xinf (\alpha)$ is a minimizer of the optimization problem
\begin{equation*}
	\underset{x \in \R^d: A x = y}{\min} \ D_{F_{\alpha,D}} \left(x, 0 \right),
\end{equation*}
it follows from the first-order optimality conditions
that for all $\tilde{n} \in \ker A$ it holds that 
\begin{equation*}
	\langle \nabla F_{\alpha,D} (\xinf (\alpha))
	- \nabla F_{\alpha,D} (0) , \tilde{n} \rangle
	= 0.
\end{equation*}
In both cases $D=2$ and $D \ge 3$ one can see
via a straightforward calculation that $\nabla F_{\alpha,D} (0) = 0$.
Moreover, we have that $\nabla F_{\alpha,D} (z) = \left( f_{\alpha,D} (z_i) \right)_{i=1}^d$
for some function $f_{\alpha,D} : \R \rightarrow \mathcal{I}$,
where $\mathcal{I} \subset \R$ is a symmetric interval, i.e., $-\mathcal{I}=\mathcal{I}$.
Thus, we obtain for $ n:= \xinf (\alpha) -\ga $ that
\begin{equation*}
	\langle \nabla F_{\alpha,D} ( \ga +n ), \tilde{n} \rangle
	=
	\sum_{i=1}^d f_{\alpha,D} ( \ga_i + n_i) \tilde{n}_i
	= 0
	\quad
	\text{ for all }
	\tilde{n} \in \ker A.
\end{equation*}
The first key observation in our proof is
that $\ga$ must have sparse support 
$\mathcal{S} \subsetneq  [d] $,
see \cref{lemma:normal_coneSimple}.
Thus, we can split the sum above into two parts,
one corresponding to the support of $\ga$,
denoted by $\SP$,
and one corresponding to the complement of the support of $\ga$,
which is $\SC := [d] \setminus \SP $.
We obtain that
\begin{align}\label{equ:proofsketch1}
	\sum_{i \in \SC} f_{\alpha,D} ( n_i) \tilde{n}_i
	=
    -
	\sum_{i \in \mathcal{S}} f_{\alpha,D} ( \ga_i + n_i) \tilde{n}_i.
\end{align}
In order to proceed further, 
we will now make a different choice for the vector $\tilde{n}$ 
depending on whether we aim to prove the upper bound or the lower bound.

\paragraph{Upper bound}
In the case of the upper bound, we will choose $\tilde{n}:=n=\xinf (\alpha)-\ga$.

\begin{remark}
We note that our proof approach for the upper bound is different
from the proof strategies in \cite{chou2023less} and \cite{wind2023implicit}.
The essential idea in these works was to compare 
the value of the potential function $F_{\alpha,D}$
(or some surrogate thereof)
at the minimizer $\xinf (\alpha)$
with the value of $F_{\alpha,D}$ at the $\ell^1$-minimizing solution $\ga$.
Using this comparison, 
it was possible to derive upper bounds on $ \norm{ \xinf (\alpha)} - \norm{\ga}_1 $.
In contrast to this, 
the crucial observation in our proof 
is that as in \cref{equ:proofsketch1}
we can split the sum into two parts,
one corresponding to the support of $\ga$, which is $\SP$ and one corresponding to its complement $\SC$.
In this way, we can treat the two parts of the sum differently.
Indeed, for the part corresponding to $\SP$,
we expect that $f_{\alpha,D}$ behaves like a linear function for small sufficiently $\alpha$, 
whereas for the part corresponding to $\SC$ 
we needa different approach.
\end{remark}

Then, from \cref{equ:proofsketch1}
and since $f_{\alpha,D}$ is monotonically increasing 
(as we will show later)
it follows that
\begin{equation}\label{equ:proofsketch4}
	\sum_{i \in \SC} f_{\alpha,D} ( n_i) n_i
	=
	-
	\sum_{i \in \mathcal{S}} 
	f_{\alpha,D} ( \ga_i+n_i ) n_i.
	\le 
	-
	\sum_{i \in \mathcal{S}} f_{\alpha,D} ( \ga_i ) n_i.
\end{equation}
The crucial observation to bound the left-hand side
is that the function $ z \mapsto z f_{\alpha,D} (z) $ is convex,
as we will verify for both $D=2$ and $D \ge 3$. 
This allows us to invoke the following well-known lemma
which is a straightforward generalization of the log sum inequality, 
see, e.g., \cite[Theorem 2.7.1]{coverElementsInformationTheory2006}.
\begin{lemma}\label[lemma]{lemma:generalized_log_sum}
	Let $I$ be a finite index set, $a = (a_i)_{i\in I} \subset \R_{\ge 0}$,
	and $b=(b_i)_{i\in I} \subset \R_+$. Let $A=\sum_{i\in I} a_i$ and $B = \sum_{i\in I} b_i$.
	Let $f\colon [0,\infty) \to \R$ be a function such that $[0,\infty)\ni t\mapsto tf(t)$ is convex. 
	Then it holds that
	\begin{align*}
		\sum_{i\in I}
		a_i f\Big( \frac{a_i}{b_i}\Big)
		\ge A \cdot f\Big(\frac{A}{B} \Big).
	\end{align*}
\end{lemma}
We recall the proof of this lemma,
which proceeds analogously as the proof of the log sum inequality,
see, e.g., \cite[Theorem 2.7.1]{coverElementsInformationTheory2006}.
\begin{proof}
	Jensen's inequality with $\alpha_i = \frac{b_i}{B}$ and $t_i = \frac{a_i}{b_i}$ yields
	\begin{equation*}
		\sum_{i\in I}
		a_i f\left( \frac{a_i}{b_i}	\right) 
		= B\sum_{i\in I} \alpha_i t_i f(t_i)
		\ge B \cdot \Big( \sum_{i\in I} \alpha_i t_i\Big) \cdot f\Big( \sum_{i\in I} \alpha_i t_i \Big)
		= A\cdot  f\Big( \frac{A}{B} \Big).
	\end{equation*}
\end{proof}
By applying this lemma to the sum  corresponding to $\SC$ in \cref{equ:proofsketch4}
with $a_i= \abs{n_i}$ and $b_i=1$,
we obtain that
\begin{equation*}
	\sum_{i \in \SC} f_{\alpha,D} ( n_i) n_i
	=
	\sum_{i \in \SC} f_{\alpha,D} ( \abs{n_i}) \abs{n_i}
	\ge
	f_{\alpha,D} \left( \frac{ \norm{n_{\SC}}_{\ell^1} }{\abs{\SC}}  \right)
	\norm{n_{\SC}}_{\ell^1},
\end{equation*}
where in the first equation we have used that $f_{\alpha,D}$ is an even function.
Inserting this inequality into \cref{equ:proofsketch4} above
and using that $f_{\alpha,D}$ is monotonically increasing,
we obtain that
\begin{align}\label{equ:proofsketch9}
	\norm{n_{\SC}}_{\ell^1}
	&\le
	\abs{\SC}
	\left( f_{\alpha, D} \right)^{-1} 
	\left( 
	\frac{-1}{\norm{n_{\SC}}_{\ell^1}}
    \sum_{i \in \mathcal{S}} f_{\alpha,D} ( \ga_i ) n_i
    \right).
\end{align}
Here, we have made the assumption
that the sum inside of $f_{\alpha, D}^{-1}$ is indeed in the domain of 
$f_{\alpha, D}^{-1}$. 
In our proofs below, we will show that this is indeed the case.
Next,
by using the definition of $\tilde{\varrho}$,
we obtain that
\begin{align*}
	\norm{n}_{\ell^1}
	=
	\norm{n_{\SP}}_{\ell^1} + \norm{n_{\SC}}_{\ell^1}
	\le 
	(1+\tilde{\varrho})
	 \norm{n_{\SC}}_{\ell^1}
	\le 
	(1+\tilde{\varrho})
	\abs{\SC}
	\left( f_{\alpha, D} \right)^{-1} 
	\left( 
	\frac{-1}{\norm{n_{\SC}}_{\ell^1}}
    \sum_{i \in \mathcal{S}} f_{\alpha,D} ( \ga_i ) n_i
    \right).
\end{align*}
It remains to estimate the sum inside of $f_{\alpha, D}^{-1}$,
see \cref{equ:proofsketch9}.
We will sketch the main idea.
Set $ \lambda := \min_{i \in \SP} \abs{\ga_i} $.
(The precise definition of $\lambda$ will be different for $D=2$ and $D \ge 3$. 
However, this choice of $\lambda$ suffices to illustrate the main idea.)
Then we note that
\begin{align*}
	&\frac{-1}{\norm{n_{\SC}}_{\ell^1}}
    \sum_{i \in \mathcal{S}} f_{\alpha,D} ( \ga_i ) n_i\\
	=
	&\frac{-1}{\norm{n_{\SC}}_{\ell^1}}
    \sum_{i \in \mathcal{S}} f_{\alpha,D} ( \abs{\ga_i} )
	\sign (\ga_i) n_i \\
	=
	&\frac{-f_{\alpha,D} ( \lambda )}{\norm{n_{\SC}}_{\ell^1}}
    \sum_{i \in \mathcal{S}} 
	\sign (\ga_i) n_i 
	+
	\frac{-1}{\norm{n_{\SC}}_{\ell^1}}
    \sum_{i \in \mathcal{S}}
	\left( 
		f_{\alpha,D} ( \abs{\ga_i} )
		-
		f_{\alpha,D} ( \lambda )
	\right)
	\sign (\ga_i) n_i\\
	\overleq{(i)}
	& \varrho \cdot f_{\alpha,D} ( \lambda )
	+
	\frac{-1}{\norm{n_{\SC}}_{\ell^1}}
    \sum_{\substack{i\in \SP \\ \sign(n^{\ast}_i)<0}}
	\left( 
		f_{\alpha,D} ( \abs{\ga_i} )
		-
		f_{\alpha,D} ( \lambda )
	\right)
	\sign (\ga_i) n_i\\
	\le& 
	\varrho \cdot f_{\alpha,D} ( \lambda ) 
	+ \varrho^{-}
	\sup_{i \in \SP}  
	\left( 
		f_{\alpha,D} ( \abs{\ga_i} )
		-
		f_{\alpha,D} ( \lambda )
	\right),
\end{align*}
where in inequality $(i)$ 
we used the definition of $\varrho$
and that $f_{\alpha,D}$ is monotonically increasing.
By inserting this into the above inequality
and using the monotonicity of $f_{\alpha,D}^{-1}$,
 we obtain that 
\begin{equation*}
	\norm{n}_{\ell^1}
	\le 
	(1+\tilde{\varrho})
	\abs{\SC}
	\left( f_{\alpha, D} \right)^{-1} 
	\left( 
	\varrho \cdot f_{\alpha,D} ( \lambda ) 
	+ \varrho^{-}
	\sup_{i \in \SP}  
	\left( 
		f_{\alpha,D} ( \abs{\ga_i} )
		-
		f_{\alpha,D} ( \lambda )
	\right)
	\right).
\end{equation*}
To obtain the final upper bound,
we use the asymptotic behavior of $f_{\alpha,D}$
as $\alpha \downarrow 0$.
For further details we refer to the proofs in \cref{section:proofs_Dtwo} 
and \cref{section:Proofs_Dgreatertwo}.
\paragraph{Lower bound}
By 
the definition of $\varrho$
and 
\cref{lemma:normal_coneSimple},
there exists a vector $m \in \ker A \setminus \{ 0 \}$ such that
\begin{equation}
	-\sum_{i\in \SP} \sign \left(\ga_i \right) m_i = \varrho \norm{m_{\SC}}_{\ell^1}.
	\label{equ:proofsketch3}
\end{equation}
The key idea in the proof of the lower bound is to set $\tilde{n}:=m$.
Then, it follows from \cref{equ:proofsketch1} that 
\begin{align}\label{equ:proofsketch2}
	\sum_{i \in \SC} f_{\alpha,D} ( n_i) m_i
	=
    -
	\sum_{i \in \mathcal{S}} f_{\alpha,D} ( \ga_i + n_i) m_i.
\end{align}
Then, since $f_{\alpha,D}$ is an even, monotonically increasing function
(as we will show later in our proofs),
we obtain for the summand on the left-hand side that
\begin{equation*}
	\sum_{i \in \SC} f_{\alpha,D} ( n_i) m_i
	=
	\sum_{i \in \SC} f_{\alpha,D} ( \abs{n_i}) \abs{m_i}
	\le 
	\norm{m_{\SC}}_{\ell^1} f_{\alpha,D} \left( \norm{n_{\SC}}_{\ell^\infty}  \right).
\end{equation*}
Combining this inequality with \cref{equ:proofsketch2} and by rearranging terms 
we obtain that
\begin{equation*}
	\norm{n_{\SC}}_{\ell^\infty}
	\ge
	f_{\alpha,D}^{-1}
	\left(
	\frac{-1}{\norm{m_{\SC}}_{\ell^1}}
	\sum_{i \in \mathcal{S}} f_{\alpha,D} \left( \ga_i + n_i \right) m_i
	\right).
\end{equation*}
As in the case of the upper bound,
for this step to be rigorous
we need to verify the sum inside of $f_{\alpha, D}^{-1}$ is 
indeed in the domain of 
$f_{\alpha, D}^{-1}$. 
This will be done in our proofs below.
In order to proceed further,
we would need to derive a lower bound for the sum inside of $f_{\alpha, D}^{-1}$.
In the following, we sketch our approach.
We again use the notation $\lambda= \min_{i \in \SP} \abs{\ga_i}$.
(Again, we use a different definition of $\lambda$ for $D=2$ and $D \ge 3$
but the ideas outlined below stay the same.)
Then, we can split the sum inside of $f_{\alpha,D}^{-1}$ into two parts,
\begin{align*}
	&\frac{-1}{\norm{m_{\SC}}_{\ell^1}}
	\sum_{i \in \mathcal{S}} f_{\alpha,D} ( \ga_i + n_i ) m_i\\
	=
	&\frac{-f_{\alpha,D} ( \lambda )}{\norm{m_{\SC}}_{\ell^1}}
	\sum_{i \in \mathcal{S}} \text{sign} (\ga_i)  m_i 
	+
	\frac{-1}{\norm{m_{\SC}}_{\ell^1}}
	\sum_{i \in \mathcal{S}}
	\left( 
	f_{\alpha,D} ( \ga_i + n_i )
	-
	f_{\alpha,D} ( \lambda )
	\sign (\ga_i)
	\right)
	m_i\\
	= & 
	\varrho 
	\cdot
	f_{\alpha,D} ( \lambda ) 
	+
	\underset{=:\Delta}{
	\underbrace{
	\frac{-1}{\norm{m_{\SC}}_{\ell^1}}
	\sum_{i \in \mathcal{S}}
	\left( 
	f_{\alpha,D} ( \ga_i + n_i )
	-
	f_{\alpha,D} ( \lambda )
	\sign (\ga_i)
	\right)
	m_i}},
\end{align*}
where in the second equation we have used Equation \eqref{equ:proofsketch3}.
It follows that
\begin{align*}
	\norm{n_{\SC}}_{\ell^\infty}
	\ge
	f_{\alpha,D}^{-1}
	\left(
	\varrho 
	\cdot
	f_{\alpha,D} ( \lambda ) 
	+
	\Delta
	\right).	
\end{align*}
In order to complete the proof,
we will show that $\vert \Delta \vert$ is small as $\alpha \downarrow 0$,
and we will use the asymptotic behavior properties of $f_{\alpha,D}$ as $\alpha \downarrow 0$.
For further details we refer to the proofs in \cref{section:proofs_Dtwo} 
and \cref{section:Proofs_Dgreatertwo}.

%% file: parts/proofs_Dtwo.tex
\subsection{Case $D=2$}\label{section:proofs_Dtwo}

\subsubsection{Some preliminaries}

Before proving our main results in the case $D=2$,
we recall some elementary properties of the function $\arsinh$.
First of all, recall that $\arsinh$ can be expressed as
\begin{equation*}
	\arsinh(t) = \log\big( t + \sqrt{t^2+1}\big)\quad \text{ for } t\in \R.
\end{equation*}
This formula indicates that for $t \gg 1$
the function $t \mapsto \arsinh(t)$ behaves approximately like $ t \mapsto \log(2t)$.
This will be used several times in our proof 
via the following technical inequalities.
\begin{lemma} \label[lemma]{lemma:basic_properties_arsinh}
	The following statements hold.
	\begin{enumerate}[(i)]
		\item For all $t\ge 0$ we have
		\begin{equation} \label{eqn183}
			\arsinh
			\left(\frac{t}{2}\right) 
			= \log(t) + \Delta(t),
		\end{equation}
		where $\Delta$ is a non-negative decreasing function that satisfies 
		\begin{equation} \label{eqn186}
			\Delta(t) \le \frac{1}{t^2}, 
			\qquad\text{and}\qquad
			\exp(\Delta(t)) \le 1 + \frac{1}{t^2}.
		\end{equation}
	
		\item For $s,t \in \R$ with $\sign(t) = \sign(s)$ we have
		\begin{equation} \label{eqn150}
			\abs{\arsinh(t) -\arsinh(s)} \le \abs{ \log\Big(\frac{t}{s}\Big)}
		\end{equation}
	
		\item The map $\R\ni t \mapsto t\cdot \arsinh(t)$ is convex.
	\end{enumerate}
\end{lemma}
We believe that these properties are well-known in the literature.
For the sake of completeness, we provide a proof of this lemma in \cref{section:basic_properties_arsinh}.
\subsubsection{Proof of the upper bound}
\begin{proof}
	If $\xinf=\ga$ there is nothing to show. 
	Assume from now on that $n:= \xinf-\ga\ne 0$. 
	Then it follows from the optimality of $\xinf$ that
	\begin{equation*}
		0 = \left.\frac{\dif}{\dif t}\right|_{t=0} D_{\Ha}(\xinf+tn, 0)
		= \langle \nabla \Ha(\xinf)-\nabla \Ha(0), n\rangle,
	\end{equation*}
	and,
	since $  \nabla \Hb(0)=0$, 
	we have that
	\begin{equation*} 
		\langle \nabla \Hb(n_{\SC}),n_{\SC}\rangle
		=-\langle \nabla \Ha (\ga_{\SP} +n_{\SP}), n_{\SP}\rangle.
	\end{equation*}
	Since $\Ha$ is convex, its gradient $\nabla \Ha$ is monotone. Therefore,
	\begin{equation}\label{eqn240}
		\langle \nabla \Hb(n_{\SC}),n_{\SC}\rangle
		\le -\langle \nabla \Ha (\ga_{\SP}), n_{\SP}\rangle.
	\end{equation}

	In the following, we will estimate the terms in \eqref{eqn240} individually. 
	For the term on the left-hand side of \eqref{eqn240}, we observe first that
	\begin{equation*}
		\begin{split}
		\langle \nabla \Hb(n_{\SC}),n_{\SC}\rangle 
		&= \sum_{i\in \SC} n_i \arsinh\Big(\frac{n_i}{2 \alpha }\Big)
		\overeq{(a)} \sum_{i\in \SC} \abs{n_i} \arsinh\Big(\frac{\abs{n_i}}{2\alpha}\Big)\\
		&\overgeq{(b)} \norm{n_{\SC}}_{\ell^1} \arsinh\Big( \frac{\norm{n_{\SC}}_{\ell^1}}{2 \vert \SC \vert \alpha }\Big).
		\end{split}
	\end{equation*}
	In equation $(a)$ we use that $\arsinh$ is an odd function. 
	Inequality $(b)$ 
	follows from the generalized log-sum inequality,
	see \cref{lemma:generalized_log_sum},
	which is applicable
	since $ t \mapsto t \arsinh (t)$ is convex, see \cref{lemma:basic_properties_arsinh}.
	For the term on the right-hand side of \eqref{eqn240}
	we observe that
	\begin{equation*}
			- \langle \nabla \Ha (\ga_{\SP}), n_{\SP}\rangle 
			= -\sum_{i \in \SP} n_i \arsinh\Big( \frac{\ga_i}{2\alpha}\Big)
			\overeq{(a)} -\sum_{i \in \SP} n_i \sign(\ga_i) \arsinh\Big(\frac{\abs{\ga_i}}{2\alpha}\Big).
	\end{equation*}
    In equality $(a)$ we used that $\arsinh$ is an odd function.
    By combining the last two estimates with \eqref{eqn240},
    we obtain that
    \begin{equation*}
        \norm{n_{\SC}}_{\ell^1} \arsinh\Big( \frac{\norm{n_{\SC}}_{\ell^1}}{2 \vert \SC \vert \alpha }\Big)
		\le -\sum_{i \in \SP} n_i \sign(\ga_i) \arsinh\Big(\frac{\abs{\ga_i}}{2\alpha}\Big).
    \end{equation*}
	Note that we can divide by $ \norm{n_{\SC}}_{\ell^1}$ 
	since $n_{\SC}\ne 0$ due to \cref{lemma:normal_coneSimple}
	and since we assumed that $n \ne 0$.
    Hence, it follows that
    \begin{align}
        \norm{n_{\SC}}_{\ell^1}
        \le
        2 \alpha \vert \SC \vert
        \sinh 
        \left(
		\frac{-1}{  \norm{n_{\SC}}_{\ell^1}}
        \sum_{i \in \SP} n_i \sign(\ga_i) \arsinh\left(\frac{\abs{\ga_i}}{2\alpha}\right)
        \right).
		\label{ineq:internDominik17}
    \end{align}
	Now let $\lambda:= \frac{\min_{i\in \SP} \abs{\ga_i}}{2\alpha}$
    and write $n^{\ast}_i := n_i \sign(\ga_i)$ for $i\in \SP$. 
	It follows that
	\begin{equation} \label{eqn241}
		-\sum_{i \in \SP} n^{\ast}_i \arsinh\left(\frac{\abs{\ga_i}}{2\alpha}\right)
		= \Big(-\sum_{i \in \SP} n^{\ast}_i \Big) \arsinh(\lambda)
			- \sum_{i\in \SP} n^{\ast}_i 
				\Big[
				\arsinh\left(\frac{\abs{\ga_i}}{2 \alpha }\right)
				- \arsinh(\lambda)
				\Big].
	\end{equation}
	For the first summand on the right-hand side of \eqref{eqn241},
    we use the definition of $ \varrho $, see \cref{nullspaceconstants},
	to obtain
	\begin{align}
		\left(-\sum_{i \in \SP} n^{\ast}_i \right) \arsinh(\lambda) 
		 &\le \varrho \norm{n_{\SC}}_{\ell^1}  \arsinh(\lambda) 
		 = \varrho \norm{n_{\SC}}_{\ell^1} \big( \log(2\lambda) + \Delta(2\lambda)\big).
		 \label{ineq:internDominik15}
	\end{align}
	Here, the function $\Delta$ is the function defined in \cref{lemma:basic_properties_arsinh}.
	For the second term on the right-hand side of \eqref{eqn241},
    we use first that 
	$\lambda\le \frac{\abs{\ga_i}}{2 \alpha}$ 
	for all $i\in \SP$  
    combined with the monotonicity of $\arsinh$
    which yields that
	\begin{align}
		- \sum_{i\in \SP} n^{\ast}_i 
			\Big[
			\arsinh\Big(\frac{\abs{\ga_i}}{2\alpha}\Big)
			- \arsinh(\lambda)
			\Big]
		&\le 
		- \sum_{\substack{i\in \SP \\ \sign(n^{\ast}_i)<0}} n^{\ast}_i 
			\Big[
			\arsinh\left(\frac{\abs{\ga_i}}{2 \alpha }\right)
			- \arsinh(\lambda)
			\Big] \nonumber\\
		&\overleq{(a)}
		- \sum_{\substack{i\in \SP \\ \sign(n^{\ast}_i)<0}} n^{\ast}_i 
			\Big[
				\log\left(  
                \frac{\abs{\ga_i}}{ 2\lambda \alpha}
				\right)
			\Big] \nonumber \\
		&\overleq{(b)} \varrho^{-} \norm{n_{\SC}}_{\ell^1} \sup_{i\in \SP} 
			\left[ \log\left(  
                \frac{\abs{\ga_i}}{ 2\lambda \alpha}
				\right) \right] \nonumber \\
		&= \varrho^{-} \norm{n_{\SC}}_{\ell^1}
			\log \left( \kappa_{\ast} \right). \label{ineq:internDominik16}
	\end{align}
    Inequality $(a)$ follows from \cref{lemma:basic_properties_arsinh}, see \cref{eqn150}.
	Inequality $(b)$ is due to the definition of $\varrho^{-}$.
    It follows from \eqref{ineq:internDominik15} and \eqref{ineq:internDominik16} that
    \begin{align*}
    \frac{-1}{\norm{n_{\SC}}_{\ell^1}}
    \sum_{i \in \SP} n^{\ast}_i \arsinh\left(\frac{\abs{\ga_i}}{2\alpha}\right)
    \le
    \varrho  \left( \log \left(2\lambda\right) + \Delta \left(2\lambda \right) \right)
    +
    \varrho^{-} \log \left( \kappa_{\ast} \right).
    \end{align*}
	In combination with \eqref{ineq:internDominik17}
	and since $\sinh$ is increasing,
    this in turn implies that
    \begin{align*}
        \norm{n_{\SC}}_{\ell^1}
        \le
        2 \alpha \abs{\SC}
        \sinh 
        \left(
            \varrho  \left( \log(2\lambda) + \Delta(2\lambda)\right)
            +
            \varrho^{-} \log \left( \kappa_{\ast} \right)
        \right).
    \end{align*}
    Using that $\sinh \le \frac{1}{2}\exp$, we deduce that
	\begin{equation*}
		\norm{n_{\SC}}_{\ell^1}
		\le 
        \alpha \abs{\SC}
        \left(2 \lambda \right)^\varrho
        \kappa_{\ast}^{\varrho^{-}}
        \exp \left(  \varrho \Delta (2\lambda) \right).
    \end{equation*}
	Next, we note that by 
	\cref{eqn186}
	in Lemma \ref{lemma:basic_properties_arsinh}
	we obtain that
	\begin{equation*}
		\exp \left( \varrho \Delta (2\lambda) \right) 
        \le 
        \left(
            1+ \frac{1}{4 \lambda^2}
        \right)^\varrho.
	\end{equation*}
    By combining the last two inequalities
    and using that $\lambda = \frac{\min_{i\in \SP} \abs{\ga_i}}{2\alpha}$
    we obtain that
    \begin{align}
		\norm{n_{\SC}}_{\ell^1}
		\le 
        \alpha^{1-\varrho} \abs{\SC}
        \left(\min_{i\in \SP} \abs{\ga_i}\right)^\varrho
        \kappa_{\ast}^{\varrho^{-}}
        \left( 
            1+ \left( \frac{\alpha}{ \min_{i \in \mathcal{S}} \vert \ga_i \vert } \right)^2
        \right)^{\varrho}.
		\label{equ:internDominik8}
    \end{align}
	Then the claimed upper bound \eqref{eqn:error_bound_sparse_hyperbolic}
	follows from
    the observation that
	\begin{equation*}
			\norm{n}_{\ell^1} 
			= \norm{n_{\SP}}_{\ell^1} + \norm{n_{\SC}}_{\ell^1}
			\le (1+\tilde{\varrho}) \norm{n_{\SC}}_{\ell^1},
	\end{equation*}
	which is a direct consequence of the definition 
	of $\tilde{\varrho}$, see \cref{nullspaceconstants}.
    This completes the proof of the upper bound.
\end{proof}

\subsubsection{Proof of the lower bound}

\begin{proof}
	Define $n := \xinf -\ga \in \ker (A)$. 
	We start with the following observation
	which we will use several times throughout the proof.
	Namely, by using 
	\cref{equ:internDominik8},
	which we have established in the proof of the upper bound,
	we obtain that
	\begin{align}
		\Vert
		n_{\mathcal{S}}
		\Vert_{\ell^{\infty}}
		\le
		&
		\Vert
		n_{\mathcal{S}}
		\Vert_{\ell^{1}}
		\le
		\tilde{\varrho}
		\Vert
		n_{\mathcal{S}^c}
		\Vert_{\ell^{1}} \nonumber \\
		\le
		&\tilde{\varrho}
        \alpha^{1-\varrho} \vert \SC \vert 
		\left( \min_{i\in \SP}\abs{\ga_i} \right)^{\varrho}
		\kappa_{\ast}^{\varrho^{-}}
		 \left(
            1 +
            \frac{\alpha^2}{ \left( \min_{i\in \SP}\abs{\ga_i} \right)^2} 
        \right)^{\varrho}\\
		\le
		& 2\tilde{\varrho}
        \alpha^{1-\varrho} \vert \SC \vert 
		\left( \min_{i\in \SP}\abs{\ga_i} \right)^{\varrho}
		\kappa_{\ast}^{\varrho^{-}} \label{ineq:internDominik18} \\
		\le
		& \frac{\underset{i \in \mathcal{S}}{\min} \ \abs{ \ga_i} }{2}.
		\label{ineq:internDominik1}
	\end{align} 
	Next, note that \cref{lemma:normal_coneSimple} 
	implies the existence of $m \in \ker(A) \setminus \{0\}$ with $m_{\SC}\ne 0$ and 
	\begin{equation}\label{eqn222}
		-\sum_{i \in \SP} \sign(\ga_i) m_i = \varrho \norm{m_{\SC}}_{\ell^1}.
	\end{equation}
	From the optimality of $\xinf$ it follows that
	\begin{equation*}
		0 = \left.\frac{\dif}{\dif t}\right|_{t=0} D_{\Ha}(\xinf+tm,0)
		= \langle \nabla \Ha(\xinf)-\nabla \Ha(0), m\rangle
		= \langle \nabla \Ha(\xinf), m\rangle.
	\end{equation*}
    It follows that
	\begin{equation} \label{eqn221}
		-\langle \nabla \Ha (\ga_{\SP} +n_{\SP}), m_{\SP}\rangle 
		= \langle \nabla \Hb(n_{\SC}),m_{\SC}\rangle.
	\end{equation}
	In the following, we will process the terms in \eqref{eqn221} individually. 
	For the term on the left-hand side,
	we obtain that
	\begin{equation}\label{eqn223}
		\langle \nabla \Hb (\ga_{\SP} +n_{\SP}), m_{\SP}\rangle 
		\le  \langle \nabla \Hb (\ga_{\SP} ), m_{\SP}\rangle 
		+ \norm{ \nabla \Hb(\ga_{\SP}) - \nabla \Hb(\ga_{\SP} + n_{\SP})}_{\ell^{\infty}} \norm{m_{\SP}}_{\ell^1}
	\end{equation}
	Next, we observe
	that we have $\sign(\ga_i) = \sign(\ga_i +n_i)$
	for all $i\in \SP$
	due to \eqref{ineq:internDominik1}.	
	Inserting the definition of $ \nabla \Hb$ 
	and using the definition of $\tilde{\varrho}$
	we obtain that
	\begin{align}
			 \norm{ \nabla \Hb(\ga_{\SP}) - \nabla \Hb(\ga_{\SP} + n_{\SP})}_{\ell^{\infty}} \norm{m_{\SP}}_{\ell^1}
			\le 
			&\sup_{i\in \SP} \abs{ 
				\arsinh\Big(\frac{\ga_i}{2\alpha}\Big) - \arsinh\Big(\frac{\ga_i+n_i}{2\alpha}\Big)
			}
			\tilde{\varrho} \norm{m_{\SC}}_{\ell^1} \nonumber \\
			= 
			&\delta_1 \norm{m_{\SC}}_{\ell^1},\label{eqn224}
	\end{align}
	where
	\begin{equation*}
		\delta_1
		:=
		\tilde{\varrho}
		\cdot
		\max_{i\in \SP} \abs{ 
			\arsinh\Big(\frac{\ga_i}{2\alpha}\Big) - \arsinh\Big(\frac{\ga_i+n_i}{2\alpha}\Big)
		}.
	\end{equation*}
	Now let $\lambda:= \frac{ \norm{\ga}_{\ell^\infty}}{\alpha}$ 
    and $m^{\ast}_i:= \sign(\ga_i)m_i$ 
	for $i\in \SP$.
	It follows that
	\begin{align}
			\langle \nabla \Ha (\ga_{\SP} ), m_{\SP}\rangle
			&= \sum_{i \in \SP} m_i \arsinh\left( \frac{\ga_i}{2 \alpha} \right)
			= \sum_{i \in \SP} m^{\ast}_i \arsinh \left( \frac{\abs{\ga_i}}{2 \alpha}\right) \nonumber \\
			&= \left(\sum_{i \in \SP} m^{\ast}_i\right) 
            \arsinh \left( \frac{\lambda}{2}\right)
			+ \sum_{i \in \SP} m^{\ast}_i
			\Big[ 
			\arsinh \left(\frac{\abs{\ga_i}}{2 \alpha}\right)
			-\arsinh \left( \frac{\lambda}{2}\right) 
			\Big] \nonumber \\
			&\overleq{(a)}
			-\varrho \norm{m_{\SC}}_{\ell^1} \arsinh\left( \frac{\lambda}{2}\right)
			+ \sum_{\substack{i \in \SP\\ m^{\ast}_i<0}} m^{\ast}_i
			\Big[ 
			\arsinh\left(\frac{\abs{\ga_i}}{2 \alpha}\right)
			-\arsinh\left( \frac{\lambda}{2}\right)
			\Big] \nonumber \\
			&\overleq{(b)}
			-\varrho \norm{m_{\SC}}_{\ell^1}\arsinh \left( \frac{\lambda}{2}\right)
			-
			\Big(
			\sum_{\substack{i \in \SP\\ m^{\ast}_i<0}} 
			m^{\ast}_i
			\Big)
			\cdot
			\sup_{i\in \SP} 
			\left[
			\log
			\left( 
			\frac{ \norm{\ga}_{\ell^\infty}}{2 \alpha}
			\cdot \frac{2 \alpha}{\abs{\ga_i}}
			\right)
			\right]
			\nonumber\\
			&\overleq{(c)}
			-\varrho \norm{m_{\SC}}_{\ell^1} \arsinh \left( \frac{ \norm{\ga}_{\ell^\infty}}{2 \alpha}\right)
			+\varrho^{-}\norm{m_{\SC}}_{\ell^1} \log\left(\kappa_{\ast} \right).\label{eqn225}
	\end{align}
    Inequality $(a)$ follows from 
	the definition of $\varrho$,
	$\lambda \ge \frac{\abs{\ga_i}}{\alpha}$ 
    for all $i\in \SP$,
	and the monotonicity of $\arsinh$.
	For inequality $(b)$ we used
	\cref{lemma:basic_properties_arsinh}, see \cref{eqn150}.
	For inequality $(c)$ we used the definition of $\varrho^{-}$.
	Combining \eqref{eqn224} and \eqref{eqn225} with \eqref{eqn223}, we infer that
	\begin{equation}\label{eqn226}
		-\langle \nabla \Ha (\ga_{\SP} +n_{\SP}), m_{\SP}\rangle
		\ge \norm{m_{\SC}}_{\ell^1} 
		\Big[  -\delta_1
		+ \varrho \arsinh
        \left( \frac{ \norm{\ga}_{\ell^\infty}}{2\alpha }
        \right)
		- \varrho^{-} \log
        \left(
            \kappa_{\ast}
        \right)
		\Big].
	\end{equation}
	For the term on the right-hand side of \eqref{eqn221},
    we use $\abs{\arsinh(t)}=\arsinh(\abs{t})$ to obtain
	\begin{align} 
			\langle \nabla \Hb(n_{\SC}),m_{\SC}\rangle 
			&= \sum_{i\in \SC} m_i \arsinh
            \left( \frac{n_i}{2\alpha}\right)
			\le \norm{m_{\SC}}_{\ell^1} \sup_{i \in \SC} 
            \abs{ \arsinh\left( \frac{n_i}{2\alpha}\right)} \nonumber \\
			&\le \norm{m_{\SC}}_{\ell^1} 
			\arsinh
            \left( 
                \frac{ \norm{n_{\SC}}_{\ell^{\infty}}}{ 2\alpha } 
            \right).\label{eqn228}
	\end{align}
	Inserting \eqref{eqn226} and \eqref{eqn228} into \eqref{eqn221},
	we obtain that
	\begin{align}
			\arsinh
            \left( \frac{ \norm{n_{\SC}}_{\ell^{\infty}}}{ 2\alpha} \right)
			\ge \varrho \arsinh
            \left( \frac{ \norm{\ga}_{\ell^\infty}}{2 \alpha}\right)
			-\varrho^{-} \log \left( \kappa_{\ast} \right)
			-\delta_1.
			\label{eqn235}
	\end{align}
    It follows that 
	\begin{align}
		 \norm{n_{\SC}}_{\ell^{\infty}}
		 \ge
		 &2\alpha
		 \sinh 
		 \left(
            \varrho \arsinh \left( \frac{ \norm{\ga}_{\ell^\infty}}{2 \alpha}\right)
			-\varrho^{-} \log \left( \kappa_{\ast} \right)
			-\delta_1 
		 \right) \nonumber \\
		 \overgeq{(a)}
		 &2\alpha
		 \sinh 
		 \left(
            \varrho \log \left( \frac{ \norm{\ga}_{\ell^\infty}}{ \alpha}\right)
			-\varrho^{-} \log \left( \kappa_{\ast} \right)
			-\delta_1 
		 \right) \nonumber \\
		 =
		 &\alpha
		 \exp
		 \left(
            \varrho \log \left( \frac{ \norm{\ga}_{\ell^\infty}}{ \alpha}\right)
			-\varrho^{-} \log \left( \kappa_{\ast} \right)
			-\delta_1 
		 \right) \nonumber \\
		 &- \alpha
		 \exp
		 \left(
            -\varrho \log \left( \frac{ \norm{\ga}_{\ell^\infty}}{ \alpha}\right)
			+\varrho^{-} \log \left( \kappa_{\ast} \right)
			+\delta_1 
		 \right) \nonumber
		 \\
		 =
		 &\alpha^{1-\varrho} \norm{\ga}_{\ell^\infty}^{\varrho} 
		 \kappa_{\ast}^{-\varrho^-} \exp \left( - \delta_1 \right)
		 -
		 \alpha^{1+\varrho}\norm{\ga}_{\ell^\infty}^{-\varrho} \kappa_{\ast}^{\varrho^-}
		 \exp \left( \delta_1 \right) \nonumber \\
		 =
		 &\alpha^{1-\varrho} \norm{\ga}_{\ell^\infty}^{\varrho} \kappa_{\ast}^{-\varrho^-}
		 \left(
			\exp \left( - \delta_1 \right)
			- \frac{\alpha^{2\varrho}}{\norm{\ga}_{\ell^\infty}^{2\varrho}}
			\kappa_{\ast}^{2\varrho^-}
			\exp \left(\delta_1 \right)
		 \right) \nonumber \\
		 =
		 &\alpha^{1-\varrho} \norm{\ga}_{\ell^\infty}^{\varrho} \kappa_{\ast}^{-\varrho^-}
		 \exp \left( \delta_1 \right)
		 \left(
			\exp \left(-2\delta_1\right)
			- \frac{\alpha^{2\varrho}}{\norm{\ga}_{\ell^\infty}^{2\varrho}}
			\kappa_{\ast}^{2\varrho^-}
		 \right) \nonumber \\
		 \ge 
		 &\alpha^{1-\varrho} \norm{\ga}_{\ell^\infty}^{\varrho} \kappa_{\ast}^{-\varrho^-}
		 \left(
			\exp \left(-2\delta_1\right)
			- \frac{\alpha^{2\varrho}}{\norm{\ga}_{\ell^\infty}^{2\varrho}}
			\kappa_{\ast}^{2\varrho^-}
		 \right). \label{equ:internDominik7}
	\end{align}
	Inequality $(a)$ follows from \cref{lemma:basic_properties_arsinh},
	see \cref{eqn183}.
    To obtain the final bound
	it remains to bound the term $\exp(\delta_1)$ from below.
	Due to \cref{ineq:internDominik1} 
	we have $\sign (\ga_i) = \sign (\ga_i+ n_i) $
	for all $i \in \mathcal{S}$.
	Then we obtain
	using the definition of $\delta_1$ and \cref{lemma:basic_properties_arsinh} that
	\begin{equation*}
		\delta_1
		\le
		\tilde{\varrho}
		\cdot
		\underset{i \in \mathcal{S}}{\max}
		\Big\vert
		\log 
		\left(
			1+ \frac{n_i}{\ga_i}
		\right)
		\Big\vert.
	\end{equation*}
	Next, we choose an index $\tilde{i} \in \SP$
	which maximizes the right-hand side in the last line. 
	If $\frac{n_{\tilde{i}}}{\ga_{\tilde{i}}} \le 0$,
	we obtain
	that
	\begin{equation*}
		\exp
		\left( - 2 \delta_1  \right)
		\ge 
		\exp
		\left(2\tilde{\varrho} \log\left( 1+ \frac{n_{\tilde{i}}}{\ga_{\tilde{i}}}\right)
		\right)
		\overgeq{(a)}
		\exp
		\left(
			\frac{4 \tilde{\varrho} n_{\tilde{i}}}{\ga_{\tilde{i}}}
		\right)
		\overgeq{(b)}
		1+\frac{4 \tilde{\varrho} n_{\tilde{i}}}{\ga_{\tilde{i}}} 
		\ge
		1-\frac{4 \tilde{\varrho} \norm{ n_{\mathcal{S}}}_{\ell_{\infty}} }{ \underset{i \in \mathcal{S}}{\min} \ \abs{\ga_{i}}},
	\end{equation*}
	where in inequality (a) we have used 
	the elementary inequality $\log (1+x) \ge \frac{x}{1-x}$
	and that $\frac{n_{\tilde{i}}}{ \ga_{\tilde{i}}} \in (-1/2,1/2)$ due to \eqref{ineq:internDominik1}.
	In inequality $(b)$ we used that $\exp(x) \ge 1+x$.
	If $\frac{n_{\tilde{i}}}{\ga_{\tilde{i}}} > 0$, 
	we obtain in a similar way that
	\begin{align*}
		\exp
		\left( -2  \delta_1 \right)
		\ge 
		&\exp
		\left(
		-2 \tilde{\varrho} 
		\log \left( 1+ \frac{n_{\tilde{i}}}{\ga_{\tilde{i}}}\right)
		\right)
		\ge 
		\exp
		\left(
		-2 \tilde{\varrho} 
		\log 
		\left( 
	    1+ \frac{\Vert n_{\mathcal{S} \Vert_{\ell^{\infty}} }}{ \underset{i \in \mathcal{S}}{\min} \ \abs{\ga_{i}} }
		\right)
		\right)
		\ge
		1 - 
		\frac{2 \tilde{\varrho} \Vert n_{\mathcal{S} \Vert_{\ell^{\infty}} }}{ \underset{i \in \mathcal{S}}{\min} \ \abs{\ga_{i}}},
	\end{align*}
	where in the last inequality we used that $\log(1+x) \le x$ for $x>-1$
	and that $ \exp(x) \ge 1+x $ for all $x \in \R$.
	By combining the last two inequalities we obtain that
	\begin{align*}
		\exp\left( -2 \delta_1 \right) 
		\ge
		1 - 
		\frac{4 \tilde{\varrho} \Vert n_{\mathcal{S}} \Vert_{\ell^{\infty}}}
		{ \underset{i \in \mathcal{S}}{\min} \ \abs{\ga_{i}}}
		\overgeq{\eqref{ineq:internDominik18}}
		1-8\tilde{\varrho}^2
        \abs{\SC} 
		\kappa_{\ast}^{\varrho^{-}}
		\frac{\alpha^{1-\varrho}}{ \left( \underset{i \in \mathcal{S}}{\min} \ \abs{\ga_{i}} \right)^{1-\varrho} }.
	\end{align*}
    By inserting this inequality into \cref{equ:internDominik7},
	we obtain that
    \begin{align*}
    \norm{n_{\SC}}_{\ell^{\infty}}
    \ge 
	&\alpha^{1-\varrho} \norm{\ga}_{\ell^\infty}^{\varrho} \kappa_{\ast}^{-\varrho^-}
	\left(
		1-8\tilde{\varrho}^2
        \abs{\SC}
		\kappa_{\ast}^{\varrho^{-}}
		\frac{\alpha^{1-\varrho}}{ \left( \underset{i \in \mathcal{S}}{\min} \ \abs{\ga_{i}} \right)^{1-\varrho} }
		-
		\frac{\alpha^{2\varrho}}{\norm{\ga}_{\ell^\infty}^{2\varrho}}
			\kappa_{\ast}^{2\varrho^-}
	\right).
	\end{align*}
	This implies the claimed inequality in Part b) of \cref{theorem:upper_bound_+-}.
\end{proof}

%% file: parts/proofs_Dgreatertwo.tex
\subsection{Case $D\ge 3$}\label{section:Proofs_Dgreatertwo}

\subsubsection{Some preliminaries}
Let $D\in \N$ with $D\ge 3$, and let $\gamma:=\frac{D-2}{D}$. 
Recall that $Q^D_{\alpha} \colon \R^d \to \R$ is given by
\begin{equation*}
	Q^D_{\alpha}(x) = \sum_{i=1}^d \alpha \cdot q_D\Big(\frac{x_i}{\alpha}\Big).
\end{equation*}
Here, we have
\begin{equation*}
	q_D(u) = \int_0^u h_D^{-1}(z)\dif z,
\end{equation*}
where
\begin{equation} \label{eqn_def_h_D}
	h_D(z) \colon (-1,1)\to \R,\quad z \mapsto (1-z)^{-\frac{D}{D-2}} - (1+z)^{-\frac{D}{D-2}}.
\end{equation}
Our first technical lemma allows us to simplify 
the expression $ D_{\QQ} (x,0) $.
\begin{lemma} \label[lemma]{lemma:basic_properties_Qalpha_1}
	Let $D\ge 3$, $\alpha>0$, and $x\in \R^d$.
	Then it holds that
	\begin{equation*}
		D_{\QQ}(x,0) = \QQ(x).
	\end{equation*}
\end{lemma}
\begin{proof}
	By definition, we have
	\begin{equation*}
		D_{\QQ}(x,0) = \QQ(x) - \QQ(0) + \langle \nabla \QQ(0),x-0\rangle.
	\end{equation*}
	Furthermore, we have
	\begin{equation*}
		\frac{\partial}{\partial x_i} Q^D_{\alpha}(x) 
		= q_D'\Big(\frac{x_i}{\alpha}\Big)
		= h_D^{-1}\Big(\frac{x_i}{\alpha}\Big).
	\end{equation*}
	Since $\hh(0)= 0$ it follows that $\hh^{-1}(0)=0$ and so $\nabla\QQ(0)=0$. Furthermore, 
	we have that $\QQ(0)=0$.
	Thus, the proof is complete.
\end{proof}
For the proofs of the following technical lemmas, we refer to \cref{sec:basic_properties_Dge3}.
The next lemma gathers some basic properties of the functions $\hh$ and $q_D$.
\begin{lemma} \label[lemma]{lemma:basic_properties_q_h_g_2}
	Let $D\in \N$ with $D\ge3$. 
	\begin{enumerate}[(i)]
		\item $\hh$ is smooth, odd, and increasing. Furthermore, it is convex on $[0,1)$.
		\item $\hh^{-1}$ is smooth, odd, and increasing. Furthermore, it is concave on $[0,\infty)$.
		\item $\qq$ is smooth, even, and convex. Furthermore, it is increasing on $[0,\infty)$.
	\end{enumerate}
\end{lemma}
As in the case $D=2$,
a key step in the proof of the upper bound lies
in using the following generalized log sum inequality,
see \cref{lemma:generalized_log_sum}.
The following \cref{lemma:almost_convex} shows 
that the generalized log sum inequality, see \cref{lemma:generalized_log_sum},
is also applicable in the case $D\ge 3$.
\begin{lemma} \label[lemma]{lemma:almost_convex}
	Let $D\in \N$ with $D\ge 3$. Then the map 
	\begin{equation*}
		[0,\infty) \to \R,\quad t\mapsto t \hh^{-1}(t)
	\end{equation*}
	is convex.
\end{lemma}

We will also need the following inequalities
which are useful for describing the asymptotic behavior of $\hh$, $\hh'$, and $\hh^{-1}$.
\begin{lemma}\cite[Proposition 3.3]{wind2023implicit} \label[lemma]{lemma:WindHansen}
	For all $u\in (0,\infty)$, we have
    \begin{equation*}
        1 - u^{-\gamma} \le \hh^{-1}(u) \le 1 - (u+1)^{-\gamma}. 
    \end{equation*}
\end{lemma}

\begin{lemma}\label[lemma]{lemma:basic_inequalities_deep_unique}
	Let $D\in \N$ with $D\ge 3$ and $\gamma:=\frac{D-2}{D}$.
	\begin{enumerate}[(i)]
		\item For all $z\in [0,1)$, we have
		\begin{equation}\label{eqn615}
			\hh'(z) \le \frac{2}{\gamma}(1-z)^{-\frac{1}{\gamma}-1}.
		\end{equation}

		\item For all $0<u,v<\infty$, we have 
		\begin{equation} \label{eqn622b}
			\abs{\hh^{-1}(u)-\hh^{-1}(v)} \le  \frac{\gamma}{\big( \min\{u,v\}\big)^{1+\gamma}}\abs{u-v}.
		\end{equation}
	\end{enumerate}
\end{lemma}

\subsubsection{Proof of the upper bound}
In this section, we prove the upper bound in \cref{theorem:upper_bound_deep}.

\begin{proof}
	Let $n:= \xinf -\ga\in \ker(A)$. By definition of $\tilde{\varrho}$, we have
	\begin{equation} \label{eqn2003}
		\norm{n}_{\ell^1} 
		= \norm{n_{\SC}}_{\ell^1} + \norm{n_{\SP}}_{\ell^1}
		\le \big(1 + \tilde{\varrho}\big) \norm{n_{\SC}}_{\ell^1}.
	\end{equation}
	Thus, to show the claim, we need to derive a bound on $\norm{n_{\SC}}_{\ell^1}$.
	
	We have $A(\xinf +tn)=y$ for all $t\in \R$. Furthermore, $D_{\QQ}(\cdot,0)$ is differentiable, see \cref{lemma:basic_properties_q_h_g_2}. 
	Using the optimality of $\xinf$ at $(a)$ and \cref{lemma:basic_properties_Qalpha_1} at $(b)$, it follows that
	\begin{align*} 
			0 
			& \overeq{(a)} \left.\frac{\dif}{\dif t}\right|_{t=0} D_{\QQ}(\xinf + t n,0)
			\overeq{(b)} \langle \nabla \QQ(\xinf),n\rangle                                
			\nonumber\\
			  & = \langle \nabla \QQ(\ga_{\SP}+n_{\SP}),n_{\SP}\rangle
			+ \langle \nabla \QQ(n_{\SC}),n_{\SC}\rangle
	\end{align*}
	Since $\QQ$ is convex, its gradient $\nabla \QQ$ is monotone. Therefore,
	\begin{equation*}
		\langle \nabla \QQ(\ga_{\SP}+n_{\SP}),n_{\SP}\rangle
		\ge \langle \nabla \QQ(\ga_{\SP}),n_{\SP}\rangle
	\end{equation*}
	We deduce that
	\begin{equation} \label{eqn600}
		-\langle \nabla \QQ(\ga_{\SP}),n_{\SP}\rangle
		\ge  \langle \nabla \QQ(n_{\SC}),n_{\SC}\rangle
	\end{equation}
	In the following, we will process the two terms in \eqref{eqn600} individually. 
	
	First, we derive an upper bound for left-hand side of \eqref{eqn600}. Define $\na_{\SP} := n_{\SP} \hada \sign(\ga_{\SP})$. Using the fact that $\hh^{-1}$ is an odd function, see \cref{lemma:basic_properties_q_h_g_2} at $(a)$, we have
	\begin{equation} \label{eqn601}
			-\langle \nabla \QQ(\ga_{\SP}),n_{\SP}\rangle
			= -\sum_{i\in \SP} \hh^{-1} \Big(\frac{\ga_i}{\alpha}\Big) n_i
			\overeq{(a)} -\sum_{i\in \SP} \hh^{-1} \Big(\frac{\abs{\ga_i}}{\alpha}\Big) \na_i.
	\end{equation}
	To estimate the right-hand side of \eqref{eqn601} from above, let $\lambda_{\min}:= \frac{\min_{i\in \SP}\abs{\ga_i}}{\alpha}$ and $\lambda_{\max}:= \frac{\max_{i\in \SP}\abs{\ga_i}}{\alpha}$.
	Using the monotonicity of $\hh^{-1}$ at $(a)$ and $(c)$, and the definitions \eqref{nullspaceconstants} of $\varrho$ and $\varrho^{-}$ at $(b)$, we infer that
	\begin{align}
			-\sum_{i\in \SP} \hh^{-1} \Big(\frac{\abs{\ga_i}}{\alpha}\Big) \na_i
			 & = -\sum_{i\in \SP} \hh^{-1}(\lambda_{\min}) \na_i
			- \sum_{i\in \SP} 
				\Big[
				\hh^{-1} \Big(\frac{\abs{\ga_i}}{\alpha}\Big)
				- \hh^{-1}(\lambda_{\min})
				\Big] \na_i                                                                 
				\nonumber\\
			 & \overleq{(a)} -\hh^{-1}(\lambda_{\min}) \sum_{i\in \SP}  \na_i
			- \sum_{\substack{i\in \SP \\ \na_i<0}}
				\Big[
				\hh^{-1} \Big(\frac{\abs{\ga_i}}{\alpha}\Big)
				- \hh^{-1}(\lambda_{\min})
				\Big]  \na_i                                                                
				\nonumber\\
			 & \overleq{(b)}  \varrho \norm{n_{\SC}}_{\ell^1} \hh^{-1}(\lambda_{\min})
			+ \varrho^{-} \norm{n_{\SC}}_{\ell^1} 
				\sup_{i\in \SP} 
					\abs{ 
					\hh^{-1} \Big(\frac{\abs{\ga_i}}{\alpha}\Big)
					- \hh^{-1}(\lambda_{\min}) 
					}                                                
				\nonumber\\
			 & \overeq{(c)}  \varrho \norm{n_{\SC}}_{\ell^1} \hh^{-1}(\lambda_{\min})
			+ \varrho^{-} \norm{n_{\SC}}_{\ell^1}  
				\Big[  
				\hh^{-1} (\lambda_{\max})
				- \hh^{-1}(\lambda_{\min}) 
				\Big]. 
				\label{eqn602}
	\end{align}

	Next, we show a lower bound for the right-hand side of \eqref{eqn600}. 
	Recall that the map $u \mapsto u\hh^{-1}(u)$ is convex, 
	see \cref{lemma:almost_convex}. 
	Therefore, the generalized log sum inequality, \cref{lemma:generalized_log_sum}, is applicable. 
	Using that $\hh^{-1}$ is an odd function at $(a)$ 
	and the generalized log sum inequality at $(b)$, we have
	\begin{equation} \label{eqn603}
		\langle \nabla \QQ(n_{\SC}),n_{\SC}\rangle
		= \sum_{i\in \SC} \hh^{-1} \Big(\frac{n_i}{\alpha}\Big) n_i
		\overeq{(a)} \sum_{i\in\SC} \hh^{-1} \Big(\frac{\abs{n_i}}{\alpha}\Big) \abs{n_i}
		\overgeq{(b)} \norm{n_{\SC}}_{\ell^1} \hh^{-1} \Big(\frac{\norm{n_{\SC}}_{\ell^1}}{\alpha\abs{\SC}}\Big).
	\end{equation}

	Now that bounds for the terms in \eqref{eqn600} are established, we proceed to derive an upper bound for $\norm{n_{\SC}}_{\ell^1}$. Combining \eqref{eqn601} with \eqref{eqn602}, and inserting this together with \eqref{eqn603} into \eqref{eqn600}, we deduce that
	\begin{equation*}
		\norm{n_{\SC}}_{\ell^1} \hh^{-1} \Big(\frac{\norm{n_{\SC}}_{\ell^1}}{\alpha\abs{\SC}}\Big)
		\le  \varrho \norm{n_{\SC}}_{\ell^1} \hh^{-1}(\lambda_{\min})
		+\varrho^{-} \norm{n_{\SC}}_{\ell^1}  \Big[  \hh^{-1} (\lambda_{\max})
			-\hh^{-1}(\lambda_{\min}) \Big].
	\end{equation*}
	Dividing both sides by $\norm{n_{\SC}}_{\ell^1}$ we obtain
	\begin{equation} \label{eqn2000}
		\hh^{-1} \Big(\frac{\norm{n_{\SC}}_{\ell^1}}{\alpha\abs{\SC}}\Big)
		\le \varrho +\delta_1,
	\end{equation}
	where
	\begin{equation*}
		\delta_1 := \varrho^{-} \cdot \big( \hh^{-1}(\lambda_{\max})-\hh^{-1}(\lambda_{\min}) \big)
		+\varrho \cdot \big( \hh^{-1}(\lambda_{\min}) -1 \big).
	\end{equation*}

	Assume for now that $\delta_1$ is sufficiently small so that $\varrho + \delta_1<1$. Then both sides of \eqref{eqn2000} are in the domain of $\hh$. Applying $h_D$ to both sides of \eqref{eqn2000} and using a Taylor expansion around $\varrho$, we infer that
	\begin{equation} \label{eqn2004}
		\norm{n_{\SC}}_{\ell^1}
		\le \alpha\abs{\SC} \cdot \hh(\varrho +\delta_1)
		= \alpha\abs{\SC} \cdot \big( \hh(\varrho) +\hh'(\xi) \cdot \delta_1 \big)
	\end{equation}
	for some $\xi \in (\varrho, \varrho+ \delta_1)$.

	To finish the proof, we need to check the assumption $\varrho+\delta_1<1$ and to derive an upper bound for $\hh'(\xi) \cdot \delta_1$.
	Let $\eps:= \frac{\alpha}{\mg}$.
	Using that $\hh^{-1}\le 1$ at $(a)$ and applying \cref{lemma:WindHansen} at $(b)$, we obtain
	\begin{equation} \label{eqn2001}
		\delta_1
		\overleq{(a)} \varrho^{-} \cdot \big( 1 - \hh^{-1}(\lambda_{\min}) \big)
		\overleq{(b)} \varrho^{-} \cdot \lambda_{\min}^{-\gamma}
		= \varrho^{-} \cdot \eps^{\gamma}.
	\end{equation}
	By assumption \eqref{eqn1001}, we get $\delta_1 <(1-\varrho)\cdot \frac{\gamma}{4}$. Since $\gamma<1$, it follows that $\varrho + \delta_1 <1$, and thus inequality \eqref{eqn2004} holds.
	Using the monotonicity of $\hh'$ at $(a)$, inequality \eqref{eqn615} of \cref{lemma:basic_inequalities_deep_unique} at $(b)$, assumption \eqref{eqn1001} at $(c)$, and \cref{lemma:technical_1} at $(d)$,	
	we infer that
	\begin{align} 
		\hh'(\xi)
		&\overleq{(a)} \hh'(\varrho+\varrho^{-}\cdot \eps^{\gamma})
		\overleq{(b)}  \frac{2}{ \gamma (1-\varrho-\varrho^{-}\cdot \eps^{\gamma})^{\frac{1}{\gamma}+1}} 
		\nonumber\\
		&=  \frac{2}{ \gamma (1-\varrho)^{\frac{1}{\gamma}+1}} 
			\cdot \Big( 1- \frac{\varrho^{-} \cdot \eps^{\gamma}}{1-\varrho} \Big)^{-\frac{1}{\gamma}-1}
		\overleq{(c)} \frac{2}{\gamma} \frac{1}{(1-\varrho)^{\frac{1}{\gamma}+1}} 
			\cdot \Big( 1- \frac{\gamma}{4} \Big)^{-\frac{1}{\gamma}-1}
		\nonumber\\
		&\overleq{(d)}  \frac{4}{\gamma(1-\varrho)^{\frac{1}{\gamma}+1}} .
		\label{eqn2002}
	\end{align}
	Finally, we insert \eqref{eqn2001} and \eqref{eqn2002} into \eqref{eqn2004} and obtain
	\begin{equation} \label{eqn2005}
		\norm{n_{\SC}}_{\ell^1} 
		\le \alpha\abs{\SC} \cdot \Big( \hh(\varrho) + 
		 \frac{4 \varrho^{-}}{ \gamma (1-\varrho)^{\frac{1}{\gamma}+1}}  \cdot \eps^{\gamma} \Big).
	\end{equation}
	The inequality \eqref{eqn1002} now follows from \eqref{eqn2003} and \eqref{eqn2005}.
\end{proof}

\subsubsection{Proof of the lower bound}

In this section, we prove the lower bound in \cref{theorem:upper_bound_deep}.

\begin{proof}
	Let $n:= \xinf -\ga$. \cref{lemma:normal_coneSimple} implies that there exists $m \in \ker(A)\setminus\{0\}$ such that $m_{\SC}\ne 0$ and 
	\begin{equation} \label{eqn607}
		-\sum_{i\in \SP} \ma_i = \varrho \norm{m_{\SC}}_{\ell^1},
	\end{equation}
	where $\ma_{\SP} := m_{\SP} \hada \sign(\ga_{\SP})$. 
	By optimality of $\xinf$, \cref{lemma:basic_properties_Qalpha_1}, and the identity $\xinf = \ga+ n$, we have
	\begin{align*} 
			0 &= \left.\frac{\dif}{\dif t}\right|_{t=0} D_{\QQ}(\xinf +tm,0)
			= \langle \nabla \QQ(\xinf),m\rangle
			\nonumber\\
			&= \langle \nabla \QQ(\ga_{\SP}+n_{\SP}),m_{\SP}\rangle
			+ \langle \nabla \QQ(n_{\SC}),m_{\SC}\rangle.
	\end{align*}
	Therefore,
	\begin{equation} 			\label{eqn613}
		- \langle \nabla \QQ(\ga_{\SP}+n_{\SP}),m_{\SP}\rangle = \langle \nabla \QQ(n_{\SC}),m_{\SC}\rangle.
	\end{equation}
	In the following, we will estimate the two terms in \eqref{eqn613} individually, and deduce a lower bound for $\norm{n_{\SC}}_{\ell^{\infty}}$.
	
	For the term on the right-hand side of \eqref{eqn613}, since $\hh^{-1}$ is odd and increasing, see \cref{lemma:basic_properties_q_h_g_2}, we have
	\begin{equation} \label{eqn612}
		\langle \nabla \QQ(n_{\SC}),m_{\SC}\rangle
		= \sum_{i\in \SC} \hh^{-1}\Big(\frac{n_i}{\alpha}\Big) m_i
		\le \norm{m_{\SC}}_{\ell^1} \sup_{i\in \SC} \abs{ \hh^{-1}\Big( \frac{n_i}{\alpha}\Big) }
		= \norm{m_{\SC}}_{\ell^1} \hh^{-1}\Big( \frac{\norm{n_{\SC}}_{\ell^{\infty}}}{\alpha}\Big) .
	\end{equation}
	For the term on the left-hand side of \eqref{eqn613}, since $\hh^{-1}$ is odd and increasing, we have 
	\begin{equation} \label{eqn605}
		\begin{aligned}
			\langle \nabla \QQ(\ga_{\SP}+n_{\SP}),m_{\SP}\rangle 
			&= \sum_{i\in \SP} \hh^{-1}\Big(\frac{\ga_i + n_i}{\alpha}\Big) m_i
			= \sum_{i\in \SP} \hh^{-1}\Big(\frac{\abs{\ga_i} + \na_i}{\alpha}\Big) \ma_i\\
			&= \hh^{-1}(\eps^{-1}) \sum_{i\in \SP} \ma_i
			+ \sum_{i\in \SP} \Big[ 
			\hh^{-1}\Big( \frac{\abs{\ga_i}+\na_i}{\alpha}\Big)
			- \hh^{-1}(\eps^{-1})
			\Big] \ma_i,
		\end{aligned}
	\end{equation}
	where $\eps := \frac{\alpha}{\mg}$. We use \eqref{eqn607} to obtain 
	\begin{equation} \label{eqn608}
		\hh^{-1}(\eps^{-1}) \sum_{i\in \SP} \ma_i 
		= -\varrho \norm{m_{\SC}}_{\ell^1} \hh^{-1}(\eps^{-1})
		= -\varrho \norm{m_{\SC}}_{\ell^1} \cdot (1-\delta_1),
	\end{equation}
	where 
	\begin{equation*}
		\delta_1:= 1- \hh^{-1}(\eps^{-1}).
	\end{equation*} 
	Using the definition of $\tilde{\varrho}$, we infer that
	\begin{equation}\label{eqn609}
		\sum_{i\in \SP} \Big[ 
		\hh^{-1}\Big( \frac{\abs{\ga_i}+\na_i}{\alpha}\Big)
		- \hh^{-1}\Big( \varepsilon^{-1} \Big)
		\Big] \ma_i
		\le \tilde{\varrho} \norm{m_{\SC}}_{\ell^1} \delta_2,
	\end{equation}
	where
	\begin{equation*}
		\delta_2:= \max_{i\in \SP}\abs{ \hh^{-1}\Big( \frac{\abs{\ga_i}+\na_i}{\alpha}\Big)
			- \hh^{-1}\Big(\varepsilon^{-1}\Big)}.
	\end{equation*}	
	Inserting \eqref{eqn608} and \eqref{eqn609} into \eqref{eqn605}, we obtain
	\begin{equation}\label{eqn611}
		-\langle \nabla \QQ(\ga_{\SP}+n_{\SP}),m_{\SP}\rangle 
		\ge \varrho \norm{m_{\SC}}_{\ell^1} \cdot (1-\delta_1) 
			- \tilde{\varrho} \norm{m_{\SC}}_{\ell^1} \delta_2
		= \norm{m_{\SC}}_{\ell^1} \big[ \varrho - \varrho \delta_1 - \tilde{\varrho}\delta_2\big].
	\end{equation}
	
	Now that the estimates for the two terms in equation \eqref{eqn613} are established, we proceed to derive a lower bound for $\norm{n_{\SC}}_{\ell^{\infty}}$. Inserting \eqref{eqn612} and \eqref{eqn611} into \eqref{eqn613}, we obtain
	\begin{equation*}
		\norm{m_{\SC}}_{\ell^1}\hh^{-1}\Big( \frac{\norm{n_{\SC}}_{\ell^{\infty}}}{\alpha}\Big) 
		\ge \norm{m_{\SC}}_{\ell^1} \big[ \varrho - \varrho \delta_1 - \tilde{\varrho}\delta_2\big].
	\end{equation*}
	Dividing by $\norm{m_{\SC}}_{\ell^1}$, we deduce that
	\begin{equation}\label{eqn614}
		\hh^{-1}\Big( \frac{\norm{n_{\SC}}_{\ell^{\infty}}}{\alpha}\Big) 
		\ge \varrho - \varrho \delta_1 - \tilde{\varrho}\delta_2.
	\end{equation}
	Assume for now that
	\begin{equation} \label{eqn1008}
		0\le \varrho - \varrho \delta_1 - \tilde{\varrho}\delta_2.
	\end{equation}
	Then both sides of \eqref{eqn614} are in $[0,1)$.
	Applying $\hh$ to both sides of $\eqref{eqn614}$ at $(a)$, and
	using the convexity of $\hh$ on $[0,1)$ at $(b)$, we obtain
	\begin{equation} \label{eqn1009}
			\norm{n_{\SC}}_{\ell^{\infty}} 
			\overgeq{(a)} \alpha 
			\cdot \hh \big( \varrho - \varrho \delta_1 - \tilde{\varrho}\delta_2 \big)
			\overgeq{(b)} \alpha\cdot \big[ \hh(\varrho) 
			- \delta_3\big],
	\end{equation}
	where
	\begin{equation*}
		\delta_3 := \hh'(\varrho) \cdot \big( \varrho \delta_1 + \tilde{\varrho}\delta_2\big).
	\end{equation*}
	To finish the proof, it remains to check that our assumption \eqref{eqn1008} holds and to give an upper bound for $\delta_3$.
	
	We first establish upper bounds for $\delta_1$ and $\delta_2$.
	It follows from \cref{lemma:WindHansen} that
	\begin{equation} \label{eqn2006}
		\delta_1 = 1- \hh^{-1}(\eps^{-1}) \le \eps^{\gamma}.
	\end{equation}
	Before estimating $\delta_2$, we derive some preliminary inequalities. From \eqref{eqn_def_h_D} we infer that $\hh(\varrho) \le (1-\varrho)^{-\frac{1}{\gamma}}$. Using this and assumption \eqref{eqn1006} at $(b)$, the upper bound \eqref{eqn1002} at $(a)$, and assumption \eqref{eqn1006} at $(c)$, we deduce that	
	\begin{align}
		\norm{n_{\SP}}_{\ell^{\infty}} 
		&\le \norm{n}_{\ell^1}
		\overleq{(a)} \alpha \abs{\SC} (1+\tilde{\varrho}) \cdot \Big( \hh(\varrho) 
		+ \frac{4\varrho^{-} \cdot \varepsilon^\gamma }{ \gamma (1-\varrho)^{\frac{1}{\gamma}+1}}  \Big)
		\nonumber\\
		&\overleq{(b)} \alpha \abs{\SC} (1+\tilde{\varrho}) \cdot \Big( \frac{1}{(1-\varrho)^{\frac{1}{\gamma}}} 
		+ \frac{1}{(1-\varrho)^{\frac{1}{\gamma}}} \Big) 
		\nonumber\\
		&\overleq{(c)} \frac{1}{2} \mg. 
		\label{eqn1007}
	\end{align}
	Hence we have
	\begin{equation*}
		\abs{\ga_i}+\na_i \ge \min_{i\in \SP}\abs{\ga_i} - \norm{n_{\SP}}_{\ell^{\infty}} \ge \frac{1}{2} \min_{i\in \SP}\abs{\ga_i} >0
	\end{equation*}
	for all $i\in \SP$.
	Using \eqref{eqn622b} of \cref{lemma:basic_inequalities_deep_unique} at $(a)$, the definition of $\eps$ at $(b)$, and \eqref{eqn1007} at $(c)$, we obtain
	\begin{align} 
			\delta_2 
			&= \max_{i\in \SP}\abs{ \hh^{-1}\Big( \frac{\abs{\ga_i}+\na_i}{\alpha}\Big)
				- \hh^{-1}\left( \varepsilon^{-1} \right)} \nonumber \\
			&\overleq{(a)}
			\max_{i\in \SP}
			\frac{
				\gamma \abs{ \frac{\abs{\ga_i}+\na_i}{\alpha} - \varepsilon^{-1}  }
			}
			{
			\left( \min \left\{ \frac{\abs{\ga_i}+\na_i}{\alpha}  ; \varepsilon^{-1} \right\}
			\right)^{1+\gamma}
			}
			\nonumber
			\\
			&\overleq{(b)}
			\frac{
				\gamma
				\max_{i \in \SP} \abs{ \frac{\abs{\ga_i}+\na_i - \min_{j \in \SP} \abs{\ga_j} }{\alpha}   }
			}
			{
			\left( \min_{i \in \SP} 
			   \left\{ \frac{\abs{\ga_i}+\na_i}{\alpha}  ; 
			   \frac{ \min_{j \in \SP} \abs{\ga_j}}{\alpha} \right\}
			\right)^{1+\gamma}
			}
			\nonumber
			\\
			&=
			\frac{
				\gamma
				\alpha^{\gamma}
				\max_{i \in \SP} \abs{ \abs{\ga_i}+\na_i - \min_{j \in \SP} \abs{\ga_j}  }
			}
			{
			\left( \min_{i \in \SP} 
			   \left\{ \abs{\ga_i}+\na_i  ; 
			   \min_{j \in \SP} \abs{\ga_j} \right\}
			\right)^{1+\gamma}
			}
			\nonumber
			\\
			& \overleq{(c)}
			\frac{ 2 \gamma \alpha^{\gamma} \max_{i \in \SP} \abs{\ga_i} }
			{ \left( \min_{i \in \SP} \abs{\ga_i}/2 \right)^{1+\gamma} }
			\nonumber\\
			&=
			2^{2+\gamma} \gamma \varepsilon^{\gamma} \kappa_{\star}.
			\label{eqn2007}
	\end{align}
	Using \eqref{eqn2006} and \eqref{eqn2007} at $(a)$, and Assumption \eqref{eqn1006} at $(b)$, we infer that
	\begin{align}
		\varrho \delta_1 + \tilde{\varrho}\delta_2
		\overleq{(a)} \big( \varrho  + 2^{2+\gamma} \tilde{\varrho}  \gamma  \kappa_{\star} \big) 
		\eps^{\gamma}
		\label{eqn2008}
		\overleq{(b)} \varrho.
	\end{align}
    This verifies that the assumption \eqref{eqn1008} is indeed satisfied.
	To conclude the proof, we derive an upper bound for $\delta_3$. 
	Using  \cref{lemma:basic_inequalities_deep_unique},
	see Equation \eqref{eqn615},
	and \eqref{eqn2008},
	\begin{equation} \label{eqn2009}
		\delta_3 
		= \hh'(\varrho) 
		\cdot \big( \varrho \delta_1 + \tilde{\varrho}\delta_2\big)
		\le  
		\frac{2}{ \gamma (1-\varrho)^{\frac{1}{\gamma}+1}} 
		\cdot 
		\big( \varrho  +  2^{2+\gamma} \tilde{\varrho} \gamma \kappa_{\star} \big) 
		\cdot \eps^{\gamma}.
	\end{equation}
	By inserting \eqref{eqn2009} into \eqref{eqn1009}, we obtain 
	\begin{equation*}
		\norm{n_{\SC}}_{\ell^{\infty}}
		\ge
		\alpha 
		\cdot \Big[ 
			\hh(\varrho) 
			-  
				\frac{2 \left(\varrho+2^{2+\gamma}\tilde{\varrho} \gamma \kappa_{\star} \right) }
				{\gamma(1-\varrho)^{\frac{1}{\gamma}+1}} 
				\cdot \eps^{\gamma}
		\Big].
	\end{equation*}
	Rearranging terms and recalling the definition of $\eps$, we deduce the lower bound \eqref{lower_bound_deepd_unique}. 
\end{proof}

%% file: parts/sharpness_Dtwo.tex
\newcommand{\talpha}{t_{\alpha}}

In this section,
we establish \cref{proposition:SharpnessDtwo} and \cref{proposition:SharpnessDlargertwo}.
Thus, our goal is to construct concrete matrices $A$ and $y$ 
which show that our upper and lower bounds are sharp.
The main idea
which we pursue
is to consider a matrix $A \in \R^{(d-1) \times d}$,
which has a one-dimensional null space $\ker (A)$.
This will allow us to derive explicit formulas
for the minimizer of the Bregman divergence, $\xinf$. 

For this purpose, we consider the following construction. 
Let  $A \in \R^{(d-1) \times d}$ be a matrix 
with $\ker A = \text{span} \{n\}$,
where
\begin{equation*}
	n
	:=
	\left(
		\gamma_1,-\gamma_2, \frac{1}{d-2},\ldots,\frac{1}{d-2}
	\right).
\end{equation*}
The constants $\gamma_1 \ge 0$ and $\gamma_2 \ge 0$ will be specified later.
Next, we define $y:=A\ga$, where
\begin{equation*}
		\ga := (\ga_1,\ga_2,0,\ldots,0)\in \R^d
\end{equation*}
for some positive numbers $\ga_1, \ga_2 >0$.
By construction, the support of $\ga$ is given by $\SP= \{1,2\}$.
Thus, 
due to \cref{nullspaceconstants}
we observe that 
the null space constant $\varrho$ 
associated with $A$ and $\ga$
is given by
\begin{equation*}
	\varrho 
	=
	\frac{\vert - \sign{(\ga_1)} n_1 - \sign{(\ga_2)} n_2 \vert }{\norm{n_{\SC}}_{\ell^1} }
	=
	\vert \gamma_2-\gamma_1 \vert.
\end{equation*}
In the following, we assume that $\vert \gamma_2 - \gamma_1 \vert < 1$.
This implies that $\varrho<1$ 
and thus due to \cref{lemma:normal_coneSimple}
the vector $\ga \in \R^d$ is the unique solution of the optimization problem
\begin{equation*}
	\underset{x: Ax=y}{\min} \norm{x}_{\ell^1}.
\end{equation*}

\subsection{Case $D=2$ (Proof of \cref{proposition:SharpnessDtwo})}
Recall from \cref{proposition:SharpnessDtwo}
that $\xinf (\alpha)$ is for any $\alpha>0$ defined by
\begin{equation*}
	\xinf (\alpha) := \argmin_{x:Ax=y} D_{\Hb}(x,0).
\end{equation*}
This is well-defined since $D_{\Hb}(\cdot,0)$ is a strictly convex function
and thus there is a unique minimizer.

Now note that since $\ker (A) = \text{span} (n)$
and $y=A\ga$
it holds that $\xinf \left( \alpha \right) =\ga +\talpha n$ for some $\talpha \in \R$.
Since the kernel of $A$ is one-dimensional
we can compute that $\talpha$ satisfies the following equation.
\begin{lemma}
	Let $\xinf (\alpha)$ as defined above.
	Then it holds that
	\begin{equation}\label{eqn230}
	    \talpha	
		= 
		2\alpha (d-2)
		\sinh 
		\left(
		    -\arsinh\left( \frac{\ga_1+ \talpha \gamma_1}{2\alpha} \right) \gamma_1
			+\arsinh\left( \frac{\ga_2- \talpha \gamma_2}{2\alpha}\right) \gamma_2
		\right).
	\end{equation}
\end{lemma}
\begin{proof}
	As described above we have $\xinf (\alpha) = \ga +\talpha n$. 
	Since 
	$s\mapsto D_{\Hb}(\ga+sn,0)$ is differentiable on $\R$, 
	we infer that $\talpha$ must satisfy the first order optimality 
	condition
	\begin{align*}
			0 
			=& 
			\frac{\dif}{\dif s}\Big\rvert_{s = \talpha } D_{\Hb}(\ga+sn,0)\\
			=& \sum_{i=1}^{d} 
			\arsinh\left( \frac{\ga_i+ \talpha n_i}{2\alpha}\right) n_i 
			\\
			=& \arsinh
			\left( \frac{\ga_1+ \talpha \gamma_1}{2\alpha} \right) \gamma_1  
			- \arsinh\left( \frac{\ga_2- \talpha \gamma_2}{ 2\alpha }\right) \gamma_2 
			+ \sum_{i=3}^{d}
			\arsinh
			\left( \frac{\talpha}{2(d-2) \alpha }\right) \frac{1}{d-2}\\
			=& \arsinh
			\left( \frac{\ga_1+ \talpha \gamma_1}{2\alpha} \right) \gamma_1  
			- \arsinh\left( \frac{\ga_2- \talpha \gamma_2}{ 2\alpha }\right) \gamma_2 
			+ 
			\arsinh
			\left( \frac{ \talpha }{2(d-2) \alpha }\right).
	\end{align*}
    By rearranging terms we obtain \cref{eqn230}.
	This completes the proof.
\end{proof}

	With \cref{eqn230} in place,
	we can prove the following key lemma,
	which describes the asymptotic behavior of $\talpha$ as $\alpha$
	converges to $0$.
    \begin{lemma}\label[lemma]{lemma:talphaDtwo}
		Assume that $\gamma_2 \ge \gamma_1 \ge 0$ and that $ \varrho=\gamma_2 - \gamma_1 < 1 $.
		Recall that for any $\alpha>0$ we have that
		$ \xinf (\alpha) = \ga + \talpha n $.
		Then it holds that
		\begin{equation*}
			\lim_{\alpha \downarrow 0}
			\frac
			{\talpha}
			{ \abs{\SC} 
			\alpha^{1-\rho}
			(\ga_2)^{\gamma_2} (\ga_1)^{-\gamma_1}
			}
			=
			1.
		\end{equation*}
	\end{lemma}

	\begin{proof}
	Our starting point is \cref{eqn230}.
    To deal with the right-hand side of this equation,
	denote by $\Delta$ the function defined in \cref{lemma:basic_properties_arsinh}.
	Moreover,
	since $ \lim_{\alpha \rightarrow 0} \xinf (\alpha)=\ga $,
	by \cref{theorem:upper_bound_+-}
	we have that $\lim_{\alpha \rightarrow 0} \talpha=0$.
	In particular, 
	we have that $\ga_1+\talpha \gamma_1>0$ and $\ga_2-\talpha \gamma_2>0$
	for sufficiently small $\alpha>0$. 
	In particular,
	for sufficiently small $\alpha>0$
	we can use 
	the decompositions
	\newcommand{\ggh}{\xi}
	\begin{align*}
		\arsinh
		\left( \frac{\ga_1+ \talpha \gamma_1}{2\alpha} \right) \gamma_1
		&=
		\log
		\left( \frac{\ga_1+ \talpha \gamma_1}{\alpha} \right) \gamma_1 
		+ \Delta\left(\frac{\ga_1+ \talpha \gamma_1}{\alpha} \right) \gamma_1
		=\log\left( \frac{\ga_1}{\alpha} \right) \gamma_1
		+ \ggh_1 (\alpha),\\
		\arsinh
		\left( \frac{\ga_2- \talpha \gamma_i}{2\alpha} \right) \gamma_2
		&=
		\log
		\left( \frac{\ga_2- \talpha \gamma_2}{\alpha} \right) \gamma_2 
		+ \Delta\left(\frac{\ga_2- \talpha \gamma_2}{\alpha} \right) \gamma_2
		=\log\left( \frac{\ga_2}{\alpha} \right) \gamma_2
		+ \ggh_2 (\alpha),
	\end{align*}
	where we have set
	\begin{align*}
		\ggh_1 (\alpha)
		&:=
		\log
		\left( \frac{\ga_1+ \talpha \gamma_1}{\ga_1} \right) \gamma_1
		+ \Delta
		\left(\frac{\ga_i+ \talpha \gamma_i}{\alpha} \right) \gamma_1,\\
		\ggh_2 (\alpha)
		&:=
		\log
		\left( \frac{\ga_2- \talpha \gamma_2}{\ga_2} \right) \gamma_2
		+ \Delta
		\left(\frac{\ga_2- \talpha \gamma_2}{\alpha} \right) \gamma_2.
	\end{align*}	
	Inserting this into \eqref{eqn230} we obtain that
	for sufficiently small $\alpha>0$
	\begin{align}
		\talpha
		&=
		2\alpha (d-2)
		\sinh
		\left(
			-\log \left( \frac{(\ga_1)^{\gamma_1}}{\alpha^{\gamma_1}} \right)
			-\ggh_1 (\alpha)
			+\log \left( \frac{(\ga_2)^{\gamma_2}}{\alpha^{\gamma_2}}\right)
			+\ggh_2 (\alpha)
		\right) \nonumber \\
		&=
		2\alpha (d-2)
		\sinh
		\left(
			\log \left( 
			\frac{ (\ga_2)^{\gamma_2} \alpha^{\gamma_1-\gamma_2}}{(\ga_1)^{\gamma_1}} \right)
			-\ggh_1 (\alpha)
			+\ggh_2 (\alpha)
		\right).
		\label{eqn231}
	\end{align}
	Since we  have 
	$$
	\sinh(s) 
	= \frac{1}{2}\left( \exp(s) - \exp(-s)\right)
	 = \frac{1}{2} \exp(s) \left( 1 -\exp(-2s)\right)$$
	we obtain that
	\begin{align*}
		\talpha
		&=
		\alpha \left(d-2 \right)
		\frac{ (\ga_2)^{\gamma_2} \alpha^{\gamma_1-\gamma_2}}{(\ga_1)^{\gamma_1}}
		\underset{=:B \left( \alpha \right)}{
		\underbrace{
		\exp \left( \ggh_2 (\alpha)-\ggh_1 (\alpha) \right)
		\left(
			1
			-
		 \frac{ (\ga_1)^{2\gamma_1} \alpha^{2(\gamma_2-\gamma_1)}}{(\ga_2)^{2\gamma_2}}
			\exp \left(
				2 \left(\ggh_1 (\alpha)-\ggh_2 (\alpha) \right)
			\right)
		\right)}}.
	\end{align*}
	By rearranging terms we obtain that
	\begin{align*}
		\frac{t_{\alpha}}
		{\left(d-2 \right)
		\alpha^{1-(\gamma_2-\gamma_1)}
		\frac{ (\ga_2)^{\gamma_2} }{(\ga_1)^{\gamma_1}}
		}
		=
	    B \left( \alpha \right).
	\end{align*}
	Recall that $\xinf (\alpha)= \ga +\talpha n$.
	Thus,
	since $ \lim_{\alpha \rightarrow 0} \xinf (\alpha)=\ga $,
	by \cref{theorem:upper_bound_+-}
	we have that $\lim_{\alpha \rightarrow 0} \talpha=0$.
	It follows from the definition of the functions $\ggh_1$ and $\ggh_2$
	and from \cref{lemma:basic_properties_arsinh}, see \cref{eqn186},
	that
	$\lim_{\alpha \downarrow 0} \ggh_1 (\alpha)=0$
	and
	$\lim_{\alpha \downarrow 0} \ggh_2 (\alpha)=0$.
	This in turn implies that 
	$\lim_{\alpha \downarrow 0} B (\alpha) =1$.
	We obtain that 
	\begin{equation*}
		\lim_{\alpha \downarrow 0}
		\
		\frac{t_{\alpha}}
		{\left(d-2 \right)
		\alpha^{1-(\gamma_2-\gamma_1)}
		\frac{ (\ga_2)^{\gamma_2} }{(\ga_1)^{\gamma_1}}
		}
		=
		1.
	\end{equation*}
	The claim follows now from
	$\abs{\SC} =d-2$
	and
	$ \gamma_2 - \gamma_1 = \varrho$.
\end{proof}
With this auxiliary lemma in place,
we can now prove \cref{proposition:SharpnessDtwo}.
\begin{proof}[Proof of \cref{proposition:SharpnessDtwo}]
We set $\gamma_1:=\varrho^{-}-\varrho$ and $\gamma_2=\varrho^{-}$.
Note that from \cref{nullspaceconstants} and the definition of $n$ 
it then follows that
\begin{align}
	\varrho
    =
	\gamma_2 - \gamma_1, \quad
	\varrho^{-}
	=
	\gamma_2, \quad
	\tilde{\varrho}
	=
	\gamma_1 + \gamma_2.
	\label{equ:rhocalculations}
\end{align}
Note that this choice of $\gamma_1$ and $\gamma_2$ was possible
since we assumed that $\varrho \le \varrho^{-} $ 
and $2\varrho^{-} - \varrho= \tilde{\varrho} $.
We will prove the two statements in \cref*{proposition:SharpnessDtwo} separately.\\

\noindent \textbf{Part a):}
We set $\ga_1:=1$ and $\ga_2=\kappa_{\ast} \ge 1$.
From \cref{lemma:talphaDtwo} we obtain that 
	\begin{equation}\label{equ:internDominik2}
		1
		=
		\lim_{\alpha \downarrow 0}
		\
		\frac{t_{\alpha}}
		{\abs{\SC}
		\alpha^{1-\varrho}
		\kappa_{\ast}^{\gamma_2} 
		}
		\overeq{\eqref{equ:rhocalculations}}
		\lim_{\alpha \downarrow 0}
		\
		\frac{t_{\alpha}}
		{\abs{\SC}
		\alpha^{1-\varrho}
		\left( \min_{i \in \mathcal{S}} \ga_i \right)^{\varrho}
		\kappa_{\ast}^{\varrho^{-}} 
		}.
	\end{equation}
	Next, since $\xinf (\alpha) =\ga + \talpha n$
	and since $ \talpha >0$
	for all sufficiently small $\alpha>0$
	it holds for all sufficiently small $\alpha>0$ that
	\begin{equation*}
		\talpha 
		=
		\frac{ \norm{\xinf(\talpha) - \ga}_{\ell^1} }{\norm{n}_{\ell^1}}
		=
		\frac{ \norm{\xinf (\talpha)- \ga}_{\ell^1} }{1+\gamma_1+\gamma_2}
		\overeq{\eqref{equ:internDominik2}}
		\frac{ \norm{\xinf (\talpha)- \ga}_{\ell^1} }{1+\tilde{\varrho}}.
	\end{equation*}
	By inserting the last equation into \cref{equ:internDominik2} we obtain 
	\cref{equ:internDominik6}.\\

    \noindent \textbf{Part b):}	
	Let $\ga_1=\kappa_{\star}\ge 1$ and $\ga_2=1$.
    From \cref{lemma:talphaDtwo} we obtain that 
	\begin{equation}\label{equ:internDominik4}
		1
		=
		\lim_{\alpha \downarrow 0}
		\
		\frac{t_{\alpha}}
		{\abs{\SC}
		\alpha^{1-\varrho}
		\kappa_{\ast}^{-\gamma_1} 
		}
		\overeq{(a)}
		\lim_{\alpha \downarrow 0}
		\
		\frac{t_{\alpha}}
		{\abs{\SC}
		\alpha^{1-\varrho}
		\kappa_{\ast}^{\varrho - \varrho^-} 
		}
		\overeq{(b)}
		\lim_{\alpha \downarrow 0}
		\
		\frac{t_{\alpha}}
		{\abs{\SC}
		\alpha^{1-\varrho}
		\norm{\ga}^{\varrho}_{\ell^{\infty}}
		\kappa_{\ast}^{ -\varrho^-} 
		}.
	\end{equation}
    For equality $(a)$ we used that $\varrho =\gamma_2 - \gamma_1 $ 
	and $ \varrho^{-}=\gamma_2 $.
	For equality $(b)$ we used that $ \kappa_{\ast}= \norm{\ga}_{\ell^{\infty}} $. 
	Since $ \xinf (\alpha) = \ga + \talpha n $ 
	and since $n_i=1/(d-2)$ for all $i \in \SC$
	we obtain that
	for all sufficiently small $\alpha >0$
	that
	\begin{equation*}
		\talpha
		=
		\frac
		{\norm{\xinf_{\SC} (\alpha) - \ga_{\SC}}_{\ell^{\infty}} }
		{\norm{n_{\SC}}_{\ell^{\infty}}}
		=
		(d-2)
		{\norm{\xinf_{\SC} (\alpha) - \ga_{\SC}}_{\ell^{\infty}} }
		.
	\end{equation*}
	Inserting the last equation into \cref{equ:internDominik4}
	we obtain that
	\begin{equation*}
		1
		=
		\lim_{\alpha \downarrow 0}
		\
		\frac{\norm{\xinf_{\SC} (\alpha) - \ga_{\SC}}_{\ell^{\infty}}}
		{
		\alpha^{1-\varrho}
		\norm{\ga}_{\ell^{\infty}}^{\varrho}
		\kappa_{\ast}^{-\varrho^{-}} 
		}.
	\end{equation*}
	This proves \cref{equ:internDominik5}
	and the proof of \cref{proposition:SharpnessDtwo} is complete.
\end{proof}

%% file: parts/sharpness_Dgreatertwo.tex
\subsection{Case $D \ge 3$}

    For any $\alpha >0$, recall from \cref{proposition:SharpnessDtwo}
    that $\xinf (\alpha)$ is defined by
    \begin{equation*}
	   \xinf (\alpha) := \argmin_{x:Ax=y} D_{\QQ}(x,0).
    \end{equation*}
	As in the case $D=2$,
	we note that since $\ker (A) = \text{span} (n)$
	and $y=A\ga$
	it holds that $\xinf (\alpha) =\ga +\talpha n$ for some $\talpha \in \R$.
	Next, we compute that $\talpha$ satisfies the following equation.
    \begin{lemma}\label[lemma]{lemma:sharpnessDequal3}
		It holds that
	    \begin{equation}\label{equ:internDominik11}
		    \frac{\talpha}{\alpha (d-2) } 
		    =	\hh \left( \varrho
			+ \left[\hh^{-1}\left( \frac{\ga_2-\talpha\gamma_2}{\alpha}\right) -1 \right]\gamma_2
			+\left[ 1- \hh^{-1}\left( \frac{\ga_1+\talpha\gamma_1}{\alpha}\right) \right]\gamma_1\right).
	    \end{equation}
	\end{lemma}

	\begin{proof}
	Since $s\mapsto D_{Q^D_{\alpha}}(\xa+sn,0)$ is differentiable, 
	$\talpha$ satisfies the first order optimality condition
	\begin{align*}
			0 
			&= 
			\left.\frac{\dif}{\dif s}\right\rvert_{s =\talpha} D_{Q^D_{\alpha}}(\ga+sn,0)\\
			&= \sum_{i=1}^{d} \hh^{-1}\left( \frac{\ga_i+ \talpha n_i}{\alpha}\right)n_i\\
			&= \hh^{-1}\left( \frac{\ga_1+\talpha\gamma_1}{\alpha}\right) \gamma_1	
				-\hh^{-1}\left( \frac{ \ga_2 -\talpha\gamma_2}{\alpha}\right) \gamma_2
				+ \sum_{i=3}^d \hh^{-1}\left( \frac{\talpha}{\alpha(d-2)}\right) \frac{1}{d-2}.
	\end{align*}
	We obtain that
	\begin{align*}
			\hh^{-1}\left( \frac{\talpha}{\alpha(d-2)}\right)
			&= 	\hh^{-1}\left( \frac{\ga_1-\talpha\gamma_2}{\alpha}\right) \gamma_2 
				-\hh^{-1}\left( \frac{ \ga_2 +\talpha\gamma_1}{\alpha}\right) \gamma_1\\
			&= \gamma_2-\gamma_1 
			  + \left[\hh^{-1}\left( \frac{ \ga_1 -\talpha\gamma_2}{\alpha}\right) -1 \right]\gamma_2
			  +\left[ 1- \hh^{-1}\left( \frac{\ga_2 +\talpha\gamma_1}{\alpha}\right) \right]\gamma_1.
	\end{align*}
	By applying $\hh$ to both sides we obtain \cref{equ:internDominik11}.
	This completes the proof.
	\end{proof}
    With \cref{equ:internDominik11}, in place
	we can prove \cref{proposition:SharpnessDlargertwo}.

    \begin{proof}[Proof of \cref{proposition:SharpnessDlargertwo}]
	Since $ \xinf (\alpha)= \ga + \talpha n $
	due to \cref{theorem:upper_bound_deep} 
	we have that $\lim_{\alpha \rightarrow 0} \talpha=0$.
	It follows from \cref{lemma:WindHansen} that
	\begin{equation*}
		\lim_{\alpha\downarrow 0}
		\hh^{-1}
		\left(
			\frac{ \ga_2 -\talpha \gamma_2}{\alpha}
		\right)
		=
		1
	\end{equation*}
	and
	\begin{equation*}
		\lim_{\alpha\downarrow 0}
		\hh^{-1}
		\left(
			\frac{ \ga_1 + \talpha \gamma_1}{\alpha}
		\right)
		=
		1.
	\end{equation*}
	Thus, it follows from \cref{lemma:sharpnessDequal3} that
	\begin{equation}\label{equ:internDominik12}
		\lim_{ \alpha \to 0} \ \frac{\talpha}{\alpha (d-2) } 
		= \hh \left(\varrho \right).
	\end{equation}
    Since, in addition $\xinf (\alpha)=\ga + \talpha n$,
	we obtain that
	for all sufficiently small $\alpha >0$ that
	\begin{align*}
		\talpha
		=
		\frac{ \norm{ \ga - \xinf ( \alpha ) }_{\ell^1} }
		{ \norm{ n }_{\ell_1} }
		=
		\frac{ \norm{ \ga - \xinf ( \alpha ) }_{\ell^1} }
		{ 1 + \tilde{\varrho} }
		\quad
		\text{ and }
		\quad
		\talpha
		= 
		(d-2)
		\norm{ \left(\xinf (\alpha)\right)_{\SC} - \ga_{\SC} }_{\ell^{\infty}}.
	\end{align*}
    By inserting the last two equations in \cref{equ:internDominik12}
	we obtain \cref{equ:internDominik13} and \cref{equ:internDominik14}.
	This completes the proof of \cref{proposition:SharpnessDlargertwo}.
    \end{proof}

%% file: parts/discussions.tex
\section{Discussions}\label{section:outlook}

In this paper, we have established sharp upper and lower bounds on the $\ell^1$-approximation error 
of deep diagonal linear networks.
This result enabled us to precisely characterize the rate of convergence 
of the approximation error with the scale of initialization $\alpha$.
Moreover, we have conducted numerical experiments to validate our theoretical findings.
They indicate that deeper networks, i.e., $D \ge 3$,     
especially in noisy settings,
exhibit stronger implicit regularization towards sparsity
and better generalization performance,

Our results open up several interesting directions for future research.
We highlight a few of them here:
\begin{enumerate}
    \item \textit{Lower bounds for the $\ell^1$-approximation error.}
    The lower bounds in our main results
    for $ \norm{\ga -\xa (\alpha)}_{\ell^p} $,
    \cref{theorem:upper_bound_+-} and \cref{theorem:upper_bound_deep},
    are stated for $p=\infty$,
    whereas the upper bounds are stated for $p=1$.
    It would be of interest to explore whether 
    one can also derive lower bounds for the $\ell^1$-approximation error.
    In the case $D\ge 3$, can we potentially compute the limit
    $\lim_{\alpha \downarrow 0} \frac{ \norm{ \xinf (\alpha) - \ga }_{\ell^1} }{\alpha}$
    explicitly?

    \item \textit{Going beyond diagonal networks.}
    Our results indicate that the depth of the network plays a crucial role in the implicit regularization towards sparsity,
    especially in noisy settings.
    It would be interesting to see whether similar results can be obtained for more general architectures,
    for example deep matrix factorizations \cite{arora2019implicit}. 
    Can we maybe even observe similar phenomena in certain neural network architectures 
    with non-linear activation functions?
\end{enumerate}

%% file: parts/main_results_non-unique.tex
\section{Main results in the non-unique case}
\label{section:main_results_non-unique}

\subsection{Setup and assumptions}
In this section, we aim to extend our theory to the scenario that
\begin{equation*}
	\min_{x:Ax=y} \norm{x}_{\ell^1}
\end{equation*}
has no unique solution.
We consider the set of all minimizers 
which is defined as
\begin{equation*}
	\Lscr_{\min}:= \Lscr_{\min}(A,y) := \argmin_{x:Ax=y} \norm{x}_{\ell^1}.
\end{equation*}
We will make the following assumptions.
\begin{assumption} \label[assumption]{assumption_on_A_y_non-unique}
	Let $A\in \R^{N\times d}$ and $y\in \R^N$. 
	We assume that
	\begin{enumerate}
    \item[(a)] there exists $x\in \R^d$ such that $Ax=y$,
    \item[(b)] $y\ne 0$,
    \item[(c)] $\ker(A)\ne \{0\}$.
	\end{enumerate}
\end{assumption}
We note that these conditions are the same  as in \cref{assumption_on_A_y}
except that we do not require the minimizer to be unique.
One reason that this setting is more challenging is
because it is no longer clear to which $\ell^1$-minimizer 
the limit point of the gradient flow, $\xinf (\alpha)$, converges
as the scale of initialization $\alpha $ approaches $0$.
Another reason is that the definitions of the null space constants, 
which were central to the theory in the unique minimizer case, 
cannot be generalized effortlessly from \cref{nullspaceconstants} to the case 
where multiple minimizers exist.
As it turns out, the following two issues arise
when we try to generalize the definitions of the null space constants:
\begin{enumerate}
	\item  The definition of the null space constants
	in Section \ref{section:MainResults}
    involve the sign pattern 
	and the support of the unique minimizer, 
	which no longer exists. 
	\item 
	Recall that the definition of null space constants in Section \ref{section:MainResults}
	involves a division by the term $\norm{n_{\SC}}_{\ell^1}$, 
	where $n$ is a non-zero element in the null space of $A$
	and $\SP$ is the support of the unique minimizer $\ga$.
	Now, since $\ell_1$-minimizers are not unique,
	the set of minimizers $\Lmin$
	is obtained by intersecting the affine subspace of all solutions to $Ax=y$
	with the $\ell^1$-ball of minimal radius.
	We can now take two distinct minimizers $x,x'\in \Lmin$
	with the same support $\SP$ .
	Then, for $n:= x-x' \in \ker(A)$
	we obtain $n_{\SC}=0$
	and thus in the old definition we would divide by zero.
\end{enumerate}
To address the first issue,
we define the generalized support of the set of minimizers $\Lmin$ as
\begin{equation*}
\SP
:=\supp(\Lmin)
=
\bigcup_{x\in L} \supp(x).
\end{equation*}
As before, we also set $\SC:=\{1,\ldots,d\}\setminus \SP$.
As mentioned above $\Lmin$ is obtained by taking the intersection 
of an affine subspace
with the $\ell^1$-ball of minimal radius. 
Since this affine subspace can intersect the $\ell^1$-ball at most in one facet,
all elements in $\Lscr_{\min}$ have the same sign pattern.
This is made rigorous in the following lemma,
which in a slightly different version was already stated in \cite[Lemma 3.22]{wind2023implicit}.
For the convenience of the reader, 
we added a proof in \cref{sec:proof_of_basic_properties_Lmin}.
\begin{lemma}\label[lemma]{basic_properties_Lmin}
	Let $A$ and $y$ be as in \cref{assumption_on_A_y_non-unique}.
	Then $\Lmin$ is a non-empty convex and compact subset. 
	Furthermore, $0 \notin \Lmin$ and there exists $\s \in \{-1,1\}^d$ 
	such that $\s \hada x\in \R^d_{\ge 0}$ for all $x\in \Lmin$.
\end{lemma}

To deal with the second issue mentioned above, 
we define the following subspace,
which can be interpreted as a tangent space of $\Lmin$:
\begin{equation} \label{tangent_cone}
	\Tan := \spanV\big\{ x-x' : x,x' \in \Lmin \big\}
	\subset
	\ker(A).
\end{equation}
The next lemma characterizes the subspace $\Tan$. 
The straightforward proof 
has been deferred to \cref{section:prooftangent_cone_characterization}. 
\begin{lemma} \label[lemma]{lemma:tangent_cone_characterization}
	Let $A$ and $y$ as in \cref{assumption_on_A_y_non-unique}.
	Let $\s \in \{-1 ;1 \}^d $ according to \cref{basic_properties_Lmin},
	i.e., it holds that $\s \hada x\in \R^d_{\ge 0}$ for all $x\in \Lmin$.
	Then it holds that 
	\begin{equation*}
		\Tan = \Big\{ n \in \ker(A) : \sum_{i \in \SP}\s n_i =0, \text{ and } n_{\SC} =0 \Big\}.
	\end{equation*}
\end{lemma}
Next, let
$\Nor \subset \ker(A)$ be such that 
\begin{equation} \label{direct_sum}
	\Tan \cap \Nor = \{0\}
	\qquad \text{and}\qquad
	\Tan + \Nor = \ker(A).
\end{equation}
The precise form of $\Nor$ will be stated later, 
because we need slightly different definitions
for the cases $D=2$ and $D\ge3$.

With these definitions in place and with $\s$ as in \cref{basic_properties_Lmin}, we can define the (generalized) null space constants as
\begin{equation} \label{eqn102}
	\begin{aligned}
		\varrho
			&:=\varrho(\Nor)
			:=\sup_{0\ne n\in \Nor} 
				\frac{-\sum_{i\in \SP} \s n_i}{\norm{n_{\SC}}_{\ell^1}},\\
		\tilde{\varrho}
			&:= \tilde{\varrho}(\Nor)
			:=\sup_{0\ne n\in \Nor}
				\frac{\norm{n_{\SP}}_{\ell^1}}{\norm{n_{\SC}}_{\ell^1}},\\
		\varrho^{-}
			&:= \varrho^{-}(\Nor)
			:=\sup_{0\ne n\in \Nor}
				\Big(\sum_{i \in \SP:\s_i n_i <0} \abs{n_i}\Big)
				\cdot \frac{1}{\norm{n_{\SC}}_{\ell^1}}.
	\end{aligned}
\end{equation}
The following proposition shows that these constants are well-defined
if \cref{assumption_on_A_y_non-unique} holds.
\begin{proposition} \label[proposition]{lemma:normal_cone}
	Assume that $A \in \R^{N \times d}$ and $y \in \R^N$ 
	fulfill \cref{assumption_on_A_y_non-unique}.
	Assume in addition that $\Nor\subset \ker(A)$ satisfies \eqref{direct_sum}. 
	Then the following statements hold.
	\begin{enumerate}[1.]
		\item For every $n\in \Nor$ with $n\ne 0$ we have $n_{\SC}\ne 0$.
		In particular, $\varrho, \tilde{\varrho}$, and $\varrho^{-}$ are well-defined.
		\item If $\Nor \ne \{0\}$, then 
		\begin{equation*}
			0 \le \varrho<1,\qquad
			0 \le \tilde{\varrho}, \,\varrho^{-} <\infty,
		\end{equation*}
		and the suprema in \eqref{eqn102} are attained.
	\end{enumerate}
\end{proposition}
We conclude this section with the following remarks.
\begin{remark}\phantom{ }
 \begin{enumerate}
     \item 
Definition \eqref{eqn102} can be seen as a strict generalization of the null space constants 
introduced in \cref{nullspaceconstants},
where we have assumed that the $\ell^1$-minimizer is unique.
Namely, if the minimizer of $\min_{x:Ax=y} \norm{x}_{\ell^1}$ is unique,
then we have that $\Tan =\{0\}$ and $\Nor =\ker(A)$.
Thus, in particular, definition \eqref{eqn102} 
coincides with definition \eqref{nullspaceconstants}.

\item 
Furthermore, note that \cref{lemma:tangent_cone_characterization} implies that
\begin{equation*}
	\varrho (\Nor)
	=
	\sup_{0\ne n\in \ker (A)} 
	\frac{-\sum_{i\in \SP} \s n_i}{\norm{n_{\SC}}_{\ell^1}}.
\end{equation*}
Thus, the null space constant $\varrho = \varrho (\mathcal{N})$ 
is independent of the choice of $\Nor$.
However, the constant $\tilde{\varrho}$ may depend on the choice of $\Nor$.
 \end{enumerate}
\end{remark}

\subsection{Case $D= 2$}

Since the $\ell^1$-minimizer is not unique,
we first need to clarify which minimizer $\xinf (\alpha)$ is  converging to
when $\alpha \downarrow 0$.
For this purpose, define the function $E\colon \R^d_{\ge 0}\to \R^d_{\ge 0}$ by
\begin{equation}\label{equ:Edefinition}
	E(x) = \sum_{i:x_i\ne 0} x_i \log(x_i)-x_i,\quad x\in \R^d_{\ge 0}.
\end{equation}
Here, we have used the convention $0\log(0)=0$.

Then we define the point $\ga$ as
\begin{equation}\label{equ:D2minimizercharacterization}
	\ga \in \argmin_{x\in \Lmin} E\big( \abs{x}\big).
\end{equation}
The minimizer $\ga$ is unique and thus well-defined.
Moreover, this minimizer has maximal support in the sense that 
$\supp (\ga) = \SP$.
The following lemma, which is similar to \cite[Lemma 3.23]{wind2023implicit}, makes this precise.
For the convenience of the reader, we have included a proof in \cref{section:proof_of_maximal_support_D_equal_2}.
\begin{lemma}[Maximal support]\label[lemma]{lemma:maximal_support}
	Let $A \in \R^{N \times d}$ and $y \in \R^N$  as in \cref{assumption_on_A_y_non-unique}. 
	Let $\ga$ be defined as in \eqref{equ:D2minimizercharacterization}.
    Then $\ga$ is well-defined and the unique minimizer of \eqref{equ:D2minimizercharacterization}.
	Moreover, we have $\supp(\ga) = \SP$.
\end{lemma}

It still remains to define the subspace $\Nor$.
For this purpose, we introduce the following bilinear form
\begin{equation} \label{scalar_prod_shallow}
	\langle \cdot, \cdot \rangle_{\ga} \colon \R^d \times \R^d \to \R, 
	\quad
	(n,m) \mapsto \sum_{i\in \SP} \frac{n_i m_i }{\abs{\ga_i}}.
\end{equation}
By \cref{lemma:maximal_support} this bilinear form is well-defined. 
It allows us to define $\Nor$ as
\begin{equation} \label{normal_cone_shallow}
	\Nor 
	:= \Big\{ n \in \ker(A) : \langle n,m\rangle_{\ga} =0 \text{ for all } m \in \Tan \Big\}.
\end{equation}
\begin{lemma} \label[lemma]{lemma:normal_cone_shallow}
	Assume that $A \in \R^{N \times d}$ 
	and $y \in \R^N$ satisfy \cref{assumption_on_A_y_non-unique} 
	and let $\Ncal$ be defined by \eqref{normal_cone_shallow}.
	Then \eqref{direct_sum} holds.
\end{lemma}
The proof of this lemma has been deferred to \cref{section:proof_of_lemma:normal_cone_shallow}.
With these definitions in place, we can state the main result for $D=2$
in the non-unique scenario.
\begin{theorem}[Upper bound]\label[theorem]{result:upper_bound_shallow_non_unique}
	Let $A \in \R^{N \times d}$ and $y \in \R^N$ as in \cref{assumption_on_A_y_non-unique}.
	Let the null space constants $\varrho$, $\tilde{\varrho}$, and $\varrho^{-}$ 
	be defined as in \eqref{eqn102} with $\Nor$ as in \eqref{normal_cone_shallow}.
	Let 
	\begin{equation*}
		\xinf \in \argmin_{x: Ax=y} D_{\Ha}(x,0).
	\end{equation*}
	Assume that the scale of initialization $\alpha>0$ satisfies
	the conditions
	\begin{equation} \label{assumption_shallow_non-unique}
	\begin{split}
		\Big(\frac{\alpha}{\mg}\Big)^2
		&\le 
		\frac{\min_{i\in \SP}\abs{\ga_i}}{ 20 \norm{\ga}_{\ell^1}},
		\quad
		\text{and}\\
		\Big(\frac{\alpha}{\mg}\Big)^{1-\varrho} 
		&\le 
		\frac{1}{4 \cdot 2^{\varrho^{-}} \cdot \tilde{\varrho}\abs{\SC}  \kappa(\ga)^{\varrho^{-}}},\\
		\left(\frac{\alpha}{\mg} \right)^{1+\varrho}
		&\le
		\frac{ \tilde{\varrho} \cdot \kappa (\ga)^{\varrho^{-}}  \abs{\SC} \min_{i \in \SP}  \abs{\ga_i} }
		{4  \norm{\ga}_{\ell^1} } ,
	\end{split}
	\end{equation}
	where $\kappa(\ga):= \frac{\max_{i\in \SP}\abs{\ga_i}}{\mg}$.
	Then it holds that
	\begin{equation*}
		\frac{\norm{\xinf-\ga}_{\ell^1}}
		{\alpha^{1-\varrho}}
		\le
		\left(
		1+\tilde{\varrho}
		+
		C_1 
		\left( \frac{\alpha}{ \min_{i \in \SP} \abs{\ga_i} } \right)^{1-\varrho}
		\right)
		\abs{\SC}
		\big(\min_{i\in \SP}\abs{\ga_i} \big)^{\varrho} 
		\kappa(\ga)^{\varrho^{-}}
		g \left( \alpha \right)
		+
		\frac{
		2\alpha^{2} \norm{\ga}_{\ell^1}}
		{ \min_{i \in \SP} \abs{\ga_i}^{2} },
	\end{equation*}
	where
	\begin{align*}
		C_1
		&:=
		\frac{
		32 
		\tilde{\varrho}^2
		\abs{\SC}
		\kappa(\ga)^{\varrho^{-}}
		\norm{\ga}_{\ell^1}
		}
		{\min_{i\in \SP}  \abs{\ga_i}},\\
		g(\alpha)
		&:=
		\left(
			1 +\frac{ 10 \alpha^2 \norm{\ga}_{\ell^1} }{ \min_{i \in \SP} \abs{\ga_i}^3 }
		\right)^{\varrho^-}.
	\end{align*} 
\end{theorem}
Thus, analogously as in the case of a unique minimizer,
the approximation error 
decreases at most with rate $\alpha^{1-\varrho}$.
We note that \cite{wind2023implicit} has already proven
that $\xinf$ converges to the minimizer $\ga$
defined in \cref{equ:D2minimizercharacterization}.
Moreover, this paper shows that the approximation error decreases with rate $\alpha^{c}$
where $c$ is an undetermined constant.
In contrast, our result determines the constant $c$ explicitly with $c=1-\varrho$.

We observe that
as $\alpha \downarrow 0$,
the right-hand side
is asymptotically 
the same as in the unique case, see \cref{theorem:upper_bound_+-}.
Namely, in both cases, the right-hand side converges to
\begin{equation*}
   (1+\tilde{\varrho})
 \abs{\SC}
 \left( \mg \right)^{\varrho}
 \kappa(\ga)^{\varrho^{-}} 
 \qquad
 \text{as }\alpha \downarrow 0.
\end{equation*}
For this reason, we would expect that this upper bound is also asymptotically tight as well.
Moreover, similarly to the unique case,
it would be interesting to determine a lower bound
which shows that the approximation error decreases exactly with rate $\alpha^{1-\varrho}$.
We leave these questions as open problems for future research.

\subsection{Case $D\ge 3$}

As in the case $D=2$, we first need to clarify to which $\ell^1$-minimizer $\ga \in \Lmin$
the Bregman minimizers $\xinf (\alpha)$ are converging to as $\alpha \downarrow 0$.
As it turns out, 
$\ga$ is given as the unique solution of the following concave maximization problem
\begin{equation} \label{def:minimizer_deep}
	\ga \in \argmax_{x \in \Lmin} \norm{x}_{\ell^{\frac{2}{D}}}.
\end{equation}
We note that $\ga$ is well-defined and that $\ga$ has full support on $\SP$.
This has already been observed in \cite[Lemma 3.23]{wind2023implicit} in a slightly different setting.
Here, we state in the following lemma a version that is adapted to our notation. 
For the convenience of the reader, we include a proof in \cref{section:proof_of_lemma:maximal_support_deep}.
\begin{lemma}[Maximal support] \label[lemma]{lemma:maximal_support_deep}
	Let $D\ge 3$. 
	Assume that $A \in \R^{N \times d}$ 
	and $y \in \R^d$ fulfill \cref{assumption_on_A_y_non-unique}. 
    Then $\ga$ is well-defined and the unique minimizer of \eqref{def:minimizer_deep}.
	Furthermore, it holds that $\supp(\ga)=\SP$.
\end{lemma}

Again, to define the null space constants
we need to define the subspace $\Nor$. 
For this purpose, we recall that $\gamma:= \frac{D-2}{D}$ and introduce the bilinear form
\begin{equation} \label{scalar_prod_deep}
	\langle \cdot, \cdot \rangle_{\ga} \colon \R^d \times \R^d \to \R, 
	\quad
	(n,m) \mapsto \sum_{i\in \SP} \frac{n_i m_i }{\abs{\ga_i}^{1+\gamma}},
\end{equation}
which is well-defined by \cref{lemma:maximal_support_deep}.
Then we define $\Nor$ as
\begin{equation} \label{normal_cone_deep}
	\Nor 
	:= \Big\{ n \in \ker(A) : \langle n,m\rangle_{\ga} =0 \text{ for all } m \in \Tan \Big\}.
\end{equation}
The following lemma shows that $\Nor$ has
the desired property \eqref{direct_sum}.
The proof is deferred to \cref{section:proof_of_lemma:normal_cone_deep}.

\begin{lemma}\label[lemma]{lemma:normal_cone_deep}
	Let $d,A,y$ as in \cref{assumption_on_A_y_non-unique} and $\Ncal$ given by \eqref{normal_cone_deep}.
	Then \eqref{direct_sum} holds.
\end{lemma}

With these definitions in place, we can state the main result for $D\ge 3$.
\begin{theorem}[Upper bound] \label[theorem]{theorem:upper_bound_deep_non_unique}
	Assume that $A \in \R^{N \times d}$ and $y \in \R^n$ satisfy
	\cref{assumption_on_A_y_non-unique}.
	Let $\varrho,\varrho^{-},\tilde{\varrho}$ be defined as in \eqref{eqn102} 
	with $\Nor$ given by \eqref{normal_cone_deep}.
	Let $D\ge 3$ and $\alpha>0$.
	Set $\gamma:=\frac{D-2}{D}$.
	Let
	\begin{equation*}
		\xinf \in \argmin_{x: Ax=y} D_{\QQ}(x,0).
	\end{equation*}
	Assume that
	\begin{equation} \label{eqn1500}
		\frac{\alpha}{\min_{i\in \SP}{\abs{\ga_i}}} 
		\le \min\Big\{ 
		\frac{1}{8}\Big(\frac{\min_{i\in \SP}\abs{\ga_i}}{\norm{\ga}_{\ell^1}}\Big)^{1+\gamma},\quad
		\frac{1}{2}\Big(\frac{(1-\varrho)^{1/\gamma +1 }\gamma}{4\varrho^{-}}  \Big)^{\frac{1}{\gamma}},\quad
		\frac{1}{8 \tilde{\varrho} \abs{\SC} ( \hh (\varrho) +1 ) }
		\Big\}.
	\end{equation}
	Then it holds that 
	\begin{align*}
		\frac{\norm{\xinf-\ga}_{\ell^1}}{\alpha}
		\le 
		(1+\tilde{\varrho}) \abs{\SC} \hh(\varrho)
		+\left(\frac{\norm{\ga}_{\ell^{1}}}{\min_{i\in \SP}\abs{\ga_i}}\right)^{1+\gamma} 
		+
		g \left( \frac{\alpha}{ \min_{i \in \SP} \abs{\ga_i} } \right) ,
	\end{align*}
	where the function $g$ is defined as 
	\begin{align*}
		g(\varepsilon)
		:=&
		C^{\sharp} \varepsilon \abs{\SC}
		\left( 
			\hh(\varrho) 
			+\frac{4\cdot 2^\gamma \varrho^{-}\varepsilon^\gamma}
			{\gamma (1-\varrho)^{\frac{1}{\gamma}+1}} 
		\right)
		+
	 10 \varepsilon
	\left(\frac{\norm{\ga}_{\ell^{1}}}{\min_{i\in \SP}\abs{\ga_i}}\right)^{1+\gamma}
	\end{align*}
	with 
	\begin{equation*}
		C^{\sharp}:= 
		5 \tilde{\varrho}
		\left(
		88 
        \left( \frac{\norm{\ga}_{\ell^{1}}}{\mg}\right)^{1+\gamma}
		+
		512 \tilde{\varrho}
		\abs{\SC}
		\left(
		\hh(\varrho)+1
		\right)
		\right)
\cdot 
		\left(
		\frac{2d \norm{\ga}_{\ell^\infty}}{ \min_{i\in \SP} \abs{\ga_i} }
		\right)^{1+\gamma}.
	\end{equation*}
\end{theorem}
Thus, this theorem shows that the approximation error decreases at least with rate $\alpha$.
This matches the rate of convergence in the scenario that there is a unique minimizer, see \cref{theorem:upper_bound_deep}.
Moreover, as $\alpha \downarrow 0$, the right-hand side converges to a sum of two terms.
The first term $ \left(1+ \tilde{\varrho} \right) \abs{\SC} \hh (\varrho) $ 
also appears in the unique case when one takes the limit, see \cref{theorem:upper_bound_deep}.
The second term $\left( \norm{\ga}_{\ell^{1+\gamma}} / \mg\right)^{1+\gamma}$ 
is a new term. 
It remains an open problem to determine whether this term is necessary or whether it can be removed.

Finally, let us note that \cite{wind2023implicit} has already proven a 
result of the form $\norm{\xinf -\ga}_{\ell^1} \le C_A \alpha$.
The above theorem improves upon this result by specifying the constants in the leading terms of the upper bound. 
(The unspecified constant is in the higher order term 
which vanishes asymptotically.)

%% file: parts/proofs_Dtwo_non-unique.tex
\subsection{Case $D=2$: Proof of \cref{result:upper_bound_shallow_non_unique}}

	Set $n:= \xinf -\ga \in \ker(A)$. 
	By \eqref{direct_sum} and \cref{lemma:normal_cone_shallow}, 
	there exist uniquely defined $\npara \in \Tan$ and $\nperp \in \Nor$ 
	such that 
	\begin{equation*}\label{direct_sum_last}
	n = \nperp + \npara.
    \end{equation*}
	Thus, in addition to controlling $\norm{\nperp}_{\ell^1}$ 
	as in the unique minimizer case,
	we will also need to control
	$ \npara $.
	For our proof, it will be useful to define the auxiliary point
	\begin{equation}\label{def:xa}
		\xa \in  \argmin_{x\in \Lmin} D_{\Ha}(x,0).
	\end{equation}
	Due to \eqref{direct_sum} and \cref{lemma:normal_cone_shallow},
	we can decompose $\xinf-\xa \in \ker (A)$ into
	\begin{equation*}
		\xinf-\xa 
		= \widetilde{ \npara} + \widetilde{\nperp}
	\end{equation*}
	with $\widetilde{\npara} \in \Tan$ and $ \widetilde{\nperp} \in \Nor$.
	Note that because of $ \ga \in \Lmin$ and $ \xa \in \Lmin $, it holds that $ \ga - \xa \in \Tan $.
    This implies in particular that $ \widetilde{\nperp} = \nperp $.
    It follows that
	\begin{align}\label{direct_sum_last2}
		\xinf - \ga
		=
		(\xinf - \xa)
		 + (\xa - \ga)
		=
	    \widetilde{\npara} + \nperp	
		 + (\xa - \ga).
	\end{align}
	Thus, using the triangle inequality it follows from \cref{direct_sum_last2} that
	\begin{align}\label{ineq:Dominiklast1}
		\norm{\xinf - \ga}_{\ell^1}
		\le
		\norm{\xa - \ga}_{\ell^1}
		+
		\norm{ \widetilde{\npara} }_{\ell^1}
		+
		\norm{\nperp}_{\ell^1}.
	\end{align}
We will control the three summands individually.\\

\noindent \textbf{Step 1 (Controlling $\norm{\xa -\ga}_{\ell^1}$):}
	To control this therm,
	we will use the strong convexity of $\Ha$,
	which was established in \cite[Lemma 4]{ghai2019exponentiated}.
We state in \cref{lemma:strong_convexity} a slightly different version that is adapted to our notation.
For the sake of completeness, we have included a proof in \cref{section:proof_strong_convexity}.
\begin{lemma}\label[lemma]{lemma:strong_convexity} 
	Let $x,n\in \R^d$ with $x \ne 0$ and $\alpha>0$.
	Then it holds that
	\begin{equation*}
		\langle \nabla^2 \Ha(x)n,n\rangle 
		\ge 
		\frac{\norm{n}_{\ell^1}^2}
		{\norm{x}_{\ell^1} +2\alpha \vert \supp (n) \vert }.
	\end{equation*}
\end{lemma}
	
With this lemma at hand, 
we can prove an upper bound for $\norm{\xa - \ga}_{\ell^1}$.
\begin{proposition} \label[proposition]{lemma:two_minimizers_shallow}
	Let $A$ and $y $ as in \cref{assumption_on_A_y_non-unique} and let $\alpha>0$.
	Moreover, assume that the assumptions 
	of \cref{result:upper_bound_shallow_non_unique}
	are satisfied. 
	Then it holds that
	\begin{equation}\label{ineq:two_minimizers_shallow}
		\norm{\xa - \ga}_{\ell^1}
		\le (1+2 \eps) \eps^2
		\norm{\ga}_{\ell^1},
	\end{equation}
	where we have set
    \begin{equation*}
	    \eps:= \frac{\alpha}{\mg}.
    \end{equation*}
In particular, it holds for all $i \in \SP$ that 
$\abs{\xa_i -\ga_i} \le \min_{i \in S} \abs{\ga_i}/4$
and thus $\sign (\xa_i) = \sign (\ga_i) $. 
\end{proposition}
\newcommand{\nhat}{\hat{n}}
\begin{proof}[Proof of \cref{lemma:two_minimizers_shallow}]
	Let $\nhat:= \xa-\ga$.
	By the definition of $\Tcal$, 
	we have $\nhat\in \Tcal$. 
	Hence, \cref{lemma:tangent_cone_characterization} implies that $\nhat_{\SC}=0$. 
	By \cref{lemma:maximal_support}, 
	we have $\supp(\ga)=\SP$. 
	Hence, for all $i\in \SP$ and all $t\in\R$ 
	such that $\abs{t}$ is sufficiently small, 
	we have $0<\abs{\ga_i+t \nhat_i} = \abs{\ga_i}+t \nhat^{\ast}_{i}$,
	where we have set $\nhat_{\SP}^{\ast}:= \sign(\ga_{\SP})\hada \nhat_{\SP}$. 
	Therefore, it holds that
	\begin{align*}
		E(\abs{\ga+t \nhat})
		&= \sum_{i \in \SP } 
		\abs{\ga_i+t \nhat_i} \log \big( \abs{\ga_i+t \nhat_i}\big) - \abs{\ga_i+t \nhat_i}\\
		&= \sum_{i\in \SP} 
		\big( \abs{\ga_i}+t \nhat^{\ast}_{\SP}\big) \log\big( \abs{\ga_i}+t \nhat^{\ast}_{\SP}\big) 
				- \big( \abs{\ga_i}+t \nhat^{\ast}_{\SP}\big)\\
		&= E\big(\abs{\ga_{\SP}} + t \nhat^{\ast}_{\SP} \big)
	\end{align*}
	Hence the map $t\mapsto E(\abs{\ga+t \nhat})$ is differentiable at $t=0$.
	Furthermore, since $\Lmin$ is convex, 
	we have $\ga+t \nhat \in \Lmin$ for all $t\in [0,1]$.
	From the optimality of $\ga$, we deduce that
	\begin{equation} \label{eqn507}
			0 \le \left.\frac{\dif}{\dif t}\right\rvert_{t=0} E(\abs{\ga+t \nhat})
			= \langle \nabla E(\abs{\ga_{\SP}}), \nhat_{\SP}^{\ast}\rangle.
	\end{equation}

	Since $\Lmin$ is convex, 
	we have $\xa-t \nhat \in \Lmin$ for all $t\in [0,1]$. 
	Moreover, using the optimality of $\xa$ at $(a)$, 
	and $\supp(\xa)\subset\SP$ together with $\nhat_{\SC}=0$ at $(b)$, 
	we obtain that
	\begin{equation} \label{eqn508}
			0 \overleq{(a)} \left.\frac{\dif}{\dif t}\right\rvert_{t=0} D_{\Ha}(\xa+t(-\nhat),0)
			= \langle \nabla \Ha(\xa), - \nhat\rangle
			\overeq{(b)} \langle \nabla \Ha(\xa_{\SP}), - \nhat_{\SP}\rangle.
	\end{equation}
	Using \cref{lemma:strong_convexity} at $(a)$ 
	and the fact that $\ga+t \nhat \in \Lmin$ for all $t\in[0,1]$ at $(b)$, we infer that
	\begin{align}
			\langle \nabla \Ha(\xa_{\SP})-\nabla\Ha(\ga_{\SP}), n_{\SP}\rangle 
			&= \langle \int_0^1 \left. \frac{\dif}{\dif s}\right\rvert_{s =t} 
			\nabla\Ha(\ga_{\SP} + s \nhat_{\SP})\dif t , \nhat_{\SP}\rangle 
			\nonumber\\
			&=\int_0^1 \langle \nabla^2 \Ha(\ga_{\SP}+tn_{\SP}) \nhat_{\SP}, \nhat_{\SP}\rangle \dif t \nonumber \\
			&\overgeq{(a)} \frac{\norm{ \nhat_{\SP}}_{\ell^1}^2}{\sup_{t\in [0,1]} \norm{\ga_{\SP}+t\nhat_{\SP}}_{\ell^1} 
			\left(1+ \frac{ 2\alpha \vert \supp (\nhat) \vert }{ \sup_{t \in [0,1]} \norm{\ga_{\SP} + t\nhat_{\SP}}_{\ell^1}} \right)}
			\nonumber\\
			&\overeq{(b)}\frac{\norm{ \nhat }_{\ell^1}^2}{\norm{\ga}_{\ell^1} 
			\left(1+2\alpha \vert \supp ( \nhat) \vert / \norm{\ga}_{\ell^1} \right)}
			\nonumber \\
			&\overgeq{(c)}\frac{\norm{\nhat}_{\ell^1}^2}{\norm{\ga}_{\ell^1} 
			\left(1+2\alpha \abs{\SP} / \norm{\ga}_{\ell^1} \right)}.
			\label{eqn508a}
	\end{align}
	In inequality $(c)$ above,
	we have used that $\supp( \nhat ) \subset \SP$.
	Next, using inequality \eqref{eqn508a} at $(a)$, inequality \eqref{eqn508} at $(b)$, and inequality \eqref{eqn507} at $(c)$, we deduce that
	\begin{align}
		\frac{\norm{ \hat{n} }_{\ell^1}^2}{\norm{\ga}_{\ell^1} 
		\left( 1+2\alpha \abs{\SP} / \norm{\ga}_{\ell^1} \right)} 
		&\overleq{(a)} \langle \nabla \Ha(\xa_{\SP})-\nabla\Ha(\ga_{\SP}), \nhat_{\SP}\rangle
			\nonumber\\
		&\overleq{(b)} \langle -\nabla\Ha(\ga_{\SP}), \nhat_{\SP}\rangle
			\nonumber\\
		&\overleq{(c)} \langle \nabla E(\ga_{\SP}), \nhat^{\ast}_{\SP}\rangle - \langle \nabla\Ha(\ga_{\SP}), \nhat_{\SP}\rangle.
			\label{eqn508b}
	\end{align}
	Using that $\arsinh$ is an odd function at $(a)$, and the equality $\sum_{i\in \SP} n^{\ast}_i =0$, 
	see \cref{lemma:tangent_cone_characterization}, at $(b)$, we obtain
	\begin{align}
		\langle \nabla E(\ga_{\SP}), \nhat^{\ast}_{\SP}\rangle -
		\langle \nabla\Ha(\ga_{\SP}), \nhat_{\SP}\rangle
		&= \sum_{i\in \SP} \log(\abs{\ga_i}) \nhat^{\ast}_i - \sum_{i\in \SP} \arsinh\Big(\frac{\ga_i}{2\alpha}\Big) \nhat_i
			\nonumber\\
		&\overeq{(a)} \sum_{i\in \SP} 
			\Big[ 
			\log(\abs{\ga_i}) 
			- \arsinh\Big(\frac{\abs{\ga_i}}{2\alpha}\Big) 
			\Big] \nhat^{\ast}_i 
			\nonumber\\
		&\overeq{(b)} \sum_{i\in \SP} 
			\Big[ 
			\log\Big(\frac{\abs{\ga_i}}{\alpha}\Big) 
			- \arsinh\Big(\frac{\abs{\ga_i}}{2\alpha}\Big) 
			\Big] \nhat^{\ast}_i.
			\label{eqn508c}
	\end{align}
	Denote by $\Delta$ the function defined 
	in \cref{lemma:basic_properties_arsinh}.
	Then by \eqref{eqn183} and \eqref{eqn186}, 
	and since $\Delta$ is non-increasing,
	we obtain that
	\begin{equation}\label{eqn509}
		\sum_{i\in \SP} 
		\Big[ 
		\log\Big(\frac{\abs{\ga_i}}{\alpha}\Big) 
		- \arsinh\Big(\frac{\abs{\ga_i}}{2\alpha}\Big) 
		\Big] \nhat^{\ast}_i
		=  - \sum_{i\in \SP} n_i^{\ast} \Delta\Big( \frac{\abs{\ga_i}}{\alpha}\Big)
		\le \norm{\nhat}_{\ell^1} \Delta\Big( \frac{\mg}{\alpha}\Big)	
		\le \norm{\nhat}_{\ell^1}  \eps^2,
	\end{equation}
	where $\eps$ has been defined in the statement of this lemma.
	Combining 
	\eqref{eqn508b}, \eqref{eqn508c}, and \eqref{eqn509}, we deduce that
	\begin{equation*}
		\frac{\norm{\nhat}_{\ell^1}^2}{\norm{\ga}_{\ell^1} 
		\left(1+  2\alpha \abs{\SP} / \norm{ \ga }_{\ell^1}  \right)} 
		\le \norm{\nhat}_{\ell^1}  \eps^2
	\end{equation*}
	By rearranging terms, it follows that
	\begin{equation}
		\norm{\nhat}_{\ell^1} \le \norm{\ga}_{\ell^1} 
		\left(1+ 2\alpha \abs{\SP} / \norm{ \ga }_{\ell^1}  \right)  \eps^2.
	\end{equation}
	Since $\norm{\ga}_{\ell^1}\ge \abs{\SP} \mg $, 
	we have $2\alpha \abs{\SP}/ \norm{\ga}_{\ell^1} \le 2 \eps$.
	This proves inequality \eqref{ineq:two_minimizers_shallow}.
    To complete the proof we observe that from Assumption \eqref{assumption_shallow_non-unique}
	it follows that
	\begin{equation}\label{equ:preliminaryequation}
		\abs{\xa_i - \ga_i}
		\le
		\norm{\xa - \ga}_{\ell^1}
		\le
		(1+2 \eps) \eps^2
		\norm{\ga}_{\ell^1}
		\le 
		\frac{\min_{i \in \SP} \vert \ga_i \vert}{4}.
	\end{equation}
	In particular, we have that
	$\sign (\ga_i) = \sign(\xa_i)$
	for all $i \in \SP$.
	This completes the proof.
\end{proof}

\noindent
    \textbf{Step 2 (Controlling $\norm{\nperp}_{\ell^1}$):}
	We will follow a similar proof strategy as in the unique minimizer case 
	to establish that $\norm{\nperp}_{\ell^1} \lesssim \alpha^{1-\varrho}$.
	This is achieved by the next lemma.

    \begin{lemma}\label[lemma]{lemma:upper_boundnperp_shallow_non_unique}
		Assume that the assumptions 
		of \cref{result:upper_bound_shallow_non_unique}
		are satisfied. 
		Then it holds that 
		\begin{equation}\label{eqn912}
		\norm{n_{\SC}}_{\ell^1}
		\le 
		\alpha^{1-\varrho}
		\abs{\SC}
		\left( \min_{i\in \SP}\abs{\ga_i} \right)^{\varrho}
		\kappa(\ga)^{\varrho^{-}}
		h \left( \varepsilon \right),
		\end{equation}
		where $\eps := \frac{\alpha}{\min_{i\in \SP}\abs{\ga_i}}$ and
		$h(\varepsilon)
		:=	
		\left(
			1 
			+\frac{10 \varepsilon^2 \norm{\ga}_{\ell^1} }{ \min_{i \in \SP} \abs{\ga_i} }
		\right)^{\varrho^-}$.
		In particular,
		we have that
		$$ \norm{\nperp_{\SP}}_{\ell^\infty} 
		\le \norm{\nperp_{\SP}}_{\ell^1}
		\le \tilde{\varrho} \norm{\nperp_{\SC}}_{\ell^1}
		\le \min_{i \in \SP} \abs{\ga_i}/4.  $$
	\end{lemma}
	
    \begin{proof}	
	We have that
	\begin{equation*}
		\langle \nabla \Ha (\xinf),n\rangle =0
	\end{equation*}
    for all $n\in \ker(A)$.
    Since $ \xinf = \xa + \nperp + \widetilde{\npara}$
	and $ \xa_{\SC} =0$,
	we obtain that 
	\begin{align*}
		-\langle \nabla \Ha (\xa_{\SP} + \nperp_{\SP} + \widetilde{\npara_{\SP}} )
		, \nperp_{\SP} + \widetilde{\npara_{\SP}}  \rangle
		=
		\langle \nabla \Ha (n_{\SC}), n_{\SC} \rangle.
	\end{align*}
	For the left-hand side we obtain that
	\begin{align*}
		\langle \nabla \Ha (\xa_{\SP} + \nperp_{\SP} + \widetilde{\npara_{\SP}} )
		, \nperp_{\SP} + \widetilde{\npara_{\SP}}  \rangle
		\overgeq{(a)}
		\langle \nabla \Ha (\xa_{\SP}),\nperp_{\SP} + \widetilde{\npara_{\SP}}  \rangle
		\overeq{(b)}
		\langle \nabla \Ha (\xa_{\SP}),\nperp_{\SP}   \rangle,
	\end{align*}
	where we have used the monotonicity of $\nabla \Ha$ in inequality $(a)$.
	For equation (b) we have used
	the first-order optimality condition 
	$ \langle \nabla \Ha (\xa_{\SP}), \widetilde{\npara_{\SP}} \rangle =0 $,
	which follows from
	$\widetilde{\npara_{\SP}}  \in \Tan$,
	the definition of $\xa$, see \eqref{def:xa}, 
	and that $\supp (\xa_{\SP}) = \SP$
	due to \cref{lemma:maximal_support} and \cref{lemma:two_minimizers_shallow}.
    It follows that 
	\begin{equation*}
		-\langle \nabla \Ha (\xa_{\SP}),\nperp_{\SP}   \rangle
		\ge
		\langle \nabla \Ha (n_{\SC}), n_{\SC} \rangle.
	\end{equation*}
	Then we can proceed analogously as in the proof of \cref{theorem:upper_bound_+-}
	and obtain the following inequality, 
	which is analogous to \cref{ineq:internDominik17} in this proof,
    \begin{align*}
        \norm{n_{\SC}}_{\ell^1}
        \le
        2 \alpha \vert \SC \vert
        \sinh 
        \left(
		\frac{-1}{  \norm{n_{\SC}}_{\ell^1}}
        \sum_{i \in \SP} \nperp_i \sign(\xa_i) \arsinh\left(\frac{\abs{\xa_i}}{2\alpha}\right)
        \right).
    \end{align*}
	Define $ (\nperp)^* := \sigma \odot \nperp $,
	where $\sigma$ is as defined in \cref{basic_properties_Lmin}.
	Analogously, as in the proof of \cref{theorem:upper_bound_+-},
	the term inside the $\sinh$-function can be bounded by
    \begin{align*}
    \frac{-1}{\norm{n_{\SC}}_{\ell^1}}
    \sum_{i \in \SP} (\nperp)^{\ast}_i \arsinh\left(\frac{\abs{\xa_i}}{2\alpha}\right)
    \le
    \varrho  \left( \log \left(2\lambda\right) + \Delta \left(2\lambda \right) \right)
    +
    \varrho^{-} \log \left( \kappa (\xa) \right),
    \end{align*}
	where $ \lambda :=  \frac{\min_{i \in \mathcal{S}} \vert \xa_i \vert}{ 2\alpha } $
	and $\kappa (\xa) =\frac{ \max_{i \in \SP} \abs{\xa_i} }{ \min_{i \in \SP} \abs{\xa_i} }$.

	Then by arguing analogously as in the proof of \cref{theorem:upper_bound_+-},
	we obtain that
	\begin{align}
		\norm{n_{\SC}}_{\ell^1}
		\le 
		&\alpha^{1-\varrho}
		\abs{\SC}
		\left( \min_{i\in \SP}\abs{\xa_i} \right)^{\varrho}
		\kappa(\xa)^{\varrho^{-}}
		\left(
		1 +
		\frac{\alpha^2}{ \min_{i \in \SP} \abs{\xa_i}^2}
		\right)^{\varrho}\\
		\le 
		&\alpha^{1-\varrho}
		\abs{\SC}
		\left( \min_{i\in \SP}\abs{\xa_i} \right)^{\varrho}
		\kappa(\xa)^{\varrho^{-}}
		\left(
		1 +
		\frac{4\alpha^2}{ \min_{i \in \SP} \abs{\ga_i}^2}
		\right)^{\varrho},
		\label{ineq:Dominiklast2}
	\end{align}
	where in the last inequality we have used that
	$\min_{i \in \SP} \abs{\xa_i} \ge \frac{1}{2} \min_{i \in \SP} \abs{\ga_i}$,
	which follows from \cref{lemma:two_minimizers_shallow}.
	Next, we note that
	\begin{align*}
		\left( \min_{i \in \SP} \abs{\xa_i} \right)^{\varrho} 
		\kappa (\xa)^{\varrho^{-}}
		=
		&\frac{\max_{i \in \SP} \abs{\xa_i}^{\varrho^-}  }
		{ \min_{i \in \SP} \abs{\xa_i}^{\varrho^{-} - \varrho} }\\
		\le
		&\frac{ \max_{i \in \SP}\left(  \abs{\ga_i} +\abs{\xa_i - \ga_i }   \right)^{\varrho^{-}} }
		{ \min_{i \in \SP} \left( \abs{\ga_i} - \abs{\xa_i -\ga_i}  \right)^{\varrho^{-} - \varrho} }\\
		\overleq{(a)}
		&\frac{ \left(  \max_{i \in \SP}\abs{\ga_i} + (1+2\varepsilon) \varepsilon^2 \norm{\ga}_{\ell^1}\right)^{\varrho^{-}} }
		{  \left( \min_{i \in \SP}\abs{\ga_i} - (1+2\varepsilon) \varepsilon^2 \norm{\ga}_{\ell^1}  \right)^{\varrho^{-} - \varrho} }\\
		=
		&\frac{ \left( 1+(1+2\varepsilon) \varepsilon^2 \cdot \frac{\norm{\ga}_{\ell^1} }{ \max_{i \in \SP} \abs{\ga_i}  }  \right)^{\varrho^{-}} \max_{i \in \SP}\abs{\ga_i}^{\varrho^-} }
		{  \left( 1-(1+2\varepsilon) \varepsilon^2 \cdot \frac{\norm{\ga}_{\ell^1} }{\min_{i \in \SP} \abs{\ga_i}} \right)^{\varrho^{-} - \varrho} \min_{i \in \SP}\abs{\ga_i}^{\varrho^{-}-\varrho} }\\
		\overleq{(b)}
		&\frac{ \left( 1+2 \varepsilon^2 \cdot \frac{\norm{\ga}_{\ell^1} }{ \max_{i \in \SP} \abs{\ga_i}  } \right)^{\varrho^{-}} \max_{i \in \SP}\abs{\ga_i}^{\varrho^-} }
		{  \left( 1-2 \varepsilon^2 \cdot \frac{\norm{\ga}_{\ell^1} }{ \min_{i \in \SP} \abs{\ga_i}  }   \right)^{\varrho^{-} - \varrho} \min_{i \in \SP}\abs{\ga_i}^{\varrho^{-}-\varrho} }\\
		\overeq{(c)}
		&\left( \min_{i \in \SP} \abs{\ga_i} \right)^{\varrho} 
		\kappa (\ga)^{\varrho^{-}}
		\left( 1+\frac{ 2\varepsilon^2  \norm{\ga}_{\ell^1} }{ \max_{i \in \SP} \abs{\ga_i}  }  \right)^{\varrho^{-}}
		\left( 1+  \frac{4\varepsilon^2 \norm{\ga}_{\ell^1} }{ \min_{i \in \SP} \abs{\ga_i}  } \right)^{\varrho^{-} - \varrho}.
	\end{align*}
	In inequality $(a)$ we have used \cref{lemma:two_minimizers_shallow}.
	Inequality $(b)$ holds due to Assumption \eqref{assumption_shallow_non-unique},
	which implies that $\varepsilon \le 1/2$.
	In inequality $(c)$ we have used
	the elementary inequality $ 1/(1-x) \le 1 +2x $ if $ 0 \le x < 1/2$,
	and that 
	$ \frac{2\varepsilon^2 \norm{\ga}_{\ell^1} }{\min_{i \in \SP} \abs{\ga_i}}  \le 1/2$ due to Assumption \eqref{assumption_shallow_non-unique}.
	It follows that
	\begin{align}
		\norm{n_{\SC}}_{\ell^1}
		\le 
		&\alpha^{1-\varrho}
		\abs{\SC}
		\left( \min_{i\in \SP}\abs{\ga_i} \right)^{\varrho}
		\kappa(\ga)^{\varrho^{-}}
		\left(1 +4\varepsilon^2\right)^{\varrho}
		\left( 1+\frac{ 2\varepsilon^2  \norm{\ga}_{\ell^1} }{ \max_{i \in \SP} \abs{\ga_i}  }  \right)^{\varrho^{-}}
		\left( 1+  \frac{4\varepsilon^2 \norm{\ga}_{\ell^1} }{ \min_{i \in \SP} \abs{\ga_i}  } \right)^{\varrho^{-} - \varrho}.
		\label{eqn9123}
	\end{align}
	We observe that
	\begin{align*}
		&\left( 1+4\varepsilon^2 \right)^{\varrho}
		\left( 1+\frac{ 2\varepsilon^2  \norm{\ga}_{\ell^1} }{ \max_{i \in \SP} \abs{\ga_i}  }  \right)^{\varrho^{-}}
		\left( 1+  \frac{4\varepsilon^2 \norm{\ga}_{\ell^1} }{ \min_{i \in \SP} \abs{\ga_i}  } \right)^{\varrho^{-} - \varrho}\\
		=&
		\left(
			\frac{1 +4\varepsilon^2}
			{1+  \frac{4\varepsilon^2 \norm{\ga}_{\ell^1} }{ \min_{i \in \SP} \abs{\ga_i}  }}
		\right)^{\varrho}
		\left(
		\left( 1+ \frac{2 \varepsilon^2  \norm{\ga}_{\ell^1} }{ \max_{i \in \SP} \abs{\ga_i}  }   \right)
		\left( 1+ \frac{4 \varepsilon^2 \norm{\ga}_{\ell^1} }{ \min_{i \in \SP} \abs{\ga_i}  } 
		\right) \right)^{\varrho^{-} }\\
		\le&
		\left(
		\left( 1+ \frac{ 2\varepsilon^2  \norm{\ga}_{\ell^1} }{ \max_{i \in \SP} \abs{\ga_i}  }   \right)
		\left( 1+ \frac{4\varepsilon^2 \norm{\ga}_{\ell^1} }{ \min_{i \in \SP} \abs{\ga_i}  } 
		\right) \right)^{\varrho^{-} }\\
		=&
		\left(
			1 +  \frac{ 2\varepsilon^2  \norm{\ga}_{\ell^1} }{ \max_{i \in \SP} \abs{\ga_i}  }  
			+ \frac{4\varepsilon^2 \norm{\ga}_{\ell^1} }{ \min_{i \in \SP} \abs{\ga_i}  } 
			+ \frac{8 \varepsilon^4 \norm{\ga}^2_{\ell^1} }
			{ (\min_{i \in \SP} \abs{\ga_i}) (\max_{i \in \SP} \abs{\ga_i}) } 
		\right)^{\varrho^-}\\
		\le&
		\left(
			1 + \frac{6\varepsilon^2 \norm{\ga}_{\ell^1} }{ \min_{i \in \SP} \abs{\ga_i} }
			+ \frac{8 \varepsilon^4 \norm{\ga}^2_{\ell^1} }
			{ (\min_{i \in \SP} \abs{\ga_i}) (\max_{i \in \SP} \abs{\ga_i}) } 
		\right)^{\varrho^-}\\
		\le&
		\left(
			1 + \frac{10 \varepsilon^2 \norm{\ga}_{\ell^1} }{ \min_{i \in \SP} \abs{\ga_i} }
		\right)^{\varrho^-},
	\end{align*}
	where in the last inequality we have used the
	assumption that 
	$\varepsilon^2 \le \frac{\max_{i \in \SP} \abs{\ga_i}}{ 2 \norm{\ga}_{\ell^1} }$
	due to Assumption \eqref{assumption_shallow_non-unique}.
	By inserting this inequality into \cref{eqn9123}
	we obtain \cref{eqn912}.

	In order to complete the proof, note that
	\begin{align*}
		\norm{\nperp_{\SP}}_{\ell^\infty} 
		\le 
		&\norm{\nperp_{\SP}}_{\ell^1}\\
		\overleq{(a)} 
		&\tilde{\varrho} \norm{\nperp_{\SC}}_{\ell^1} \\
		\overleq{(b)}
		&
		\tilde{\varrho}
		\alpha^{1-\varrho}
		\abs{\SC}
		\left( \min_{i\in \SP}\abs{\ga_i} \right)^{\varrho}
		\kappa(\ga)^{\varrho^{-}}
		\left(
			1 
			+\frac{10 \varepsilon^2 \norm{\ga}_{\ell^1} }{ \min_{i \in \SP} \abs{\ga_i} }
		\right)^{\varrho^-}\\
		\overleq{(c)} &
		2^{ \varrho^-}
		\tilde{\varrho}
		\alpha^{1-\varrho}
		\abs{\SC}
		\left( \min_{i\in \SP}\abs{\ga_i} \right)^{\varrho}
		\kappa(\ga)^{\varrho^{-}}
		\\
		\overleq{(d)} & \frac{1}{4} \underset{i \in \SP}{\min} \abs{\ga_i} .
	\end{align*}
	Inequality (a) follows from the definition of $\tilde{\varrho}$,
	see \eqref{eqn102}.
	Inequality (b) follows from \cref{eqn912}.
	Inequalities (c) and (d) follow
	from the assumption on $\alpha$, see \cref{assumption_shallow_non-unique}.
	This completes the proof.
    \end{proof}

\newcommand{\tildenpara}{\widetilde{\npara}}
\noindent \textbf{Step 3 (Bounding $\norm{\tildenpara}_{\ell^1}$):}
It remains to control the third summand 
in \cref{ineq:Dominiklast1}.
This is achieved by the following lemma.
\begin{lemma}\label[lemma]{lemma:nparallelbound_shallow_non_unique}
Assume that the assumptions of \cref{result:upper_bound_shallow_non_unique} are satisfied.
Then it holds that
\begin{equation}\label{eqn1214}
	\norm{\tildenpara}_{\ell^1}
	\le
	\frac{
		32 \norm{\ga}_{\ell^1}
		\varepsilon^{1-\varrho}
		\tilde{\varrho}^2
		\abs{\SC}
		\kappa(\ga)^{\varrho^{-}}
		\norm{
			 \nperp_{\SC}}_{\ell^{1}}
		}
		{\min_{i\in \SP}  \abs{\ga_i}}
\end{equation}
where
\begin{align*}
	\xi (\varepsilon)
	:=
	\frac{6 \varepsilon^{1+\varrho}   \norm{\ga}_{\ell^1} }
	{ \tilde{\varrho} \cdot \kappa (\ga)^{\varrho^{-}}  \abs{\SC} \min_{i \in \SP}  \abs{\ga_i} }
	+ \left(
		1 + 10 \varepsilon^2 \abs{\SP} \kappa (\ga)
	\right)^{\varrho^-}
	\quad \text{and} \quad
	\varepsilon
	:=
	\frac{\alpha}{ \min_{i \in \SP} \abs{\ga_i}}.
\end{align*}
\end{lemma}	
\begin{proof}
	Since $\xinf$ and $\xa$ are minimizers of the functional $D_{\Ha}(\cdot,0)$ on the subsets $\Lscr$ and $\Lmin$, 
	respectively, we obtain 
	from the first order optimality conditions
	that
	\begin{align*}
		0 
		&= \left.\frac{\dif}{\dif t}\right\rvert_{t=0} D_{\Ha}(\xinf+t\tildenpara,0)
		= \langle \nabla \Ha(\xinf),\tildenpara\rangle,\\
		0 
		&= \left.\frac{\dif}{\dif t}\right\rvert_{t=0} 
		D_{\Ha}(\xa+t\tildenpara,0)
		= \langle \nabla \Ha(\xa),\tildenpara\rangle.
	\end{align*}
	By combining these two equations we obtain that
	\begin{equation} \label{eqn405aa}
		\langle \nabla \Ha(\xinf) , \tildenpara \rangle = \langle \nabla \Ha(\xa),\tildenpara\rangle.
	\end{equation}
	By recalling that $\xinf = \xa + \tildenpara +\nperp$ 
	we rewrite \eqref{eqn405aa} as
	\begin{equation} \label{eqn405a}
		\langle \nabla \Ha(\xa + \tildenpara +\nperp) - \nabla\Ha(\xa + \nperp ) , \tildenpara\rangle 
		= \langle \nabla \Ha(\xa) - \nabla \Ha(\xa+\nperp), \widetilde{\npara }\rangle.
	\end{equation}
	
	First, we derive a lower bound for the left-hand side of \eqref{eqn405a}. 
	Using that $\widetilde{\npara_{\SC}}=0$ at $(a)$
	and \cref{lemma:strong_convexity} at $(b)$, we infer that
	\begin{align}
			&\langle \nabla \Ha(\xa + \tildenpara +\nperp) - \nabla\Ha(\xa + \nperp ) , \tildenpara\rangle 
			\nonumber \\
			\overeq{(a)}& \langle \nabla \Ha(\xa_{\SP} + \widetilde{\npara_{\SP}} +\nperp_{\SP}) - \nabla\Ha(\xa_{\SP} + \nperp_{\SP} ), 
			\widetilde{\npara_{\SP}}\rangle
				\nonumber\\
			=& \langle \int_0^1 \left. \frac{\dif}{\dif s}\right\rvert_{s=t} \nabla \Ha(\xa_{\SP}  +\nperp_{\SP} 
			+ s\widetilde{\npara_{\SP}}) \dif t , \widetilde{\npara_{\SP}}\rangle \nonumber \\
			=& \int_0^1 \langle \nabla^2 \Ha\left(\xa_{\SP}  +\nperp_{\SP}+ t\widetilde{\npara_{\SP}} \right) \widetilde{\npara_{\SP}} , 
			\widetilde{\npara_{\SP}} \rangle \dif t
				\nonumber\\
			\overgeq{(b)} & \frac{ \norm{\widetilde{\npara_{\SP}}}_{\ell^1}^2 }
			{ \max_{t\in[0,1]} \norm{\xa_{\SP}  +\nperp_{\SP}+ t \widetilde{\npara_{\SP}}}_{\ell^1} + 2 \alpha \abs{\SP}}. 
				\label{eqn405b}
	\end{align}
	In equation $(a)$ we have used that $\widetilde{\npara_{\SC}}=0$.
	Inequality $(b)$ follows from \cref{lemma:strong_convexity}.

	Next, we derive an upper bound for the right-hand side of \eqref{eqn405a}. 
	Using that $\arsinh$ is an odd function, we obtain
	\begin{align}
			\langle \nabla \Ha(\xa) - \Ha(\xa+\nperp), \tildenpara\rangle
			&= \sum_{i\in \SP} \widetilde{ \npara_i} \Big[\arsinh\Big( \frac{\xa_i}{2\alpha}\Big) 
			- \arsinh\Big( \frac{\xa_i+\nperp_i}{2\alpha}\Big) \Big]
				\nonumber\\
			&= \sum_{i\in \SP} (\tildenpara)_i^{\ast} \Big[\arsinh\Big( \frac{\abs{\xa_i}}{2\alpha}\Big) 
			- \arsinh\Big( \frac{\abs{\xa_i}+(\nperp_i)^{\ast}}{2\alpha}\Big) \Big],
				\label{eqn405c}
	\end{align}
	where $(\nperp_{\SP})^{\ast} := \sign(\xa_{\SP})\hada \nperp_{\SP}$ and 
	$\left(\widetilde{\npara_{\SP}} \right)^{\ast} 
	:= \sign(\xa_{\SP})\hada \widetilde{\npara_{\SP}}$.
	Here,
	we have used that 
	$\sign (\ga_i)=\sign(\xa_i)=\sign( \xa_i + \nperp_i)$
	for all $i \in \SP$,
	which follows from
	$\abs{\xa_i - \ga_i} \le \underset{i \in \SP}{\min} \abs{\ga_i}/4$
	and 
	$ \abs{\nperp_i}  \le \underset{i \in \SP}{\min} \abs{\ga_i}/4$ 
	for all $i \in \SP$
	due to \cref{lemma:two_minimizers_shallow} and \cref{lemma:upper_boundnperp_shallow_non_unique}.
	Next, 
	note that the map $\phi\colon t\mapsto \arsinh\Big( \frac{\abs{\xa_i}+t (\nperp_i)^{\ast}}{2\alpha}\Big)$
	has derivatives
	\begin{equation*}
		\phi'(t) = \frac{(\nperp_i)^{\ast}}{\sqrt{ \big( \abs{\xa_i}+t(\nperp_i)^{\ast}\big)^2 + 4 \alpha^2}},
		\quad
		\phi''(t) 
		= - \frac{ \big( \abs{\xa_i}+t(\nperp_i)^{\ast}\big) (\nperp_i)^2}
		{ \big[ \big( \abs{\xa_i}+t(\nperp_i)^{\ast}\big)^2+4\alpha^2\big]^{\frac{3}{2}}}.
	\end{equation*}
	Hence, for each $i \in \SP$ there exists $t_i\in (0,1)$ such that
	\begin{equation*}
		\arsinh\Big( \frac{\abs{\xa_i}+(\nperp_i)^{\ast}}{2\alpha}\Big)
		= \arsinh\Big( \frac{\abs{\xa_i}}{2\alpha}\Big) 
		+ \frac{(\nperp_i)^{\ast}}{\sqrt{ \abs{\xa_i}^2 + 4 \alpha^2}}
		-\frac{ \big( \abs{\xa_i}+t_i (\nperp_i)^{\ast}\big) (\nperp_i)^2}
		{ 2 \big[ \big( \abs{\xa_i}+t_i (\nperp_i)^{\ast}\big)^2+4\alpha^2\big]^{\frac{3}{2}}}.
	\end{equation*}
	Hence,
	\begin{align}
		&\sum_{i\in \SP}  ( \widetilde{\npara_i})^{\ast}
\Big[\arsinh\Big( \frac{\abs{\xa_i}}{2\alpha}\Big) 
		- \arsinh\Big( \frac{\abs{\xa_i}+(\nperp_i)^{\ast}}{2\alpha}\Big) \Big]
			\nonumber\\
		= 
		&\sum_{i\in \SP}
		\left[ \frac{- (\nperp_i)^{\ast} \left(\widetilde{\npara_i} \right)^{\ast}}{\sqrt{ \abs{\xa_i}^2 
		+ 4 \alpha^2}}
		+ \frac{ \big( \abs{\xa_i}+t_i (\nperp_i)^{\ast}\big) (\nperp_i)^2  ( \widetilde{\npara_i})^{\ast}}
		{ 2 \big[ \big( \abs{\xa_i}+ t_i (\nperp_i)^{\ast}\big)^2+4\alpha^2\big]^{\frac{3}{2}}}
		\right].
			\label{eqn405d}
	\end{align}
	By definition of $\Nor$
	and since $ \widetilde{\npara} \in \Tan $ and $ \nperp \in \Nor  $
	we have $\langle \nperp,\tildenpara\rangle_{\ga}=0$.
	Hence, we obtain that
	\begin{align}
		-\sum_{i\in \SP} \frac{(\nperp_i)^{\ast} \left( \widetilde{\npara_i} \right)^{\ast}}{\sqrt{ \abs{\xa_i}^2 + 4 \alpha^2}}
		=& -\sum_{i\in \SP} \frac{(\nperp_i)^{\ast} \left(\widetilde{\npara_i} \right)^{\ast}}{\sqrt{ \abs{\xa_i}^2 + 4 \alpha^2}}
		+ \sum_{i\in \SP} \frac{(\nperp_i)^{\ast} \left(\widetilde{\npara_i} \right)^{\ast}}{\abs{\ga_i}} \nonumber \\
		\le &
		\max_{i\in \SP} 
		\abs{ \frac{1}{\sqrt{ \abs{\xa_i}^2 +4 \alpha^2}}- \frac{1}{\abs{\ga_i}}} 
		\norm{\nperp_{\SP}}_{\ell^{\infty}}
		\norm{\tildenpara}_{\ell^1} \nonumber \\
		\le &
		 \frac{
		\max_{i\in \SP} 
		\abs{ \sqrt{ \abs{\xa_i}^2 +4 \alpha^2} - \abs{\ga_i}}}
		{ \min_{i \in \SP} \left( \abs{\ga_i} 
		  \abs{\xa_i} \right) }
		\cdot
		\norm{\nperp_{\SP}}_{\ell^{\infty}}
		\norm{\tildenpara}_{\ell^1} 
		\nonumber \\
		\le &
		 \frac{2
		\max_{i\in \SP} 
		\abs{ \sqrt{ \abs{\xa_i}^2 +4 \alpha^2} - \abs{\ga_i}}}
		{ \min_{i \in \SP}  \abs{\ga_i}^2}
		\cdot
		\norm{\nperp_{\SP}}_{\ell^{\infty}}
		\norm{\tildenpara}_{\ell^1}.
			\label{eqn405e}
	\end{align}
	Moreover, we observe that
	from the monotonicity and concavity of the square root function
	it follows that
	\begin{align*}
		- \norm{\xa - \ga }_{\ell^1}
		\le
		\abs{\xa_i} - \abs{\ga_i}
		\le
		\sqrt{ \abs{\xa_i}^2 +4 \alpha^2} - \abs{\ga_i}
		\le
		\abs{\xa_i}+ \frac{2 \alpha^2}{\abs{\xa_i}} - \abs{\ga_i}
		\le 
		\norm{\xa- \ga}_{\ell^1} + \frac{2 \alpha^2}{\abs{\ga_i}}.
	\end{align*}
	It follows that 
	\begin{align*}
		\max_{i\in \SP} 
		\abs{ \sqrt{ \abs{\xa_i}^2 +4 \alpha^2} - \abs{\ga_i}}
		\le
		\left(1+2\eps \right) \eps^2 \norm{\ga}_{\ell^1}
         +
		 \frac{2\alpha^2}{ \abs{\ga_i}}
		 \le
		 4 \eps^2
		 \norm{\ga}_{\ell^1},
	\end{align*}
	where we have used
	\cref{lemma:two_minimizers_shallow}
	in the first inequality
	and
	Assumption \eqref{assumption_shallow_non-unique} in the second inequality.
	Inserting this estimate into \cref{eqn405e} 
	we obtain that 
	\begin{align}
		-\sum_{i\in \SP} \frac{(\nperp_i)^{\ast} \left(\widetilde{\npara_i} \right)^{\ast}}
		{\sqrt{ \abs{\xa_i}^2 + 4 \alpha^2}}
		\le 
		\frac{
			8 \eps^2 \norm{\ga}_{\ell^1}
			\norm{\nperp_{\SP}}_{\ell^{\infty}}
			\norm{\tildenpara}_{\ell^1}
			} 
			{ \min_{i \in \SP}  \abs{\ga_i}^2   
			}.
			\label{DominikIntern22}
	\end{align}
	Furthermore, we have for the second term in \eqref{eqn405d} that
	\begin{align}
		\sum_{i\in \SP}
		\frac{ \big( \abs{\xa_i}+t_i(\nperp_i)^{\ast}\big) (\nperp_i)^2 \left(\widetilde{\npara_i} \right)^{\ast}}
		{ 2 \big[ \big( \abs{\xa_i}+ t_i  (\nperp_i)^{\ast}\big)^2+4\alpha^2\big]^{\frac{3}{2}}}
		\le
		& \left( \max_{i \in \SP} 
		\frac{  \abs{\abs{\xa_i}+t_i(\nperp_i)^{\ast}}}
		{ 2 \big[ \big( \abs{\xa_i}+ t_i  (\nperp_i)^{\ast}\big)^2+4\alpha^2\big]^{\frac{3}{2}}}
		\right)
		\norm{\tildenpara}_{\ell^1}\norm{\nperp_{\SP}}_{\ell^{\infty}}^2
		\nonumber \\
		\le
		& \frac
		{\norm{\tildenpara}_{\ell^1}\norm{\nperp_{\SP}}_{\ell^{\infty}}^2}
		{2 \min_{i\in \SP}  \abs{ \abs{\xa_i}+ t_i  (\nperp_i)^{\ast} }^2 }
		\nonumber \\
		\overleq{(a)}
		& \frac
		{2 \norm{\tildenpara}_{\ell^1}\norm{\nperp_{\SP}}_{\ell^{\infty}}^2}
		{\min_{i\in \SP}  \abs{\ga_i}^2  },
			\label{eqn405f}
	\end{align}
	where $(a)$
	follows from
	$\abs{\xa_i - \ga_i} \le \underset{i \in \SP}{\min} \abs{\ga_i}/4$
	and 
	$ \abs{ \nperp_i}  \le \underset{i \in \SP}{\min} \abs{\ga_i}/4$ 
	for all $i \in \SP$
	due to \cref{lemma:two_minimizers_shallow} and \cref{lemma:upper_boundnperp_shallow_non_unique}.
	Combining \eqref{eqn405c}, \eqref{eqn405d}, \eqref{DominikIntern22} 
	and \eqref{eqn405f}, we obtain
	\begin{align}
		\langle \nabla \Ha(\xa) - \Ha(\xa+\nperp), \tildenpara\rangle
		&\le 
		2
		\left(
			8 \eps^2 \norm{\ga}_{\ell^1}
		+
			\norm{\nperp_{\SP}}_{\ell^{\infty}}
		\right)
		\frac{\norm{\tildenpara}_{\ell^1} \norm{\nperp_{\SP}}_{\ell^{\infty}} }
		{\min_{i\in \SP}  \abs{\ga_i}^2  }.
			\label{eqn405g}
	\end{align}
	Inserting the lower bound \eqref{eqn405b} 
	and the upper bound \eqref{eqn405g} into \eqref{eqn405a}, we deduce that
	\begin{align}
	\norm{ \widetilde{\npara_{\SP}}}_{\ell^1}
    \le
	& 2\left(  
		\max_{t\in[0,1]} \norm{\xa_{\SP}  +\nperp_{\SP}
		+ t \widetilde{\npara_{\SP}}}_{\ell^1} + 2 \alpha \abs{\SP}
	    \right)
        \left(
		8 \eps^2 \norm{\ga}_{\ell^1}
		+
		\norm{\nperp_{\SP}}_{\ell^{\infty}}
		\right)
		\frac{
		\norm{
			 \nperp_{\SP}}_{\ell^{\infty}
		}
		}
		{\min_{i\in \SP}  \abs{\ga_i}^2}.
		\label{eqn405h3}
	\end{align}
	In order to proceed we note that
	\begin{align*}
		8\varepsilon^2 \norm{\ga}_{\ell^1}
		+
		\norm{\nperp_{\SP}}_{\ell^{\infty}}
		&\le
		8\varepsilon^2 \norm{\ga}_{\ell^1}
		+ \tilde{\varrho}
		\norm{\nperp_{\SC}}_{\ell^1}\\
		&\overleq{(a)}
		8\varepsilon^2 \norm{\ga}_{\ell^1}
		+
		\tilde{\varrho}
		\varepsilon^{1-\varrho}
		\abs{\SC}
		\left(\min_{i\in \SP}\abs{\ga_i} \right)
		\kappa(\ga)^{\varrho^{-}}
		\left(
			1+ 
			\frac{10 \varepsilon^2 \norm{\ga}_{\ell^1}}
			{\min_{i\in \SP} \abs{\ga_i}}
		\right)\\
		&\overleq{(b)}
		8\varepsilon^2 \norm{\ga}_{\ell^1}
		+
		2 \tilde{\varrho}
		\varepsilon^{1-\varrho}
		\abs{\SC}
		\left( \min_{i\in \SP}\abs{\ga_i} \right)
		\kappa(\ga)^{\varrho^{-}}\\
		&=
		2\varepsilon^{1-\varrho}
		\left(
		4\varepsilon^{1+\varrho} \norm{\ga}_{\ell^1}
		+
		\tilde{\varrho}
		\abs{\SC}
		\left( \min_{i\in \SP}\abs{\ga_i} \right)
		\kappa(\ga)^{\varrho^{-}}
		\right)\\
		&\overleq{(c)}
		4
		\varepsilon^{1-\varrho}
		\tilde{\varrho}
		\abs{\SC}
		\kappa(\ga)^{\varrho^{-}}
		 \min_{i\in \SP}\abs{\ga_i},
	\end{align*}
	where $(a)$ follows from \cref{lemma:upper_boundnperp_shallow_non_unique} and
	$(b)$ and $(c)$ follow from the Assumption \eqref{assumption_shallow_non-unique}.
	By inserting this estimate into \eqref{eqn405h3} we obtain that
	\begin{align*}
	    \norm{ \widetilde{\npara_{\SP}}}_{\ell^1}
		\le
        8\left(
		\max_{t\in[0,1]} \norm{\xa_{\SP}  +\nperp_{\SP}
		+ t \widetilde{\npara_{\SP}}}_{\ell^1} + 2 \alpha \abs{\SP}
	    \right)
		\frac{
		\varepsilon^{1-\varrho}
		\tilde{\varrho}
		\abs{\SC}
		\kappa(\ga)^{\varrho^{-}}
		\norm{
			 \nperp_{\SP}}_{\ell^{\infty}}
		}
		{\min_{i\in \SP}  \abs{\ga_i}}.
	\end{align*}
In order to proceed further we note that
\begin{align*}
\max_{t\in[0,1]} \norm{\xa_{\SP}  +\nperp_{\SP}
+ t \widetilde{\npara_{\SP}}}_{\ell^1} + 2 \alpha \abs{\SP}
&\le 
\norm{\xa_{\SP}}_{\ell^1}
+\norm{\nperp_{\SP}}_{\ell^1}
+ \norm{ \widetilde{\npara_{\SP}}}_{\ell^1}
+ 2 \alpha \abs{\SP}\\
&\le
\norm{\ga}_{\ell^1}
+\norm{\xa- \ga}_{\ell^1}
+ \tilde{\varrho} \norm{\nperp_{\SC}}_{\ell^1}
+ \norm{ \widetilde{\npara_{\SP}}}_{\ell^1}
+2 \alpha \abs{\SP}\\
&\overleq{(a)}
\norm{\ga}_{\ell^1}
+ (1+2\varepsilon) \varepsilon^2 \norm{\ga}_{\ell^1}
+ \frac{\min_{i \in \SP} \abs{\ga_i}}{4}
+2 \alpha \abs{\SP}
+\norm{\widetilde{\npara_{\SP}}}_{\ell^1}\\
&\overleq{(b)} 
2 \norm{\ga}_{\ell^1}
+\norm{\widetilde{\npara_{\SP}}}_{\ell^1},
\end{align*}
where inequality $(a)$ follows from  
\cref{lemma:two_minimizers_shallow} and
\cref{lemma:upper_boundnperp_shallow_non_unique}.
Inequality (b) is due to Assumption \eqref{assumption_shallow_non-unique}.
Then we obtain that
\begin{align*}
	    \norm{ \widetilde{\npara_{\SP}}}_{\ell^1}
		&\le
		\frac{
		16 \norm{\ga}_{\ell^1}
		\varepsilon^{1-\varrho}
		\tilde{\varrho}
		\abs{\SC}
		\kappa(\ga)^{\varrho^{-}}
		\norm{
			 \nperp_{\SP}}_{\ell^{\infty}}
		}
		{\min_{i\in \SP}  \abs{\ga_i}}
		\cdot
		\underset{\le 2}
		{\underbrace{\frac{1}{ \left( 1- 
			8\varepsilon^{1-\varrho}
			\tilde{\varrho}
			\abs{\SC}
			\kappa(\ga)^{\varrho^{-}}
			\norm{
				\nperp_{\SP}}_{\ell^{\infty}}/{\min_{i\in \SP}  \abs{\ga_i}} \right) }}}\\
		&\overleq{(a)}
		\frac{
		32 \norm{\ga}_{\ell^1}
		\varepsilon^{1-\varrho}
		\tilde{\varrho}
		\abs{\SC}
		\kappa(\ga)^{\varrho^{-}}
		\norm{
			 \nperp_{\SP}}_{\ell^{\infty}}
		}
		{\min_{i\in \SP}  \abs{\ga_i}}\\
		&\le
		\frac{
		32 \norm{\ga}_{\ell^1}
		\varepsilon^{1-\varrho}
		\tilde{\varrho}^2
		\abs{\SC}
		\kappa(\ga)^{\varrho^{-}}
		\norm{
			 \nperp_{\SC}}_{\ell^{1}}
		}
		{\min_{i\in \SP}  \abs{\ga_i}},
\end{align*}
where inequality $(a)$ follows from 
\begin{align*}
    \frac{8\varepsilon^{1-\varrho} \tilde{\varrho} \abs{\SC}\kappa(\ga)^{\varrho^{-}}
	      \norm{\nperp_{\SP}}_{\ell^{\infty}}}
         {\min_{i\in \SP}  \abs{\ga_i}}	
	\le
    2\varepsilon^{1-\varrho} \tilde{\varrho} \abs{\SC}\kappa(\ga)^{\varrho^{-}}
	\le \frac{1}{2},
\end{align*}
which is due to 
\cref{lemma:upper_boundnperp_shallow_non_unique},
which states that
$ \norm{\nperp_{\SP}}_{\ell^{\infty}} \le ( \min_{ i \in \SP} \abs{\ga_i})/4  $,
and Assumption \eqref{assumption_shallow_non-unique}.
This completes the proof of \cref{lemma:nparallelbound_shallow_non_unique}.
\end{proof}

\noindent \textbf{Step 4 (Combining the bounds):}
Having established upper bounds for 
$\norm{\xa - \ga}_{\ell^1}$,
$\norm{\nperp_{\SC}}_{\ell^1}$, and $\norm{\widetilde{\npara}}_{\ell^1}$, 
we can now combine them to obtain the final result.

	\begin{proof}[Proof of \cref{result:upper_bound_shallow_non_unique}]
	In order to complete the proof, we combine all the previously obtained bounds.
	\begin{align*}
		\norm{\xinf-\ga}_{\ell^1} 
		\overleq{(a)}&
		\norm{\nperp}_{\ell^1} 
		+ \norm{\widetilde{\npara}}_{\ell^1}
		+\norm{\xa-\ga}_{\ell^1}\\
		\overleq{(b)} &
		\norm{\nperp}_{\ell^1} 
		+
		\frac{
		32 \norm{\ga}_{\ell^1}
		\varepsilon^{1-\varrho}
		\tilde{\varrho}^2
		\abs{\SC}
		\kappa(\ga)^{\varrho^{-}}
		\norm{
			 \nperp_{\SC}}_{\ell^{1}}
		}
		{\min_{i\in \SP}  \abs{\ga_i}}
		+
		(1+2\varepsilon) \varepsilon^2 \norm{\ga}_{\ell^1}\\
		\le&
		\left(
		1+\tilde{\varrho}
		+
		\underset{=: C_1}
		{\underbrace{
		\frac{
		32 \norm{\ga}_{\ell^1}
		\tilde{\varrho}^2
		\abs{\SC}
		\kappa(\ga)^{\varrho^{-}}
		}
		{\min_{i\in \SP}  \abs{\ga_i}}
		}}
		\cdot
		\varepsilon^{1-\varrho}
		\right)
		\norm{\nperp_{\SC}}_{\ell^1}
		+
		\underset{\le 2}
		{\underbrace{(1+2\varepsilon)}}
		\varepsilon^2 \norm{\ga}_{\ell^1}\\
		\overleq{(c)} &
		\left(
		1+\tilde{\varrho}
		+
		C_1 \varepsilon^{1-\varrho}
		\right)
		\norm{\nperp_{\SC}}_{\ell^1}
		+
		2\varepsilon^2 \norm{\ga}_{\ell^1}\\
		\overleq{(d)} &
		\left(
		1+\tilde{\varrho}
		+
		C_1 \varepsilon^{1-\varrho}
		\right)
		\alpha^{1-\varrho}
		\abs{\SC}
		\left( \min_{i\in \SP}\abs{\ga_i} \right)^{\varrho}
		\kappa(\ga)^{\varrho^{-}}
		h \left( \varepsilon \right)
		+
		2\varepsilon^2 \norm{\ga}_{\ell^1}.
	\end{align*}
	Inequality $(a)$ follows from \cref{ineq:Dominiklast1}.
	In inequality $(b)$ 
	we have used \cref{lemma:nparallelbound_shallow_non_unique} 
	and \cref{lemma:two_minimizers_shallow}.
	Inequality $(c)$ follows from $\varepsilon \le 1/2$,
	see Assumption \eqref{assumption_shallow_non-unique}.
	Inequality $(d)$ follows from \cref{lemma:upper_boundnperp_shallow_non_unique},
	where the function $h$ is as defined in this lemma.
	This completes the proof of \cref{result:upper_bound_shallow_non_unique}.
\end{proof}

%% file: parts/proofs_Dgreatertwo_non-unique.tex
\subsection{Case $D\ge3$ : Proof of \cref{theorem:upper_bound_deep_non_unique}}
Recall that
\begin{equation*}
	\ga = \argmax_{x \in \Lmin} \norm{x}_{\ell^{\frac{2}{D}}}
	\qquad\text{and}\qquad
	\xinf \in \argmin_{x:Ax=y} D_{\QQ}(x,0).
\end{equation*}
In order to prove \cref{theorem:upper_bound_deep_non_unique} we need to obtain 
an upper bound for $\norm{\ga-\xinf}_{\ell^1}$.
We will proceed similarly as in the proof  
of \cref{result:upper_bound_shallow_non_unique},
which is concerned with the shallow non-unique case.

As before, we define $n:= \xinf -\ga \in \ker(A)$. 
Then, by \eqref{direct_sum} and \cref{lemma:normal_cone_deep}, 
there exist uniquely defined $\npara \in \Tan$ and $\nperp \in \Nor$ 
such that 
$n = \nperp + \npara$.
Next, we define the auxiliary point
\begin{equation} \label{optimality_x_star}
	\xa \in \argmin_{x \in \Lmin} D_{\QQ}(x,0).
\end{equation}
Due to \eqref{direct_sum} and \cref{lemma:normal_cone_deep},
we can decompose $\xinf-\xa \in \ker (A)$ into
$\xinf-\xa = \widetilde{ \npara} + \widetilde{\nperp}$
with $\widetilde{\npara} \in \Tan$ and $ \widetilde{\nperp} \in \Nor$.
As in the case of $D=2$, because of $ \ga \in \Lmin$ and $ \xa \in \Lmin $, it holds that $ \ga - \xa \in \Tan $,
which implies $ \widetilde{\nperp} = \nperp $.
It follows that
\begin{align}
	\xinf - \ga
	=
	(\xinf - \xa)
	 + (\xa - \ga)
	=
	    \widetilde{\npara} + \nperp	
		 + (\xa - \ga).
\end{align}
From the triangle inequality it follows  that
\begin{align}\label{ineq:Dominiklast11}
	\norm{\xinf - \ga}_{\ell^1}
	\le
	\norm{\xa - \ga}_{\ell^1}
	+
	\norm{ \widetilde{\npara} }_{\ell^1}
	+
	\norm{\nperp}_{\ell^1}.
\end{align}
Similarly, as in the proof of the shallow case,
we will bound these three terms individually.

To keep the notation more concise in our proof,
we will introduce the following notation:
\begin{equation*}
	\eps:= \frac{\alpha}{\mg}.
\end{equation*}

\noindent \textbf{Step 1 (Controlling $\norm{\xa -\ga}_{\ell^1}$):}
We will first provide an equivalent characterization of $\ga$
which will allow us to compare $\xa$ and $\ga$ more easily.
For that purpose, define the function $\gD\colon [0,\infty)\to \R$ by
\begin{equation*}
	\gD(u) := u- \frac{D}{2} u^{\frac{2}{D}}.
\end{equation*}
Then we can define
$\GD\colon \R^d_{\ge0}\to\R$ as
\begin{equation*}
	\GD(x) 
	:= 
	\sum_{i=1}^d \alpha  \gD\Big( \frac{x_i}{\alpha}\Big), \quad x\in \R^d_{\ge0}.
\end{equation*}

The following lemma is an adaption of \cite[Proposition 3.20]{wind2023implicit} to our setting and notation and shows that 
$\ga$ can be equivalently characterized as a minimizer of $\GD$ on $\Lmin$. 
For the convenience of the reader, we have included a proof in \cref{section:proof_of_lemma:two_minimizers_deep}.

\begin{lemma} \label[lemma]{lemma:two_minimizers_deep}
	Let $d,A,y$ as in \cref{assumption_on_A_y_non-unique}. Let $D\in \N$ with $D\ge 3$ and $\alpha>0$. Then
	\begin{equation} \label{eqn1020}
		\ga 
		=\argmin_{x\in \Lmin} \GD(\abs{x}).
	\end{equation}
\end{lemma}

In order to bound $\norm{\xa -\ga}_{\ell^1}$,
we will need to compare $\QQ$ and $\GD$.
This can be done using the following lemma,
which provides a bound between $\hh^{-1}$ and $\gD'$.
Moreover, this lemma contains several inequalities
which are useful for describing the asymptotic behavior of  $\hh^{-1}$.
They will be useful throughout the proof of \cref{theorem:upper_bound_deep_non_unique}.
The proof of the next lemma has 
been deferred to \cref{appendix:basic_inequalities_deep_non-unique}.
\begin{lemma}\label[lemma]{lemma:basic_inequalities_deep_non-unique}
	Let $D\in \N$ with $D\ge 3$ and $\gamma:=\frac{D-2}{D}$.
	Then the following statements hold:
	\begin{enumerate}[(i)]
	
		\item 	For $u,v>0$ we have
		\begin{equation}\label{eqn606-non-unique}
			0 \le \hh^{-1}(u)-\gD'(u) \le 
			\frac{\gamma}{u^{1+\gamma}}.
		\end{equation}
	
		\item 	For all $u\ge 1$ we have
		\begin{equation}\label{equ:DominikInternfinal1}
			\frac{\gamma}
			{\left(1+\frac{5}{u}\right)u^{1+\gamma}}
			\le \big(\hh^{-1}\big)'(u) 
			\le \frac{\gamma}{u^{1+\gamma}}
		\end{equation}
		and
		\begin{equation} \label{eqn:bounds_on_h_D_5}
			0 \le \frac{\gamma}{u^{1+\gamma}} - \big(\hh^{-1}\big)'(u) \le \frac{5\gamma}{u^{2+\gamma}}.
		\end{equation}
	
		\item For all $u\ge 1$ we have
		\begin{equation} \label{eqn:bounds_on_h_D_4}
			0
			\le
			\big(\hh^{-1}\big)''(u) 
			\le 16 \gamma  u^{-2-\gamma}.
		\end{equation}
	\end{enumerate}
\end{lemma}

To show that $\xa$ and $\ga$ are close to each other, we will use the strong convexity of $\QQ$.
This property of $\QQ$ is shown by the following lemma,
whose proof has been deferred to \cref{appendix:strong_convexity_deep}.
\begin{lemma} \label[lemma]{lemma:strong_convexity_deep}
	Let $D\in \N$ with $D\ge 3$ and $\gamma:=\frac{D-2}{D}$. Let $\alpha>0$, and $x,n \in \R^d$. 
	Then it holds that
	\begin{equation} \label{eqn:strong_convexity_deep_1}
		\langle \nabla^2 \QQ(x)n, n\rangle 
		\ge 
		\frac{\norm{n}_{\ell^1}^2  \gamma\alpha^{\gamma}}
		{3 \abs{\supp(n)} \alpha^{1+\gamma} + 2\norm{x}_{\ell^{1+\gamma}}^{1+\gamma}}.
	\end{equation}
	If $\supp(n)\subset \supp(x)=:S$ and $\alpha<\min_{i\in S} \abs{x_i}$, then
	\begin{equation} \label{eqn:strong_convexity_deep_2}
		\langle \nabla^2 \QQ(x)n, n\rangle 
		\ge  
		\frac{\norm{n}_{\ell^1}^2 \gamma\alpha^{\gamma}}
		{\norm{x}_{\ell^{1+\gamma}}^{1+\gamma} 
		\left(1+ \frac{5\alpha}{\min_{i\in S}\abs{x_i}} \right)}.
	\end{equation}
\end{lemma}
With these preparations in place,
we can now prove the upper bound for $\norm{\xa -\ga}_{\ell^1}$.
\begin{lemma}\label[lemma]{lemma:two_minimizers_deep_2}
	Assume that $\alpha < \min_{i \in \SP} \abs{\ga_i}/2$.
	Then the following statements hold:
	\begin{enumerate}
		\item \label{two_minimizers_deep_2_item_1}
		We have
		\begin{equation} \label{eqn1315}
			\norm{\xa-\ga}_{\ell^1} 
			\le  \alpha  
			\Big[ 2\Big(\frac{\norm{\ga}_{\ell^{1}}}{\min_{i\in \SP}\abs{\ga_i}}\Big)^{1+\gamma}
			+3\abs{\SP} \eps^{1+\gamma} 
			\Big].
		\end{equation}
		
		\item \label{two_minimizers_deep_2_item_2}In addition,
		if it holds that 
		\begin{equation} \label{eqn1314}
			\frac{\alpha}{\mg} 
			\le \frac{1}{8} \Big(\frac{\min_{i\in \SP}\abs{\ga_i}}{\norm{\ga}_{\ell^{1}}}\Big)^{1+\gamma},
		\end{equation}
		we have that 
		\begin{equation} \label{eqn1316}
			\norm{\xa-\ga}_{\ell^1} \le \frac{1}{2}\mg
		\end{equation}
		and thus $ \supp (\xa ) = \supp (\ga )$.
		Furthermore, 
		then also the stronger bound
		\begin{equation} \label{eqn1317}
			\norm{\xa-\ga}_{\ell^1} 
		\le  
		\alpha  \Big( \frac{\norm{\ga}_{\ell^{1}}}{\mg}\Big)^{1+\gamma} 
		(1+10 \eps)
		\end{equation}
		holds.
	\end{enumerate}
\end{lemma}
\begin{proof}
	\emph{Proof of Statement \ref{two_minimizers_deep_2_item_1}.}
	Let $ \hat{n} := \xa-\ga$, and $ \hat{n}^*  := \sigma \odot \hat{n} $, where $\sigma$ 
	is as defined in \cref{basic_properties_Lmin}.
    We have $\supp(\hat{n})\subset \supp(\xa)\cup \supp(\ga)$. 
	By \cref{lemma:maximal_support_deep} we get 
	$\supp(\xa)\cup \supp(\ga)= \supp(\ga)$ and so $\hat{n}_{\SC}=0$.
	Hence, 
	the map $t\mapsto \GD(\abs{\ga +t \hat{n} })$ is differentiable at $t=0$.

	By \cref{lemma:two_minimizers_deep}, $\ga$ minimizes $\GD$ over $\Lmin$.
	Hence, we have that
	\begin{equation} \label{eqn707}
		0 =\left.\frac{\dif}{\dif t}\right\rvert_{t=0} 
		\GD(\abs{\ga +t \hat{n} })
		= \langle \nabla \GD(\abs{\ga_{\SP}}),\hat{n}^*_{\SP} \rangle.
	\end{equation}
	Furthermore, since $\Lmin$ is convex, we have $\xa-t \hat{n} \in \Lmin$ for all $t\in [0,1)$. 
	By optimality of $\xa$, we get
	\begin{equation} \label{eqn708}
		0 \le \left.\frac{\dif}{\dif t}\right\rvert_{t=0} D_{\QQ}(\xa+t(- \hat{n} ),0)
		= \langle \nabla\QQ(\xa_{\SP}), -\hat{n}_{\SP}\rangle.
	\end{equation}
	Moreover, we have
	\begin{equation} \label{eqn709}
		\begin{split}
			\langle \nabla\QQ(\xa_{\SP}) - \nabla\QQ(\ga_{\SP}), \hat{n}_{\SP}\rangle
			&= \langle \int_0^1 \left. \frac{\dif}{\dif s}\right\rvert_{s=t} \nabla\QQ(\ga_{\SP}+s\hat{n}_{\SP}) \dif s, \hat{n}_{\SP}\rangle \\
			&=  \int_0^1 \langle \nabla^2 \QQ(\ga_{\SP}+t\hat{n}_{\SP}) \hat{n}_{\SP}, 
			\hat{n}_{\SP}\rangle \dif t.
		\end{split}
	\end{equation}
	Combining \eqref{eqn707}, \eqref{eqn708}, and \eqref{eqn709}, we obtain
	\begin{equation}\label{eqn710}
		\int_0^1 \langle \nabla^2 \QQ(\ga_{\SP}+t \hat{n}_{\SP}) \hat{n}_{\SP}, 
		\hat{n}_{\SP}\rangle \dif t
		\le \langle \nabla \GD(\abs{\ga_{\SP}}), \hat{n}^*_{\SP} \rangle 
		- \langle  \nabla\QQ(\ga_{\SP}), \hat{n}_{\SP}\rangle.
	\end{equation} 
	Let us first consider the right-hand side of \eqref{eqn710}. 
	Since $\hh^{-1}$ is an odd function, we have
	\begin{equation*}
		\begin{split}
			\langle \nabla \GD(\abs{\ga_{\SP}}),\hat{n}^*_{\SP}\rangle 
			- \langle  \nabla\QQ(\ga_{\SP}), \hat{n}_{\SP}\rangle
			&=\sum_{i\in \SP} \gD'\Big( \frac{\abs{\ga_i}}{\alpha}\Big) \hat{n}_i^* 
			- \hh^{-1}\Big(\frac{\ga_i}{\alpha}\Big) \hat{n}_i\\
			&=\sum_{i\in \SP} \Big[ \gD'\Big( \frac{\abs{\ga_i}}{\alpha}\Big) 
			- \hh^{-1}\Big(\frac{\abs{\ga_i}}{\alpha}\Big)\Big] \hat{n}_i^* \\
			&\le \norm{\hat{n}_{\SP}}_{\ell^1} \sup_{i\in \SP} 
			\abs{ \gD'\Big( \frac{\abs{\ga_i}}{\alpha}\Big) 
				- \hh^{-1}\Big(\frac{\abs{\ga_i}}{\alpha}\Big)}.
		\end{split}
	\end{equation*}
	By \cref{lemma:basic_inequalities_deep_non-unique}, 
	see \eqref{eqn606-non-unique},
	we have for all $i\in \SP$ that
	\begin{equation*}
		\abs{ \gD'\Big( \frac{\abs{\ga_i}}{\alpha}\Big) 
			- \hh^{-1}\Big(\frac{\abs{\ga_i}}{\alpha}\Big)}
		\le \gamma  \Big(\frac{\alpha}{\abs{\ga_i}}\Big)^{1+\gamma}
		\le \gamma  \Big(\frac{\alpha}{\min_{i\in \SP}\abs{\ga_i}}\Big)^{1+\gamma},
	\end{equation*}
	and so
	\begin{equation}\label{eqn711}
		\langle \nabla \GD(\abs{\ga_{\SP}}), \hat{n}^*_{\SP} \rangle 
		- \langle  \nabla\QQ(\ga_{\SP}), \hat{n}_{\SP}\rangle
		\le \gamma  \norm{\hat{n}_{\SP}}_{\ell^1} 
		\Big(\frac{\alpha}{\min_{i\in \SP}\abs{\ga_i}}\Big)^{1+\gamma}.
	\end{equation}
	For all $t \in (0,1)$ we have
	\begin{equation*} 
		\norm{\ga_{\SP} + t \hat{n}_{\SP}}_{\ell^{1+\gamma}}^{1+\gamma}
		\le \norm{\ga_{\SP}+t \hat{n}_{\SP}}_{\ell^1}^{1+\gamma}
		= \norm{\ga}_{\ell^1}^{1+\gamma},
	\end{equation*}
	where the inequality follows from $\norm{\cdot}_{\ell^{1+\gamma}} \le \norm{\cdot}_{\ell^1}$
	and the equation holds due to the fact 
	that $\ga_{\SP}+ t \hat{n}_{\SP} \in \Lmin$
	and the $\ell^1$-norm is constant on $\Lmin$ by definition.
	Furthermore,
	we have that 
	$\supp( \hat{n} )\subset \SP$. Therefore, \cref{lemma:strong_convexity_deep} implies that
	\begin{equation} \label{eqn714}
		\int_0^1 \langle \nabla^2 \QQ(\ga_{\SP}+t \hat{n}_{\SP}) \hat{n}_{\SP}, \hat{n}_{\SP}\rangle \dif t
		\ge 
		 \frac{\norm{ \hat{n}_{\SP}}_{\ell^1}^2  \gamma \alpha^{\gamma}}
		 {3 \abs{\SP} \alpha^{1+\gamma} + 2\norm{\ga}_{\ell^{1}}^{1+\gamma}}.
	\end{equation}
	Inserting \eqref{eqn711} and \eqref{eqn714} into \eqref{eqn710}, we infer that
	\begin{equation*}
		\frac{\norm{\hat{n}_{\SP}}_{\ell^1}^2  
		\gamma\alpha^{\gamma}}{3\abs{\SP} \alpha^{1+\gamma} + 2\norm{\ga}_{\ell^{1}}^{1+\gamma}}
		\le\gamma  \norm{ \hat{n}_{\SP}}_{\ell^1}  
		\Big(\frac{\alpha}{\min_{i\in \SP}\abs{\ga_i}}\Big)^{1+\gamma}.
	\end{equation*}
	Therefore,
	we obtain that 
	\begin{align*}
		\norm{\hat{n}_{\SP}}_{\ell^1} 
		\le  
		&\frac{\alpha  \big( 3\abs{\SP} \alpha^{1+\gamma} + 2\norm{\ga}_{\ell^{1}}^{1+\gamma}\big)}
		{\min_{i\in \SP}\abs{\ga_i}^{1+\gamma}}\\
		=
		& \alpha
		\left(
			3 \abs{\SP} \varepsilon^{1+\gamma} + 2 \Big(\frac{\norm{\ga}_{\ell^{1}}}{\min_{i\in \SP}\abs{\ga_i}}\Big)^{1+\gamma}
		\right).
	\end{align*}
	This proves the first statement.\\

	\noindent \emph{Proof of Statement \ref{two_minimizers_deep_2_item_2}.}
	In the following, we assume in addition that Assumption \eqref{eqn1314} holds. 
	We have that 
	\begin{equation*}
		3 \eps^{1+\gamma} \abs{\SP} 
		\le 3\eps^{1+\gamma} \frac{\norm{\ga}_{\ell^1}}{\mg}
		\le 3\eps^{1+\gamma} \Big(\frac{\norm{\ga}_{\ell^1}}{\mg}\Big)^{1+\gamma}
		\overleq{(a)} 2 \Big(\frac{\norm{\ga}_{\ell^1}}{\mg}\Big)^{1+\gamma},
	\end{equation*}
	where inequality (a) holds due to Assumption \eqref{eqn1500}.
	Then we obtain that
	\begin{equation*}
		\norm{\xa-\ga}_{\ell^1} 
		\overleq{(a)}  \alpha \cdot 
		\Big[ 2\Big(\frac{\norm{\ga}_{\ell^{1}}}{\min_{i\in \SP}\abs{\ga_i}}\Big)^{1+\gamma}
		+3\abs{\SP} \eps^{1+\gamma} 
		\Big]
		\overleq{(b)} 4 \alpha \Big(\frac{\norm{\ga}_{\ell^{1}}}{\min_{i\in \SP}\abs{\ga_i}}\Big)^{1+\gamma}
		\overleq{(c)} \frac{\min_{i\in \SP}\abs{\ga_i}}{2} ,
	\end{equation*}
	where in inequality (a) we used the first statement of this proposition,
	\cref{eqn1315}.
	Inequality (b) follows from 
	the above estimate.
	Inequality (c) follows from Assumption \eqref{eqn1314}.
	This proves \cref{eqn1316}.
	
	Therefore, we have also shown that $\supp (\xa)=\supp ( \ga )$.
	Moreover, for all $i\in \SP$ we have
	\begin{equation*}
		\abs{\ga_i+t \hat{n}_i}\ge \mg -\norm{\xa-\ga}_{\ell^1} 
		\ge\frac{\min_{i \in \SP} \abs{\ga_i} }{2} > \alpha,
	\end{equation*}
	where in the last step we have used again Assumption \eqref{eqn1500}.
	Hence, for all $t\in(0,1)$, we have
	\begin{equation}\label{eqn:711}
		\frac{5\alpha}{\min_{i\in \SP} \abs{\ga_i+ t \hat{n}_i}}
		\le  \frac{10 \alpha}{\mg} = 10 \eps.
	\end{equation}
	Then from \cref{lemma:strong_convexity_deep}, see \cref{eqn:strong_convexity_deep_2}, 
	it follows that
	\begin{align} 
		\int_0^1 \langle \nabla^2 \QQ(\ga_{\SP}+t\hat{n}_{\SP}) \hat{n}_{\SP}, \hat{n}_{\SP}\rangle \dif t
		&\ge 
		\frac{\norm{ \hat{n}_{\SP}}_{\ell^1}^2  \gamma  \alpha^{\gamma}}
		     { \max_{t \in (0,1)} \left( \norm{\ga + t \hat{n}_{\SP}}_{\ell^{1+\gamma}}^{1+\gamma} 
		\cdot (1+\frac{5\alpha}{\min_{i\in \SP}\abs{\ga_i + t \hat{n}_i }}) \right)} \nonumber \\
		&\ge 
		\frac{\norm{ \hat{n}_{\SP}}_{\ell^1}^2  \gamma  \alpha^{\gamma}}
		     { \max_{t \in (0,1)}  \norm{\ga + t \hat{n}_{\SP}}_{\ell^{1+\gamma}}^{1+\gamma} 
		\cdot (1+10\varepsilon ) }, \label{eqn1313}
	\end{align}
	where in the last line we have used \cref{eqn:711}.
	Inserting \eqref{eqn711} and \eqref{eqn1313} into \eqref{eqn710}, we infer that
	\begin{align*}
	\norm{\hat{n}_{\SP}}_{\ell^1} 
	\le 
	\frac{\alpha \max_{t \in (0,1)} \norm{\ga + t \hat{n}_{\SP}}_{\ell^{1+\gamma}}^{1+\gamma}}
	{\min_{i \in \SP} \abs{\ga_i}^{1+\gamma} }
	(1+10\varepsilon ).
	\end{align*}
	Now note that
	\begin{align*}
		\norm{ \ga + t \hat{n}_{\SP}}_{\ell^{1+\gamma}}
		\le 
		\norm{\ga + t \hat{n}_{\SP}}_{\ell^1}
		=
		\norm{\ga}_{\ell^1}.
	\end{align*}
	Here we have used 
	for the inequality that $\norm{\cdot}_{\ell^{1+\gamma}} \le \norm{\cdot}_{\ell^1}$
	and for the equation we have used $\ga + t \hat{n}_{\SP} \in \Lmin$ 
	and the fact that $\ell^1$-norm is constant on $\Lmin$. 
	Thus, it follows that
	\begin{equation*}
		\norm{\hat{n}_{\SP}}_{\ell^1} 
		\le  \alpha  \Big( \frac{\norm{\ga}_{\ell^{1}}}{\mg}\Big)^{1+\gamma} 
		(1+10 \eps).
	\end{equation*}
	This completes the proof.
\end{proof}
\noindent \textbf{Step 2 (Controlling $\norm{\nperp}_{\ell^1}$):}
As a next step,
we will provide an upper bound for $\norm{\nperp}_{\ell^1}$.
For this prove,
we will use similar arguments as in the scenario where there exists a unique solution.
\begin{lemma}\label[lemma]{lemma:nperpbound_deep_non_unique}
Assume that the assumptions of \cref{theorem:upper_bound_deep_non_unique} holds.
Then it holds that 
	\begin{equation}\label{eqn905}
		\norm{\nperp_{\SC}}_{\ell^1}
		\le 
		\alpha \abs{\SC}
		\Big[ 
		\hh(\varrho) 
		+\frac{4\cdot 2^\gamma \varrho^{-}}{\gamma (1-\varrho)^{\frac{1}{\gamma}+1}} 
		 \left( \frac{\alpha}{\min_{i \in \SP} \abs{ \ga_i }} \right)^{\gamma}
		\Big].
	\end{equation}
	In particular, it holds that 
    $ \norm{ \nperp_{\SP} }_{\ell^\infty} \le \min_{i \in \SP} \abs{\ga_i}/4$ for all $i\in \SP$.
\end{lemma}

\begin{proof}
	If $\nperp_{\SC}=0$, then \cref{lemma:normal_cone} implies that $\nperp =0$. In the following, we will assume that $\nperp_{\SC}\ne 0$.
	By optimality of $\xinf$ we have
	\begin{equation} \label{eqn700a}
			0 = \left.\frac{\dif}{\dif t}\right|_{t=0} D_{\QQ}(\xinf 
			+ t (\widetilde{\npara} + \nperp) ,0)
			= \langle \nabla \QQ(\xinf),\widetilde{\npara} + \nperp\rangle.
	\end{equation}
	Recall that $\xinf = \xa + \widetilde{\npara} + \nperp$ 
	and that $\xa_{\SC}= \widetilde{\npara_{\SC}} =0$.
	Separating the right-hand side of \eqref{eqn700a} into $\SP$ and $\SC$ at $(a)$,
	we infer that 
	\begin{equation} \label{eqn700b}
		\begin{split}
		-\langle \nabla \QQ (\xa_{\SP}+\widetilde{\npara_{\SP}} + \nperp_{\SP} ),
		\widetilde{\npara_{\SP}}  +  \nperp_{\SP}\rangle
		=& -\langle \nabla \QQ(\xinf_{\SP}), \widetilde{\npara_{\SP}}  +  \nperp_{\SP}\rangle\\
		\overeq{(a)} &\langle \nabla \QQ(\xinf_{\SC}),\widetilde{\npara_{\SC}}  +  \nperp_{\SC}\rangle\\
		=& \langle \nabla \QQ( \nperp_{\SC}), \nperp_{\SC}\rangle.
		\end{split}
	\end{equation}
	Since $\QQ$ is convex, its gradient is monotone.
	Therefore,
	it holds that
	\begin{equation} \label{eqn700c}
		\langle \nabla \QQ(\xa_{\SP}+ \widetilde{ \npara_{\SP}} + \nperp_{\SP} ),
		\widetilde{ \npara_{\SP}} + \nperp_{\SP}\rangle 
		\ge \langle \nabla \QQ(\xa_{\SP}),\widetilde{ \npara_{\SP}} + \nperp_{\SP}\rangle.
	\end{equation}
	Inserting \eqref{eqn700c} into \eqref{eqn700b}, we deduce that
	\begin{equation} \label{eqn700}
		-\langle \nabla \QQ(\xa_{\SP}), \widetilde{\npara_{\SP}} + \nperp_{\SP}  \rangle
		\ge \langle \nabla \QQ( \nperp_{\SC}), \nperp_{\SC}\rangle.
	\end{equation}
	In order to simplify the left-hand side of \eqref{eqn700},
	we invoke the optimality of $\xa$, see \eqref{optimality_x_star}.
	Namely, since $\supp(\xa)=\SP$ and $\widetilde{\npara} \in \Tcal$, 
	it follows from \cref{lemma:tangent_cone_characterization} 
	and \cref{lemma:two_minimizers_deep_2} 
	that $\xa+t \widetilde{\npara} \in \Lmin$ for all sufficiently small $t>0$.
	Analogously, replacing $\widetilde{\npara}$ by $-\widetilde{\npara}$, 
	we also have $\xa+t\npara = \xa+ \abs{t}\cdot (-\widetilde{\npara}) \in \Lmin$ for all $t<0$ 
	with $\abs{t}$ sufficiently small. 
	Therefore, using the first order optimality condition at $(a)$ 
	and the identity $\widetilde{\npara_{\SC}}=0$, 
	see \cref{lemma:tangent_cone_characterization}, at $(b)$, we have
	\begin{equation*} 
		0 \overeq{(a)} \left.\frac{\dif}{\dif t}\right|_{t=0} 
		D_{\QQ}(\xa + t \widetilde{ \npara},0)
		= \langle \nabla \QQ(\xa), \widetilde{\npara} \rangle
		\overeq{(b)} \langle \nabla \QQ(\xa_{\SP}),\widetilde{\npara_{\SP}}\rangle.
	\end{equation*}
	Inserting this equation 
	into \eqref{eqn700}, we obtain
	\begin{equation} \label{eqn700d}
		-\langle \nabla \QQ(\xa_{\SP}),\nperp_{\SP}\rangle 
		\ge \langle \nabla \QQ(\nperp_{\SC}),\nperp_{\SC}\rangle.
	\end{equation}

	We note that inequality \eqref{eqn700d} is analogous to inequality \eqref{eqn600} 
	in the proof of \cref{theorem:upper_bound_deep} 
	with $\nperp$ instead of $n$ and $\xa$ instead of $\ga$. 
	Furthermore, we note that 
	\begin{align*}
		\frac{\alpha}{ \min_{i \in \SP} \abs{ \xa_i }}
		\overleq{(a)}
		\frac{2 \alpha}{ \min_{i \in \SP} \abs{ \ga_i }}
		\overleq{(b)}
		\left(
			\frac{(1-\varrho) \gamma }{4 \varrho^{-}}
		\right)^{1/\gamma},
	\end{align*}
	where in  
	inequality (a) we have used
	that $\min_{i \in \SP} \abs{ \xa_i } \ge \frac{1}{2} \min_{i \in \SP} \abs{ \ga_i }$
	due to \cref{lemma:two_minimizers_deep_2}
	and in inequality (b) we have used Assumption \eqref{eqn1500}. 
	Note that this inequality is analogous 
	to Assumption \eqref{eqn1001} in \cref{theorem:upper_bound_deep},
	where $\ga$ is replaced by $\xa$. 
	Therefore, by proceeding analogously 
	as in the proof of \cref{theorem:upper_bound_deep},
	we obtain that 
	\begin{equation*}
		\norm{\nperp_{\SC}}_{\ell^1}
		\le \alpha \abs{\SC}
		\Big[ 
		\hh(\varrho) 
		+\frac{4\varrho^{-}}{\gamma (1-\varrho)^{\frac{1}{\gamma}+1}} 
		\cdot \left( \frac{\alpha}{\min_{i \in \SP} \abs{ \xa_i }} \right)^{\gamma}
		\Big].
	\end{equation*}
	Now recall
	that $\min_{i \in \SP} \abs{ \xa_i } \ge \frac{1}{2} \min_{i \in \SP} \abs{ \ga_i }$.
	Therefore, we have
	\begin{equation*}
		\norm{\nperp_{\SC}}_{\ell^1}
		\le 
		\alpha \abs{\SC}
		\Big[ 
		\hh(\varrho) 
		+\frac{4\cdot 2^\gamma  \varrho^{-}}{\gamma (1-\varrho)^{\frac{1}{\gamma}+1}} 
		\cdot \left( \frac{\alpha}{\min_{i \in \SP} \abs{ \ga_i }} \right)^{\gamma}
		\Big].
	\end{equation*}
	This proves \cref{eqn905}.
	It remains to show that
	$\norm{\nperp_{\SP}}_{\ell^\infty}  \le \min_{i \in \SP} \abs{\ga_i}/4$.
    We observe that
	\begin{align*}
		\norm{\nperp_{\SP}}_{\ell^\infty}
		\overleq{(a)}&
		\tilde{\varrho} \norm{\nperp_{\SC}}_{\ell^1}\\
        \overleq{(b)}&
		 \tilde{\varrho} \alpha \abs{\SC}
		\Big[ 
		\hh(\varrho) 
		+\frac{4\cdot 2^\gamma  \varrho^{-}}{\gamma (1-\varrho)^{\frac{1}{\gamma}+1}} 
		\cdot \left( \frac{\alpha}{\min_{i \in \SP} \abs{ \ga_i }} \right)^{\gamma}
		\Big]\\
		\overleq{(c)}&
		2 \tilde{\varrho} \alpha \abs{\SC}
		\left( \hh(\varrho) +1 \right) \\
		\overleq{(d)}&
		\frac{1}{4} \min_{i \in \SP} \abs{\ga_i}.
	\end{align*}
	In inequality (a) we have used the definition of $\tilde{\varrho}$,
	see \cref{eqn102}.
	Inequality (b) follows from \cref{eqn905}
	and inequalities (c) and (d) follow 
	both from Assumption \eqref{eqn1500}.
\end{proof}

\noindent \textbf{Step 3 (Bounding $\norm{\tildenpara}_{\ell^1}$):}
In order to conclude, it remains to prove an upper bound for $\norm{\tildenpara}_{\ell^1}$.
For this proof, we will need the following a-priori bound for $\xinf$.
\begin{lemma}[A priori bound] \label{lemma:a_priori_bounds_deep}
	Let $d,A,y$ as in Assumption \eqref{assumption_on_A_y_non-unique}.
	Let $D \in \N$ with $D\ge 3$ and let $\alpha>0$.
	Let
		$\tilde{g} \in \Lmin$
	be arbitrary.
	Then it holds that
	\begin{equation*}
		\norm{\xinf}_{\ell^1} 
		\le d  \norm{\tilde{g}}_{\ell^{\infty}}.
	\end{equation*}
\end{lemma}
\begin{proof}
	It follows from the definition of $\xinf$ that 
	\begin{equation} \label{eqn624}
		\QQ(\xinf) = D_{\QQ}(\xinf,0) \le D_{\QQ}(\tilde{g},0) = \QQ(\tilde{g}).
	\end{equation}
	Furthermore, using that $\qq$ is an even function and that it is convex on $[0,\infty)$, we infer that
	\begin{equation}\label{eqn625}
		\begin{split}
			\QQ(\xinf) 
			&= \alpha \sum_{i=1}^{d} \qq\Big( \frac{\xinf_i}{\alpha}\Big) 
			= \alpha \sum_{i=1}^d \qq\Big( \frac{\abs{\xinf_i}}{\alpha}\Big)
			= \alpha d \sum_{i=1}^d \frac{\qq\Big( \frac{\abs{\xinf_i}}{\alpha}\Big)}{d} \\
			&\ge \alpha d   \qq\left( \sum_{i=1}^d \frac{\abs{\xinf_i}}{\alpha d}\right)
			= \alpha d  \qq \left( \frac{\norm{\xinf}_{\ell^1}}{\alpha d}\right).
		\end{split}
	\end{equation}
	In addition, using that $\qq$ is an even function and that it is increasing on $[0,\infty)$, we infer that
	\begin{equation} \label{eqn626}
		\QQ(\tilde{g}) 
		= \alpha \sum_{i=1}^{d} \qq\Big( \frac{\tilde{g}_i}{\alpha}\Big)
		= \alpha \sum_{i=1}^{d} \qq\Big( \frac{\abs{\tilde{g}_i}}{\alpha}\Big)
		\le \alpha d \max_{ i \in [d]} \qq \Big( \frac{\abs{\tilde{g}_i}}{\alpha}\Big)
		= \alpha d \qq \Big( \frac{\norm{\tilde{g}}_{\ell^{\infty}}}{\alpha}\Big)
	\end{equation}
	Inserting \eqref{eqn625} and \eqref{eqn626} into \eqref{eqn624}, and dividing by $\alpha d$, we deduce that
	\begin{equation*}
		\qq\Big( \frac{\norm{\xinf}_{\ell^1}}{\alpha d}\Big) 
		\le \qq \Big( \frac{\norm{\tilde{g}}_{\ell^{\infty}}}{\alpha}\Big).
	\end{equation*}
	We complete the proof by using the monotonicity of $\qq$.
\end{proof}
With this lemma at hand, 
we can now prove the upper bound for $\norm{\tildenpara}_{\ell^1}$.
\begin{lemma}\label[lemma]{lemma:npara_bound_deep_non_unique}
Assume that the assumption of \cref{theorem:upper_bound_deep_non_unique} holds.
Then it holds that 
	\begin{equation*}
		\norm{\tildenpara}_{\ell^1} 
		\le
		C^{\sharp} \varepsilon \norm{\nperp_{\SC}}_{\ell^{1}},
	\end{equation*}
	where 
	\begin{align*}
		C^{\sharp}
		:=	
		5 \tilde{\varrho}
		\left(
		88 
        \left( \frac{\norm{\ga}_{\ell^{1}}}{\mg}\right)^{1+\gamma}
		+
		512 \tilde{\varrho}
		\abs{\SC}
		\left(
		\hh(\varrho)+1
		\right)
		\right)
		\left(
		\frac{2d \norm{\ga}_{\ell^\infty}}{ \min_{i\in \SP} \abs{\ga_i} }
		\right)^{1+\gamma}
		.
	\end{align*}
\end{lemma}

\begin{proof}
	Since $\xinf$ and $\xa$ are minimizers of the functional $D_{\QQ}(\cdot,0)$ 
	on the subsets $\Lscr$ and $\Lmin$, respectively, we obtain two first order optimality conditions.
	By optimality of $\xinf$ and $\xa$ we have
	\begin{equation} \label{eqn800}
		\begin{split}
		0 &= 
		\left.\frac{\dif}{\dif t}\right\rvert_{t=0} D_{\QQ}(\xinf+t \widetilde{\npara},0)
		= \langle \nabla \QQ(\xinf), \widetilde{\npara} \rangle,\\
		0 
		&= \left.\frac{\dif}{\dif t}\right\rvert_{t=0} D_{\QQ}(\xa+t \widetilde{\npara},0)
		= \langle \nabla \QQ(\xa), \widetilde{\npara}  \rangle.\\
		\end{split}
	\end{equation}
	Combining the first-order optimality conditions \eqref{eqn800} 
	we infer that
	\begin{equation} \label{eqn800a}
		\langle \nabla \QQ(\xinf) , \widetilde{\npara} \rangle 
		= \langle \nabla \QQ(\xa), \widetilde{\npara} \rangle.
	\end{equation}
	By recalling that $\xinf = \xa  +\nperp + \widetilde{\npara}$ 
	and by introducing the intermediate point $\xa + \nperp$,
	 we can rewrite \eqref{eqn800a} as
	\begin{equation} \label{eqn800b}
		\langle \nabla \QQ(\xa  +\nperp + \widetilde{\npara}) - \nabla\QQ(\xa + \nperp ) , 
		\widetilde{\npara}\rangle 
		= \langle \nabla \QQ(\xa) - \nabla \QQ(\xa+\nperp), \widetilde{\npara}\rangle.
	\end{equation}
	First, we derive a lower bound for the left-hand side of \eqref{eqn800b} 
	via a strong convexity argument. 
	Using $\npara_{\SC}=0$ at $(a)$ 
	and \cref{lemma:strong_convexity_deep} at $(b)$, 
	we infer that
	\begin{align} 
			\langle \nabla \QQ(\xa  +\nperp + \widetilde{\npara}) - \nabla\QQ(\xa + \nperp ) , \widetilde{\npara}\rangle 
			&\overeq{(a)} \langle \nabla \QQ(\xa_{\SP}+\nperp_{\SP} 
			+ \widetilde{\npara_{\SP}}) - \nabla \QQ(\xa_{\SP}+\nperp_{\SP}),
			 \widetilde{\npara_{\SP}}\rangle
				\nonumber\\
			&=  \langle \int_0^1 \left.\frac{\dif}{\dif s }\right\rvert_{s =t} \nabla\QQ(\xa_{\SP}+\nperp_{\SP} 
			+s \widetilde{\npara_{\SP}} ) \dif t , \widetilde{\npara_{\SP}}\rangle
				\nonumber\\
			&= \int_0^1 \langle \nabla^2\QQ(\xa_{\SP}+\nperp_{\SP} +t \widetilde{ \npara_{\SP}}) \widetilde{\npara_{\SP}}, 
			\widetilde{\npara_{\SP}}\rangle \dif t
				\nonumber\\
			&\overgeq{(b)} 
			\frac{ \norm{ \widetilde{ \npara} }_{\ell^1}^2\gamma \alpha^{\gamma}}
			{3 \abs{\SP} \alpha^{1+\gamma} + 2B_1^{1+\gamma}},
				\label{eqn1300a}
	\end{align}
	where
	\begin{equation*}
		B_1:= \max_{t\in [0,1]} 
		\norm{ \xa_{\SP}+\nperp_{\SP} +t \widetilde{\npara_{\SP}} }_{\ell^{1+\gamma}}.
	\end{equation*}
	Next, we derive an upper bound for the right-hand side of \eqref{eqn800b}. 
	Using that $ \widetilde{\npara_{\SC}}=0$ at $(a)$ 
	and that $\big(\hh^{-1}\big)'$ is even at $(b)$, 
	we infer that
	\begin{equation}\label{eqn1300}
		\begin{split}
			\langle \nabla \QQ(\xa) - \nabla \QQ(\xa+\nperp), \widetilde{\npara} \rangle
			&\overeq{(a)} \langle \nabla \QQ(\xa_{\SP}) - \nabla \QQ(\xa_{\SP}+\nperp_{\SP}), 
			\widetilde{\npara_{\SP}}\rangle\\
			&= \sum_{i\in \SP} \widetilde{\npara_i} 
			\Big[\hh^{-1}\Big( \frac{\xa_i}{\alpha}\Big) 
			- \hh^{-1}\Big( \frac{\xa_i+\nperp_i}{\alpha}\Big)\Big] \\
			&\overeq{(b)} \sum_{i\in \SP} 
			\widetilde{n^{\para,\ast}_i}
			\Big[\hh^{-1}\Big( \frac{\abs{\xa_i}}{\alpha}\Big) 
			- \hh^{-1}\Big( \frac{\abs{\xa_i}+n^{\perp,\ast}_i}{\alpha}\Big)\Big],
		\end{split}
	\end{equation}
	where $n^{\perp\ast}_i := \nperp_i \sign(\xa_i)$ 
	and $ \widetilde{n^{\para,\ast}_i}:= \widetilde{\npara_i} \sign(\xa_i)$ for $i\in \SP$.
	Note that in (b) we have also used that
	$ \norm{\nperp_\SP}_{\ell^\infty} \le \min_{i \in \SP} \abs{\ga_i}/4 $,
	see \cref{lemma:nperpbound_deep_non_unique},
	and that $\abs{\xa_i} \ge \frac{1}{2} \min_{i \in \SP} \abs{\ga_i}$,
	see \cref{lemma:two_minimizers_deep_2}.
	For $i\in \SP$ and $t\in \R$ define
	\begin{equation*}
		\phi_{i}(t):= \hh^{-1}\Big( \frac{\abs{\xa_i}+tn^{\perp,\ast}_i}{\alpha}\Big).
	\end{equation*}
	Because of 
	$\abs{\xa_i} \ge \frac{1}{2} \min_{i \in \SP} \abs{\ga_i}$
	and $\abs{\nperp_i} \le \frac{1}{4} \min_{i \in \SP} \abs{\ga_i}$, 
	see \cref{lemma:two_minimizers_deep_2}  
	and \cref{lemma:nperpbound_deep_non_unique},
	we have that $\abs{\xa_i}+tn^{\perp,\ast}_i>0$ for all $t\in (-1,1)$ 
	and so the map $\phi_{i}$ is differentiable on $(-1,1)$. 
	Hence, using the Taylor expansion of $\phi_i$ there exists $\xi_i\in (0,1)$ such that
	\begin{equation}\label{eqn1301}
		\begin{split}
			\hh^{-1}\Big( \frac{\abs{\xa_i}}{\alpha}\Big) 
			- \hh^{-1}\Big( \frac{\abs{\xa_i}+n^{\perp,\ast}_i}{\alpha}\Big)
			&= \phi_{i}(0)-\phi_{i}(1)
			=-\phi_{i}'(0) - \frac{1}{2} \phi_{i}''(\xi_i)\\
			&= -\big(\hh^{-1})'\Big( \frac{\abs{\xa_i}}{\alpha}\Big)\frac{n^{\perp,\ast}_i}{\alpha}
			- \frac{1}{2}\big(\hh^{-1})''\Big( \frac{\abs{\xa_i}+\xi_i n^{\perp,\ast}_i}{\alpha}\Big)\frac{(n^{\perp}_i)^2}{\alpha^2}.
		\end{split}
	\end{equation}
	Recall that by definition of $\Nor$ and of $\langle\cdot ,\cdot \rangle_{\ga}$
	we have
	\begin{equation} \label{eqn1302}
		\gamma  \alpha^{1+\gamma} 
		\sum_{i\in \SP} \frac{\nperp_i \widetilde{\npara_i}}{\abs{\ga_i}^{1+\gamma}} 
		= \gamma \alpha^{1+\gamma} \langle \widetilde{\npara},\nperp\rangle_{\ga} 
		=0.
	\end{equation}
	Inserting \eqref{eqn1301} into \eqref{eqn1300}, and using \eqref{eqn1302}, we deduce that
	\begin{align}
			&\langle \nabla \QQ(\xa) - \nabla \QQ(\xa+\nperp), \widetilde{\npara} \rangle \nonumber \\
			=& \sum_{i\in \SP} \frac{ \widetilde{\npara_i} \nperp_i}{\alpha} 
			\Big[ 
			\frac{ \gamma \alpha^{1+\gamma}}{\abs{\ga_i}^{1+\gamma}}
			- \big(\hh^{-1}\big)'\Big( \frac{\abs{\xa_i}}{\alpha}\Big) 
			\Big]
			- \frac{1}{2}  \sum_{i\in \SP}
			\frac{ \widetilde{n^{\para,\ast}_i} (\nperp_i)^2}{\alpha^2} 
			\big(\hh^{-1}\big)''\Big(\frac{\abs{\xa_i}+\xi_i n^{\perp,\ast}_i}{\alpha}\Big) \nonumber \\
			\le&  
			\frac{B_2 \norm{\widetilde{\npara_{\SP}}}_{\ell^1} \norm{\nperp_{\SP}}_{\ell^{\infty}}}{\alpha} 
			+
			 \frac{B_3\norm{ \widetilde{\npara_{\SP}}}_{\ell^1} \norm{\nperp_{\SP}}_{\ell^{\infty}}^2}
			 {\alpha^2}, \label{eqn1303}
	\end{align}
	where
	\begin{equation*}
		B_2 := \max_{i\in \SP} 
		\abs{ \frac{\gamma \alpha^{1+\gamma}}{\abs{\ga_i}^{1+\gamma}}
			- \big(\hh^{-1}\big)'\Big( \frac{\abs{\xa_i}}{\alpha}\Big) }
		\qquad \text{and} \qquad
		B_3 := \frac{1}{2} \max_{i\in \SP} \abs{  \big(\hh^{-1}\big)''\Big(\frac{\abs{\xa_i}+\xi_i n^{\perp,\ast}_i}{\alpha}\Big)}.
	\end{equation*}
	Inserting the lower bound \eqref{eqn1300a} and the upper bound \eqref{eqn1303} 
	into \cref{eqn800b},
	we obtain
	\begin{equation*}
		\frac{\norm{ \widetilde{ \npara }}_{\ell^1}^2 \gamma \alpha^{\gamma}}
		{3 \abs{\SP} \alpha^{1+\gamma} + 2B_1^{1+\gamma}}
		\le 
		\frac{ B_2 \norm{ \widetilde{ \npara_{\SP}}}_{\ell^1} \norm{\nperp_{\SP}}_{\ell^{\infty}}}{\alpha}  
		+\frac{ B_3 \norm{ \widetilde{\npara_{\SP}}}_{\ell^1} \norm{\nperp_{\SP}}_{\ell^{\infty}}^2}{\alpha^2} .
	\end{equation*}
	It follows that
\begin{equation} \label{eqn1307}
	\norm{ \widetilde{\npara} }_{\ell^1}
	\le \frac{3 \abs{\SP} \alpha^{1+\gamma} + 2B_1^{1+\gamma}}{ \gamma \alpha^{\gamma}}
	\left( \frac{B_2 \norm{\nperp_{\SP}}_{\ell^{\infty}}}{\alpha}
	+\frac{B_3 \norm{\nperp_{\SP}}_{\ell^{\infty}}^2}{\alpha^2}\right).
\end{equation}
In order to complete the proof, we need to bound $B_1$, $B_2$, and $B_3$
from above.	
We start by bounding $B_2$. 
We have for all $i\in \SP$ that
	\begin{equation}\label{eqn1308}
		\abs{ \frac{\gamma \alpha^{1+\gamma}}{\abs{\ga_i}^{1+\gamma}}
			- \big(\hh^{-1}\big)'\Big( \frac{\abs{\xa_i}}{\alpha}\Big) }
		\le \abs{ \frac{\gamma \alpha^{1+\gamma}}{\abs{\ga_i}^{1+\gamma}}
			- \frac{\gamma \alpha^{1+\gamma}}{\abs{\xa_i}^{1+\gamma}} }
		+ \abs{ \frac{\gamma \alpha^{1+\gamma}}{\abs{\xa_i}^{1+\gamma}}
			- \big(\hh^{-1}\big)'\Big( \frac{\abs{\xa_i}}{\alpha}\Big) }.
	\end{equation}
	By \cref{lemma:basic_inequalities_deep_non-unique},
	see \cref{eqn:bounds_on_h_D_5},
	and since $\min_{i \in \SP} \abs{\xa_i} \ge \min_{i \in \SP} \abs{\ga_i}/2 $,
	which follows from $ \norm{\xa- \ga}_{\ell^1} \le \min_{i \in \SP} \abs{\ga_i}/2 $,
	see \cref{lemma:two_minimizers_deep_2},
	we have
	\begin{align}
		\abs{ \frac{\gamma \alpha^{1+\gamma}}{\abs{\xa_i}^{1+\gamma}}
			- \big(\hh^{-1}\big)'\Big( \frac{\abs{\xa_i}}{\alpha}\Big) } 
		\le& 5 \gamma \Big(\frac{\abs{\xa_i}}{\alpha}\Big)^{-2-\gamma} 
		\le  5 \gamma \left( \frac{\alpha}{ \min_{i \in \SP} \abs{\xa_i} } \right)^{2+\gamma}
		\le 5 \cdot 2^{2+\gamma}  
		\gamma \left( \frac{\alpha}{ \min_{i \in \SP} \abs{\ga_i} } \right)^{2+\gamma} \nonumber \\
		=& 5 \cdot 2^{2+\gamma} \gamma \varepsilon^{2+\gamma}.
		\label{eqn1318}
	\end{align}
	This bounds the first term on the right-hand side of \cref{eqn1308}.
	To bound the second term on the right-hand side of \cref{eqn1308},
	we define $\psi(t):= (1+t)^{-1-\gamma}$ for $t\in (-1,1)$. 
	We obtain that
	\begin{align}
		\abs{ \frac{\gamma \alpha^{1+\gamma}}{\abs{\ga_i}^{1+\gamma}}
			- \frac{\gamma \alpha^{1+\gamma}}{\abs{\xa_i}^{1+\gamma}} }
		=&  \frac{\gamma \alpha^{1+\gamma}}{\abs{\ga_i}^{1+\gamma}} 
		\abs{ 1 - \frac{\abs{\ga_i}^{1+\gamma}}{\abs{\xa_i}^{1+\gamma} }  }
		= \frac{\gamma \alpha^{1+\gamma}}{ \abs{\ga_i}^{1+\gamma}}
		\abs{ \psi (0) - \psi \left( \frac{\abs{\xa_i}}{\abs{\ga_i}} -1 \right)  } \nonumber \\
		=&  \frac{\gamma \alpha^{1+\gamma}}{\abs{\ga_i}^{1+\gamma}} 
		\abs{ \psi(0) - \psi\Big( \frac{\abs{\xa_i}-\abs{\ga_i}}{\abs{\ga_i}}\Big)  }\nonumber \\
		\overeq{(a)}& \frac{\gamma  \alpha^{1+\gamma}}{\abs{\ga_i}^{1+\gamma}} 
		\abs{\psi'(\xi) \cdot \frac{\abs{\xa_i}-\abs{\ga_i}}{\abs{\ga_i}}} \nonumber \\
		\le &
		\frac{\gamma \alpha^{1+\gamma}}{ \min_{i \in \SP} \abs{\ga_i}^{2+\gamma}} 
		\underset{\le 2^{\gamma+2} (1+\gamma)}{\underbrace{\left( \max_{ \xi \in [-1/2,1/2] } \abs{ \psi'(\xi)} \right)}}
		\max_{i\in \SP} \abs{\xa_i - \ga_i} \nonumber \\
		\le &
		\frac{2^{\gamma+2} \gamma (1+\gamma) \alpha^{1+\gamma}}{ \min_{i \in \SP} \abs{\ga_i}^{2+\gamma}} 
		\norm{\xa - \ga}_{\ell^1} \nonumber \\
		\overleq{(b)} &
		\frac{2^{\gamma+2}\gamma (1+\gamma) \alpha^{2+\gamma}}{ \min_{i \in \SP} \abs{\ga_i}^{2+\gamma}} 
		 \Big( \frac{\norm{\ga}_{\ell^{1}}}{\mg}\Big)^{1+\gamma} 
		(1+10 \eps)\nonumber \\
		= &
		2^{\gamma+2}
		\gamma (1+\gamma)
		\varepsilon^{2+\gamma}
		 \Big( \frac{\norm{\ga}_{\ell^{1}}}{\mg}\Big)^{1+\gamma} 
		(1+10 \eps).
		\label{eqn1312}
	\end{align}
	Equation (a) follows from the Taylor expansion of $\psi$ at $0$.
	Note that we have $\abs{\xi} \le 1 /2 $
	due to 
	$\max_{i\in \SP} \abs{\xa_i - \ga_i} \le \norm{\xa -\ga}_{\ell^1} \le 1/2\min_{i \in \SP} \abs{\ga_i} $,
	see \cref{lemma:two_minimizers_deep_2}. 
	Equation (b) follows again from \cref{lemma:two_minimizers_deep_2}.
	By combining inequalities \eqref{eqn1318} and \eqref{eqn1312} 
	we then obtain that
	\begin{align}
		B_2
		&\le 
		\gamma 2^{2+\gamma}
		\varepsilon^{2+\gamma}
		\left(
		5   
		+ (1+\gamma) \left( \frac{\norm{\ga}_{\ell^{1}}}{\mg}\right)^{1+\gamma} 
		(1+10 \eps)
		\right) \nonumber \\
		&\overleq{(a)}
		8 \gamma
		\varepsilon^{2+\gamma}
		\left(
		5   
		+ 2 \left( \frac{\norm{\ga}_{\ell^{1}}}{\mg}\right)^{1+\gamma} 
		(1+10 \eps)
		\right)\nonumber\\
		&\overleq{(b)}
		8 \gamma
		\varepsilon^{2+\gamma}
		\left(
		5   
		+ 6 \left( \frac{\norm{\ga}_{\ell^{1}}}{\mg}\right)^{1+\gamma} 
		\right) \nonumber \\
		&\le 
		88 \gamma
		\varepsilon^{2+\gamma}
        \left( \frac{\norm{\ga}_{\ell^{1}}}{\mg}\right)^{1+\gamma},
		\label{eqn1310}
	\end{align}
	where in inequality (a) we used that $ \gamma = \frac{D-2}{D} \le 1 $.
	In inequality (b) we used the assumption that $\eps \le 1/8$.

	In the next step, we will derive an upper bound for $B_3$.
	First we note that for all $i\in \SP$
	\begin{align}
		\frac{\alpha}{\abs{\xa_i} + \xi_i n^{\perp,\ast}_i} 
		\overleq{(a)}
		&\frac{\alpha}{\abs{\xa_i} \left(1 - \frac{\abs{\nperp_i}}{\abs{\xa_i}}\right)}
		\nonumber \\
		\le  
		&\frac{\alpha}{\min_{i \in \SP} \abs{\xa_i}}
		\cdot
		\frac{1}{ \min_{i \in \SP} \left( 1 - \frac{\abs{\nperp_i}}{\abs{\xa_i}} \right)}
		\nonumber \\
		\overleq{(b)} 
		&\frac{4 \alpha}{ \min_{i \in \SP} \abs{\ga_i}}
        =
		4\varepsilon.
		\label{eqn:bounds_on_h_D_42}
	\end{align}
	In inequality $(a)$ we used that $\abs{\xi_i} \le 1$.
	In inequality $(b)$ we used that $\abs{\xa_i} \ge \abs{\ga_i}/2$ for $i\in \SP$
	and $ \abs{\nperp_i} \le \min_{i \in \SP} \abs{\ga_i}/4 \le \abs{\xa_i}/2 $,
	see \cref{lemma:nperpbound_deep_non_unique}. 
    
	Note that inequality \eqref{eqn:bounds_on_h_D_42} also 
	implies that $\frac{\abs{\xa_i} + \xi_i n^{\perp,\ast}_i}{\alpha}  \ge 1$ 
	since $\eps \le 1/4$
	due to Assumption \eqref{eqn1500}. 
	Then we can apply \cref{lemma:basic_inequalities_deep_non-unique},
	see \cref{eqn:bounds_on_h_D_4},
	in inequality (a) and obtain that
	\begin{align} 
		B_3 
		= 
		&\frac{1}{2}
		\max_{i\in \SP} 
		\abs{  \big(\hh^{-1}\big)''\Big(\frac{\abs{\xa_i}+\xi_i n^{\perp,\ast}_i}{\alpha}\Big)}
		\nonumber \\
		\overleq{(a)} 		
		& 8  \gamma 
		\left( 
			 \frac{\alpha}
			 { \min_{i \in \SP} \left( \abs{\xa_i}+\xi_i n^{\perp,\ast}_i \right)}
		\right)^{2+\gamma} \nonumber \\
		\overleq{(b)} 
		& 8 \cdot 4^{2+\gamma} \gamma
		\left( 
			 \frac{\alpha}
			 { \min_{i \in \SP}  \abs{\ga_i} }
		\right)^{2+\gamma} \nonumber \\
		=
		& 8 \cdot 4^{2+\gamma} \gamma \eps^{2+\gamma}
		\nonumber \\
		\overleq{(c)}
		& 512 \gamma \varepsilon^{2+\gamma}.
		\label{eqn1309}
	\end{align}
	In inequality (b) we have used \cref{eqn:bounds_on_h_D_42}.
    In inequality (c) we used that $\gamma \le 1$.

	It remains to bound $B_1$ from above. 	
	We compute that 
	\begin{align}
			B_1 
			= 
			&\max_{t\in [0,1]} \norm{ \xa_{\SP}+\nperp_{\SP} 
			+t \widetilde{\npara_{\SP}}}_{\ell^{1+\gamma}} \nonumber \\
			\overleq{(a)}	
			&\max_{t\in [0,1]} \norm{ \xa_{\SP}+\nperp_{\SP} 
			+t \widetilde{\npara_{\SP}}}_{\ell^{1}} \nonumber \\
			\overeq{(b)}
			&\max_{t\in [0,1]} \norm{ 
			t \xinf_{\SP}+ (1-t)\xa_{\SP} + (1-t) \nperp_{\SP} 
			}_{\ell^{1}} \nonumber \\
			\le
			&\max_{t\in [0,1]}
			\big[ 
			t \norm{\xinf_{\SP}}_{\ell^{1}} + (1-t) \norm{\xa_{\SP}}_{\ell^{1}}
			+ (1-t) \norm{\nperp_{\SP}}_{\ell^{1}}	
			\big] \nonumber \\
			\overleq{(c)}
			&\max_{t\in [0,1]}
			\big[ 
			t d \norm{ \ga }_{\ell^{\infty}} + (1-t) \norm{ \ga }_{\ell^{1}}
			+ (1-t) \norm{\nperp_{\SP}}_{\ell^{1}}	
			\big] \nonumber \\
			\le &d\norm{\ga}_{\ell^{\infty}} + \norm{\nperp_{\SP}}_{\ell^1} \nonumber \\
			\overleq{(d)} &d\norm{\ga}_{\ell^{\infty}} +  \min_{i \in \SP} \abs{\ga_i}/4  \nonumber \\
			\le
			& 2d \norm{\ga}_{\ell^\infty}.
			\label{eqn1308_d}
	\end{align}
	In inequality (a) we used that 
	$ \norm{\cdot}_{\ell^{1+\gamma}} \le \norm{\cdot}_{\ell^1} $.
	Equation (b) is due to $ \xinf = \xa + \nperp + \widetilde{\npara} $.
	For inequality (c) we used \cref{lemma:a_priori_bounds_deep}
	and for inequality (d) we used \cref{lemma:nperpbound_deep_non_unique}.
	We obtain that
	\begin{align} 
		\norm{ \widetilde{ \npara }}_{\ell^1}
		\overleq{(a)} &
		\frac{3 \abs{\SP} \alpha^{1+\gamma} + 2B_1^{1+\gamma}}{ \gamma \alpha^{1+\gamma}}
		\left( B_2
		+\frac{B_3 \norm{\nperp_{\SP}}_{\ell^{\infty}}}{\alpha}\right) 
		\norm{\nperp_{\SP}}_{\ell^{\infty}}
		\nonumber \\
		\overleq{(b)} &
		\frac{
            \left(3 \abs{\SP} \alpha^{1+\gamma} + 2^{2+\gamma}d^{1+\gamma} \norm{\ga}_{\ell^{\infty}}^{1+\gamma}\right)
		    \varepsilon^{2+\gamma}}
		{ \gamma \alpha^{1+\gamma} }
		\left(  88 \gamma
        \left( \frac{\norm{\ga}_{\ell^{1}}}{\mg}\right)^{1+\gamma}
		+\frac{512 \gamma \norm{\nperp_{\SP}}_{\ell^{\infty}}}{\alpha}\right)
		\norm{\nperp_{\SP}}_{\ell^{\infty}}
		\nonumber \\
		=&
		\varepsilon
		\left(
		3 \abs{\SP} \varepsilon^{1+\gamma}
		+2
		\left(\frac{ 2d \norm{\ga}_{\ell^{\infty}}}
		{\min_{i\in \SP} \abs{\ga_i}}\right)^{1+\gamma}
		\right)
		\left(  88 
        \left( \frac{\norm{\ga}_{\ell^{1}}}{\mg}\right)^{1+\gamma}
		+\frac{512  \norm{\nperp_{\SP}}_{\ell^{\infty}}}{\alpha}\right)
		\norm{\nperp_{\SP}}_{\ell^{\infty}} \nonumber \\
		\overleq{(c)} &
		5 \varepsilon
		\left(\frac{ 2d \norm{\ga}_{\ell^{\infty}}}
		{\min_{i\in \SP} \abs{\ga_i}}\right)^{1+\gamma}
		\left(  88 
        \left( \frac{\norm{\ga}_{\ell^{1}}}{\mg}\right)^{1+\gamma}
		+\frac{512  \norm{\nperp_{\SP}}_{\ell^{\infty}}}{\alpha}\right)
		\norm{\nperp_{\SP}}_{\ell^{\infty}},
       \label{eqn1311}
	\end{align}
    where in inequality (a) we used \eqref{eqn1307},
	in inequality (b) we used \eqref{eqn1310}, \eqref{eqn1309}, and \eqref{eqn1308_d},
	and inequality (c) follows from 
	\begin{align*}
		\abs{\SP}
		\le 
		\frac{d \norm{\ga}_{\ell^\infty}}{ \min_{i\in \SP} \abs{\ga_i} }
		\le
		\left(
		\frac{2d \norm{\ga}_{\ell^\infty}}{ \min_{i\in \SP} \abs{\ga_i} }
		\right)^{1+\gamma}.
	\end{align*}
    In order to proceed, we note that
	\begin{align*}
		\frac{512  \norm{\nperp_{\SP}}_{\ell^{\infty}}}{\alpha}
		\le 
		&\frac{512 \tilde{\varrho}  \norm{\nperp_{\SC}}_{\ell^{1}}}{\alpha}\\
		\le 
		&512 \tilde{\varrho}
		\abs{\SC}
		\Big[ 
		\hh(\varrho) 
		+\frac{4\cdot 2^\gamma \varrho^{-}}{\gamma (1-\varrho)^{\frac{1}{\gamma}+1}} 
		 \left( \frac{\alpha}{\min_{i \in \SP} \abs{ \ga_i }} \right)^{\gamma}
		\Big]\\
		\overleq{(a)} 
		&
		512 \tilde{\varrho}
		\abs{\SC}
		\left(
		\hh(\varrho)+1
		\right),
	\end{align*}
where inequality $(a)$ follows from Assumption \eqref{eqn1500}.
By inserting this estimate into \cref{eqn1311} we obtain that
\begin{align*}
	\norm{ \widetilde{ \npara }}_{\ell^1}
	\le 
	\frac{ C^{\sharp} \varepsilon \norm{\nperp_{\SP}}_{\ell^{\infty}}}
	{ \tilde{\varrho}}
	\le 
	C^{\sharp} \varepsilon \norm{\nperp_{\SC}}_{\ell^{1}},
\end{align*}
where
\begin{equation*}
    C^{\sharp} := 5\tilde{\varrho} 
    \cdot \Bigg(
        88 
        \left( \frac{\norm{\ga}_{\ell^{1}}}{\mg}\right)^{1+\gamma}
		+
		512 \tilde{\varrho}
		\abs{\SC}
		\left(
		\hh(\varrho)+1
		\right)
\Bigg)
\cdot 
		\left(
		\frac{2d \norm{\ga}_{\ell^\infty}}{ \min_{i\in \SP} \abs{\ga_i} }
		\right)^{1+\gamma}.
\end{equation*}
This completes the proof.
\end{proof}

\noindent \textbf{Step 4 (Combining the bounds):}
After having proven upper bounds for 
$\norm{\xa - \ga}_{\ell^1}$,
$\norm{\nperp_{\SC}}_{\ell^1}$, and $\norm{\widetilde{\npara}}_{\ell^1}$, 
we combine these bounds to obtain \cref{theorem:upper_bound_deep_non_unique}.

\begin{proof}[Proof of \cref{theorem:upper_bound_deep_non_unique}]
Recall \cref{ineq:Dominiklast11} which implies that 
\begin{align*}
	\norm{\xinf-\ga}_{\ell^1} 
	\le & 
	\norm{\nperp}_{\ell^1} + \norm{\widetilde{\npara}}_{\ell^1} + \norm{\xa-\ga}_{\ell^1}\\
	\overleq{(a)} & 
	(1+\tilde{\varrho} + C^{\sharp} \varepsilon) \norm{\nperp_{\SC}}_{\ell^1}  
	+ \norm{\xa-\ga}_{\ell^1}\\
	\overleq{(b)} & 
	(1+\tilde{\varrho} + C^{\sharp} \varepsilon) \norm{\nperp_{\SC}}_{\ell^1}  
	+
	\alpha  \Big( \frac{\norm{\ga}_{\ell^{1}}}{\mg}\Big)^{1+\gamma} 
		(1+10 \eps)
	\\
	\overleq{(c)} & 
	(1+\tilde{\varrho} + C^{\sharp} \varepsilon) 
	\alpha \abs{\SC}
		\left( 
		\hh(\varrho) 
		+\frac{4\cdot 2^\gamma \varrho^{-}}{\gamma (1-\varrho)^{\frac{1}{\gamma}+1}} 
		 \left( \frac{\alpha}{\min_{i \in \SP} \abs{ \ga_i }} \right)^{\gamma}
		\right)	\\
		&+\alpha  \Big( \frac{\norm{\ga}_{\ell^{1}}}{\mg}\Big)^{1+\gamma} 
		(1+10 \eps).
\end{align*}
Inequality (a) is due to 
$ \norm{\nperp}_{\ell^1} \le (1+\tilde{\varrho}) \norm{ \nperp_{\SC}}_{\ell^1}  $,
which is due to the definition of $\tilde{\varrho}$,
and from \cref{lemma:npara_bound_deep_non_unique}.
In inequality (b) we used \cref{lemma:two_minimizers_deep_2}
and in inequality (c) we used \cref{lemma:nperpbound_deep_non_unique}.
By rearranging terms we obtain that
\begin{align*}
	\frac{\norm{\xinf-\ga}_{\ell^1}}{\alpha}
	\le 
	(1+\tilde{\varrho}) \abs{\SC} \hh(\varrho)
	+ \left(\frac{\norm{\ga}_{\ell^{1}}}{\min_{i\in \SP}\abs{\ga_i}}\right)^{1+\gamma} 
	+
	g(\varepsilon),
\end{align*}
where the function $g$ is defined as 
\begin{align*}
	g(\varepsilon)
	:=&
	C^{\sharp} \varepsilon \abs{\SC}
	\left( 
		\hh(\varrho) 
		+\frac{4\cdot 2^\gamma \varrho^{-}\varepsilon^\gamma}
		{\gamma (1-\varrho)^{\frac{1}{\gamma}+1}} 
	\right)
	+
	10 \varepsilon
	\left(\frac{\norm{\ga}_{\ell^{1}}}{\min_{i\in \SP}\abs{\ga_i}}\right)^{1+\gamma}.
\end{align*}
This completes the proof of \cref{theorem:upper_bound_deep_non_unique}.
\end{proof}

%% file: parts/technical_preliminaries_non-unique.tex
\subsection{Lemmas regarding the solution space and the null space property constants}
\subsubsection{Proof of \cref{basic_properties_Lmin}} \label{sec:proof_of_basic_properties_Lmin}
\begin{proof}[Proof of \cref{basic_properties_Lmin}]
	By \cref{assumption_on_A_y_non-unique} the set $\Lscr$ is non-empty. Since $\Lscr$ is a finite-dimensional affine space, and the map $\norm{\cdot}_{\ell^1}$ is coercive and continuous, we deduce the existence of a minimizer. Hence $\Lmin \ne \emptyset$. Since the set $\Lscr$ and the map $\norm{\cdot}_{\ell^1}$ are convex, so is $\Lmin$. Since $\norm{\cdot}_{\ell^1}$ is continuous and $\Lscr$ closed, so is $\Lmin$. By definition, $\Lmin$ is bounded. Hence it is compact.
	Since $y\ne 0$, we have $A0 \ne y$ and so $0 \notin \Lmin$.
	
	Assume that no such $\sigma$ exists and let $c:= \min_{x\in \Lscr} \norm{x}_{\ell^1}$. Then there exist $x,x' \in \Lmin$ and $i\in [d]$ such that $x_ix_i' <0$. Hence $\abs{x_i-x_i'}< \abs{x_i}+\abs{x_i'}$. Since $\Lmin$ is convex, we have $\frac{x+x'}{2}\in \Lmin$. Hence
	\begin{equation*}
		2c
		= 2 \norm{x+x'}_{\ell^1} 
		= \sum_{j=1}^d \abs{x_j+x_j'} 
		< \sum_{j=1}^d \abs{x_j}+\abs{x_j'} 
		= 2c,
	\end{equation*}
	a contradiction.
\end{proof}

\subsubsection{Proof of \cref{lemma:tangent_cone_characterization}}\label{section:prooftangent_cone_characterization}

\begin{proof}[Proof of \cref{lemma:tangent_cone_characterization}]
	Recall that our goal is to prove that
	\begin{equation*}
		\Tan = \Big\{ n \in \ker(A) : \sum_{i \in \SP}\s n_i =0, \text{ and } n_{\SC} =0 \Big\}.
	\end{equation*}
	Denote by $\widetilde{\Tan}$ the right-hand side of the above equation. 	
	Let $x,x' \in \Lmin$ with $x\ne x'$.
	By definition of $\SP$, we have $x_i=x'_i=0$ for all $i\in \SC$.
	Hence 
	\begin{equation*}
		(x-x')_{\SC}=0.
	\end{equation*}
	Since $\Lmin$ is convex, $x+t(x'-x)\in \Lmin$ for all $t\in (0,1)$.
    Hence, by definition of $\Lmin$ we have that 
	\begin{equation*}
		\frac{\norm{x+t(x'-x)}_{\ell^1} - \norm{x}_{\ell^1}}{t}  
		= 0
	\end{equation*}
	for $t \in (0,1)$.
	Therefore, for $t>0$ sufficiently small, we have
	\begin{align*}
		0 
		&= 
		\sum_{i\in \supp(x)} \sign(x_i)(x'_i-x_i) + \sum_{i\in [d]\setminus \supp(x)} \abs{x'_i-x_i}\\
		&=
		\sum_{i\in \supp(x)} \sign(x_i)(x'_i-x_i) 
		+ \sum_{i\in [d]\setminus \supp(x)} \abs{x'_i}.
	\end{align*}
	By \cref{basic_properties_Lmin} and since $\supp(x) \subset \SP$, we have $\sign(x_i)=\s$ for all $i\in \supp(x)$ and $\abs{x'_i} =\s x'_i$ for all $i\in [d]$. Hence
	\begin{align*}
		0 
		&= \sum_{i\in \supp(x)} \s (x'_i-x_i) + \sum_{i\in [d]\setminus \supp(x)} \s x'_i\\
		&= \sum_{i\in \supp(x)} \s (x'_i-x_i) + \sum_{i\in [d]\setminus \supp(x)} \abs{x'_i}\\
		&=\sum_{i\in \SP} \s (x'_i-x_i).
	\end{align*}
	Therefore, $x'-x \in \widetilde{\Tan}$. Since $\widetilde{\Tan}$ is a linear space, we deduce that $\Tan \subset \widetilde{\Tan}$.
	
	Conversely, let $n \in \widetilde{\Tan}$ and let $x\in \Lmin$ such that $\supp(x)=\SP$. 
	Such $x$ exists since all $x' \in \Lmin$ have the same sign pattern,
	see \cref{basic_properties_Lmin},
	and since $\Lmin$ is convex.
	Furthermore, \cref{basic_properties_Lmin} implies that $0\ne x$.
	We obtain for some sufficiently small $0<t$ that 
	\begin{equation*}
		\frac{\norm{x+t n}_{\ell^1} - \norm{x}_{\ell^1}}{t}  
		= \sum_{i\in \supp(x)} \sign(x_i)n_i +  \sum_{i\in [d]\setminus \supp(x)} \abs{n_i}=\sum_{i\in \SP} \s n_i=0,
	\end{equation*}
	where in the last equality we used \cref{basic_properties_Lmin}.
	Therefore, $x+tn\in \Lmin$ and so $n = \frac{1}{t} \big((x+tn)-x\big) \in \Tan$.
\end{proof}

\subsubsection{Proof of \cref{lemma:normal_coneSimple} and \cref{lemma:normal_cone}}
\label{section:proof_of_normal_conelemma}
We note that \cref{lemma:normal_coneSimple} is a special case of \cref{lemma:normal_cone}
since, if the minimizer is unique, we have $ \Nor= \ker A $.
Thus, in the following we only prove \cref{lemma:normal_cone}.
For the proof of \cref{lemma:normal_cone} we need the following technical lemma.
\begin{lemma} \label[lemma]{lemma:tangent_cone}
	Let $d,A,y$ as in \cref{assumption_on_A_y_non-unique}.
	\begin{enumerate}[a)]
		\item \label{lemma:tangent_cone_item1}For every $m\in \ker(A)$ and $x \in \Lmin$ we have
		\begin{equation} \label{eqn101}
			- \sum_{i\in \supp(x)} \sign(x_i)m_i \le \sum_{i\notin \supp(x)} \abs{ m_i}.
		\end{equation}
		
		\item \label{lemma:tangent_cone_item2}If $m \in \ker(A)$ satisfies $m_{\SC}=0$, then $m\in \Tan$.
	\end{enumerate}
\end{lemma}
We believe 
that the proof of this lemma
might be well-known to experts in the field.
However, since we could not find a reference,
we provide a proof for the sake of completeness.
\begin{proof}
	\emph{Proof of part a)} Let $m\in \ker(A)$. For every $t>0$, we have $x+tm \in \Lscr$. By the minimality of 
	$\norm{x}_{\ell^1}$ it follows that 
	\begin{equation*}
		0\le \frac{ \norm{x+tm}_{\ell^1} - \norm{x}_{\ell^1}}{t}.
	\end{equation*}
	Thus, for sufficiently small $t>0$, we have that
	\begin{equation*}
		0 \le \sum_{i\in \supp(x)} \sign(x_i)m_i +  \sum_{i\notin \supp(x)} \abs{m_i}.
	\end{equation*} 
	From this we infer \eqref{eqn101}.
	
	\emph{Proof of part b)} Let $x\in \Lmin$ with $\supp(x)=\SP$. 
	Such $x$ exists due to the convexity of $\Lmin$
	and due to the fact that all $x' \in \Lmin$
	have the same sign pattern. 
	Applying \eqref{eqn101} to both $m$ and $-m$, we obtain
	\begin{equation*}
		-\sum_{i\in \supp(x)} \sign(x_i)m_i \le 0
		\quad \text{and} \quad
		\sum_{i\in \supp(x)} \sign(x_i)m_i \le 0.
	\end{equation*}
	Using this and \cref{basic_properties_Lmin}, we infer that 
	\begin{equation*}
		0= \sum_{i\in \supp(x)} \sign(x_i) m_i
		=\sum_{i\in \SP} \s m_i.
	\end{equation*}
	Hence $m \in \Tan$ by \cref{lemma:tangent_cone_characterization}.
\end{proof}
With this lemma at hand, we can prove \cref{lemma:normal_cone}.
\begin{proof}[Proof of \cref{lemma:normal_cone}]
	Let $n\in \Nor$ with $n_{\SC}=0$. Then \cref{lemma:tangent_cone} implies that $n \in \Tan$. Hence $n\in \Tan \cap \Nor$ and it follows from \eqref{direct_sum} that $n=0$.
	
	Now assume that $\Nor \ne \{0\}$. By assumption, $\Nor\setminus \{0\} \ne \emptyset$ and so the suprema \eqref{eqn102} exist in $(-\infty,\infty]$.
	
	Let $\Nor_1 := \Nor \cap \partial B_1(0)$ and for $m \in \Nor\setminus\{0\}$ let 
	\begin{equation*}
		\varrho(m):= \frac{1}{\norm{n_{\SC}}_{\ell^1}}
		\cdot \Big( -\sum_{i\in \SP} \s n_i\Big).
	\end{equation*}
	Since $\varrho(tm)= \varrho(m)$ for all $t>0$ and $m\in \Nor\setminus\{0\}$, we have
	\begin{equation*}
		\sup_{n \in \Nor\setminus\{0\}} \varrho(m)
		= \sup_{n \in \Nor_1} \varrho(m).
	\end{equation*}
	Since $\varrho(\cdot)$ is continuous and $\Nor_1$ is compact, the supremum is attained. 
	
	Let $n\in \Nor\setminus\{0\}$ be such that $\varrho = \varrho(n)$. By \eqref{eqn101} of \cref{lemma:tangent_cone}, we have $\varrho(n)\le 1$. Since $\Nor$ is a linear space, we also have $-n \in \Nor$. Since $\varrho(-n) = - \varrho(n)$, it follows that $\varrho \ge \abs{\varrho(n)}\ge 0$.
	
	Assume for the sake of contradiction that $\varrho(n)=1$.
	Let $\xa\in \Lmin$ with full support $S( \xa )=\SP$.
	Then, for sufficiently small $\eps>0$, we have
	\begin{equation*}
		\norm{\xa+\eps n}_{\ell^1} 
	 		= \norm{\xa}_{\ell^1} 
			+ \eps \sum_{i\in \SP} \s n_i		
			+ \eps \sum_{i\in \SC} \abs{n_i}		
		= \norm{\xa}_{\ell^1},
	\end{equation*}
	where the last equation follows from $\varrho (n) =1$.
	Hence $\xa+\eps n\in \Lmin$. 
	Since, by assumption, $S(\xa) = \SP$, 
	it follows that $S(\xa+\eps n) \subset S(\xa) = \SP $ and so
	$n_{\SC}=0$.
	Then part b) of \cref{lemma:tangent_cone} implies that $n\in \Tan$. 
	We infer from $ \Tan \cap \Nor = \{0\}$, 
	see \cref{direct_sum}, that $n=0$, a contradiction.
	
	The claims for $\tilde{\varrho}$ and $\varrho^{-}$ are deduced analogously.
\end{proof}

\subsection{Lemmas regarding the solution space in the case $D=2$}

\subsubsection{Proof of \cref{lemma:maximal_support}}\label{section:proof_of_maximal_support_D_equal_2}

In order to prove \cref{lemma:maximal_support},
we will first establish the following technical lemma.
\begin{lemma} \label[lemma]{lemma:maximal_support_shallow_convex}
	Let the function $E$ be as defined in \cref{equ:Edefinition}.
	Let $C\subset\R^d_{\ge0}$ be a non-empty convex and compact subset.
	Let
	\begin{equation*}
		x \in \argmin_{z \in C} E(z).
	\end{equation*}
	Then  
	for all $n\in\R^d$ for which there exists $\lambda>0$ 
	such that $x+\lambda n\in C$ 
	it holds that
	$n_{i}=0$ for all $i\notin \supp(x)$.
	In particular, $ \supp(C) = \supp (x)$.
\end{lemma}
\begin{proof}
	Let $S:=\supp(x)$ and $S^c:= [d]\setminus \supp(x)$.
	Assume for the sake of contradiction, 
	that there exist $m\in \R^d$ and $\lambda>0$ such $x+\lambda m\in C$ and $m_{S^c}\ne 0$.
	Since $C$ is convex, 
	we have that $x +tm \in C$ for all $0<t\le \lambda$. 
	By minimality, we deduce that 
	\begin{equation*}
		E(x_S +t m_S) + E(t m_{S^c}) = E(x+tm)
		\ge E(x) = E(x_S)
	\end{equation*}
	Separating the entropy into its components on $S$ and $S^c$, and dividing by $t$, we obtain
	\begin{equation} \label{eqn1205}
		\frac{1}{t} \big(E(x_S +t m_S) -E(x_S)\big) 
		\ge - \frac{1}{t}E(t m_{S^c}).
	\end{equation}
	Since the map $t\mapsto E(x_S +t m_S)$ is differentiable at $t=0$, the left-hand side of \eqref{eqn1205} converges to a finite number as $t\downarrow0$. For the right-hand side, we compute
	\begin{equation*}
		\liminf_{t\downarrow0} \Big[-\frac{1}{t} E(t m_{S^c})\Big]
		= \liminf_{t\downarrow0} \Big[- \frac{1}{t} \sum_{i\in S^c} tm_i \log(tm_i)- tm_i\Big]
		= \sum_{i\in S^c} m_i - \Big[\limsup_{t \downarrow 0} \sum_{i\in S^c} m_i \log(tm_i)\Big].
	\end{equation*}
	Since $m_i \ge0$ for all $i \in S^c$ and there exists by assumption $j\in S^c$ such that $m_j>0$, we have
	\begin{equation*}
		-\Big[\limsup_{t \downarrow 0} \sum_{i\in S^c} m_i \log(tm_i)\Big] = \infty.
	\end{equation*}
	This contradicts \eqref{eqn1204} for sufficiently small $t>0$.
	
	Now let $z\in C$ be arbitrary and let $n:= z-x$. 
	Then $x+n \in C$ and so $z_{S^c} = n_{S^c} =0$. 
	Therefore, $\supp(z) \subset \supp(x)$ and so $\supp(C) \subset \supp(x)$. 
	Since it also holds that $x\in C$ we obtain equality.
\end{proof}
From \cref{lemma:maximal_support_shallow_convex} 
we can now immediately deduce \cref{lemma:maximal_support}.

\begin{proof}[Proof of \cref{lemma:maximal_support}]
	By \cref{basic_properties_Lmin}, the set $\Lmin$ is non-empty, convex and compact. 
	Furthermore, $E(\abs{x}) = E(\sigma x)$ for all $x\in\Lmin$ and $\sigma$ as in \cref{basic_properties_Lmin}. 
	Therefore, replacing $\Lmin$ by $\sigma \Lmin$ 
	we may assume 
	without loss of generality
	that $\Lmin\subset \R^d_{\ge0}$. Then \cref{lemma:maximal_support} follows from \cref{lemma:maximal_support_shallow_convex}.
\end{proof}

\subsubsection{Proof of \cref{lemma:normal_cone_shallow}}\label{section:proof_of_lemma:normal_cone_shallow}

It remains to prove \cref{lemma:normal_cone_shallow},
which states that 
$\Tan + \Nor = \ker(A)$
and 
$\Tan \cap \Nor = \{0\}$ 
holds.
\begin{proof}[Proof of \cref{lemma:normal_cone_shallow}]
	To show that $\ker(A) = \Tan +\Nor$, let $m\in \ker(A)$ be arbitrary. Since $\supp(\Tan) \subset \SP$ by \cref{lemma:tangent_cone_characterization}, we identify $\Tan$ with its restriction to $\R^{\SP}$.
	The restriction of the map $\langle \cdot , \cdot \rangle_{\ga}$ to $\R^{\SP}$ is a scalar product on $\R^{\SP}$. Since $m_{\SP} \in \R^{\SP}$, there exist $m_{\SP,\para} \in \R^{\SP}$ and $m_{\SP,\perp}\in \R^{\SP}$ such that 
	\begin{equation*}
		m_{\SP} = m_{\SP,\para} + m_{\SP,\perp},
		\qquad
		m_{\SP,\para}\in \Tan,
		\qquad\text{and}\qquad
		\langle n , m_{\SP,\perp}\rangle_{\ga} =0 \quad \text{for all } n\in \Tan
	\end{equation*}
	since $ \langle \rangle_{\ga} $ is a scalar product on $ \R^{\SP} $.
	Define 
	\begin{equation*}
		m_{\para} := m_{\SP,\para}
		\qquad\text{and}\qquad
		m^{\perp} := m_{\SP,\perp} + m_{\SC}.
	\end{equation*}
	It follows that
	\begin{equation} \label{eqn3000}
		m 
		= m_{\SP} + m_{\SC} 
		=  m_{\SP,\para} + m_{\SP,\perp} + m_{\SC} 
		= m_{\para} + m^{\perp}.
	\end{equation}
	Since $m \in \ker(A)$ and $m^{\para} \in \Tan \subset\ker(A)$, we have $m^{\perp}\in \ker(A)$. Furthermore, for all $n\in \Tan$ we have
	\begin{equation*}
		\langle m^{\perp},n\rangle_{\ga} 
		= \sum_{i\in \SP} \frac{n_i m^{\perp}_i}{\abs{\ga_i}}
		= \sum_{i\in \SP} \frac{n_i m_{\SP,\perp,i}}{\abs{\ga_i}}
		=0.
	\end{equation*}
	Therefore, $m^{\perp} \in \Ncal$. Now \eqref{eqn3000} implies that $\ker(A) = \Tan +\Nor$.
	
	It remains to show that $\Nor \cap \Tan =\{0\}$. 
	For that, let $n\in \Nor \cap \Tan$.
	Then it holds that
	\begin{equation*}
		0 =\langle n, n \rangle_{\ga} 
		= \sum_{i\in \SP} \frac{n_i^2}{\abs{\ga_i}}.
	\end{equation*}
	Hence $n_{\SP} =0$. 
	Furthermore, since $n\in \Tan$, we have $n_{\SC}= 0$ by \cref{lemma:tangent_cone_characterization}.
	This completes the proof.
\end{proof}

\subsection{Lemmas regarding the solution space in the case $D\ge 3$}

\subsubsection{Proof of \cref{lemma:maximal_support_deep}}\label{section:proof_of_lemma:maximal_support_deep}

\begin{lemma} \label[lemma]{lemma:maximal_support_deep_convex}
	Let $D \in \N$ with $D\ge 3$. 
	Let $C\subset\R^d_{\ge0}$ be a non-empty convex 
	and compact subset.
	Let
	\begin{equation*}
		x \in \argmax_{x \in C} \norm{x}_{\ell^{\frac{2}{D}}}.
	\end{equation*}
	Then $n_{i}=0$ for all $i\notin \supp(x)$ and all $n\in\R^d$ for which there exists $\lambda>0$ such that $x+\lambda n\in C$. In particular, $\supp(x) = \supp(C)$.
\end{lemma}
\begin{proof}
	Let $S:=\supp(x)$ and $S^c:= [d]\setminus \supp(x)$.
	Assume for the sake of contradiction, that there exist $m\in \R^d$ and $\lambda>0$ such $x+\lambda m\in C$ and $m_{S^c}\ne 0$.	Since $C$ is convex, we have that $x +tm \in C$ for all $0<t\le \lambda$. By maximality, we deduce that 
	\begin{equation*}
		\norm{x+tm}_{\ell^{\frac{2}{D}}}^{\frac{2}{D}} 
		\le \norm{x}_{\ell^{\frac{2}{D}}}^{\frac{2}{D}}.
	\end{equation*}
	Separating the sums into its components on $S$ and $S^c$, and dividing by $t$, we obtain
	\begin{equation} \label{eqn1204}
		\frac{1}{t} \big(\norm{x_{\SP}+tm_{\SP}}_{\ell^{\frac{2}{D}}}^{\frac{2}{D}} -\norm{x_{\SP}}_{\ell^{\frac{2}{D}}}^{\frac{2}{D}}\big) 
		\le - \frac{1}{t}\norm{tm_{S^c}}_{\ell^{\frac{2}{D}}}^{\frac{2}{D}}.
	\end{equation}
	Since the map $t\mapsto \norm{x_{\SP}+tm_{\SP}}_{\ell^{\frac{2}{D}}}^{\frac{2}{D}}$ is differentiable at $t=0$, the left-hand side of \eqref{eqn1204} converges to a finite number as $t\downarrow0$.
    For the right-hand side, we compute
	\begin{equation*}
		\limsup_{t\downarrow0} -\frac{1}{t} \norm{tm_{S^c}}_{\ell^{\frac{2}{D}}}^{\frac{2}{D}}
		= \limsup_{t\downarrow0} -t^{\frac{2}{D}-1} \sum_{i\in \SP^c} \abs{m_i}^{\frac{2}{D}}
		= -\infty.
	\end{equation*}
	
	Now let $z\in C$ be arbitrary and let $n:= z-x$. Then $x+n \in C$ and so $z_{S^c} = n_{S^c} =0$. Therefore, $\supp(z) \subset \supp(x)$ and so $\supp(C) \subset \supp(x)$. Since, $x\in C$, we obtain equality.
\end{proof}

\begin{proof}[Proof of \cref{lemma:maximal_support_deep}]
	By \cref{basic_properties_Lmin}, the set $\Lmin$ is non-empty, convex and compact.
    Furthermore, $\norm{x}_{\ell^{\frac{2}{D}}} = \norm{\sigma x}_{\ell^{\frac{2}{D}}}$ for all $x\in\R^d$ and $\sigma$ as in \cref{basic_properties_Lmin}.
    Therefore, replacing $\Lmin$ by $\sigma \Lmin$, if necessary, we may assume that $\Lmin\subset \R^d_{\ge0}$.

	Now, we show that the maximization problem \eqref{def:minimizer_deep} has a unique solution.
    Let $z',z \in \argmax_{x \in \Lmin} \norm{x}_{\ell^{\frac{2}{D}}}$ and assume for the sake of contradiction that $z'\ne z$.
    Let $c:= \max_{x\in \Lmin} \norm{x}_{\ell^{\frac{2}{D}}}$.
    By \cref{basic_properties_Lmin}, the set $\Lmin$ is convex and thus $\frac{z'+z}{2}\in \Lmin$.
    Furthermore, \cref{lemma:maximal_support_deep} implies that $\supp(z')=\supp(z)$.
    The map $\R^{\SP}_{>0} \ni \xi \mapsto \norm{\xi}_{\ell^{\frac{2}{D}}}^{\frac{2}{D}} = \sum_{i\in \SP} \xi_i^{2/D}$ is strictly concave. Since $z'_{\SP},z_{\SP} \in \R^{\SP}_{>0}$, we compute
	\begin{equation*}
		c\ge \norm{ \frac{z'+z}{2}}_{\ell^{\frac{2}{D}}}^{\frac{2}{D}} 
		= \norm{ \frac{z'_{\SP}+z_{\SP}}{2}}_{\ell^{\frac{2}{D}}}^{\frac{2}{D}} 
		> \frac{1}{2}\norm{ z'_{\SP} }_{\ell^{\frac{2}{D}}}^{\frac{2}{D}} 
		+ \frac{1}{2}\norm{ z_{\SP}}_{\ell^{\frac{2}{D}}}^{\frac{2}{D}} 
		= \frac{1}{2}\norm{ z' }_{\ell^{\frac{2}{D}}}^{\frac{2}{D}} 
		+ \frac{1}{2}\norm{ z}_{\ell^{\frac{2}{D}}}^{\frac{2}{D}}
		= c,
	\end{equation*}
	a contradiction. 
	
    The claim about the support follows from \cref{lemma:maximal_support_deep_convex} with $C:= \Lmin$.
\end{proof}

\subsubsection{Proof of \cref{lemma:normal_cone_deep}}\label{section:proof_of_lemma:normal_cone_deep}

\begin{proof}[Proof of \cref{lemma:normal_cone_deep}]
	The proof is similar to the proof of \cref{lemma:normal_cone_shallow}.
\end{proof}

\subsubsection{Proof of \cref{lemma:two_minimizers_deep}} \label{section:proof_of_lemma:two_minimizers_deep}

\begin{proof}[Proof of \cref{lemma:two_minimizers_deep}]
	By \cref{basic_properties_Lmin}, the set $\Lmin$ is compact.
    Since $\GD(\abs{\cdot})$ is continuous, the set in \eqref{eqn1020} is non-empty.
	Using that $\norm{\cdot}_{\ell^1}$ is constant on $\Lmin$ at $(a)$, we obtain
	\begin{equation*}
		\begin{split}
			\argmin_{x\in \Lmin} \GD(\abs{x})
			&= \argmin_{x\in \Lmin}\Big[ \sum_{i=1}^d \alpha \Big( \frac{\abs{x_i}}{\alpha} - \frac{D}{2} \Big( \frac{\abs{x_i}}{\alpha}\Big)^{\frac{2}{D}}\Big)\Big]\\
			&= \argmin_{x\in \Lmin} \Big[ \norm{x}_{\ell^1} -\frac{D}{2}\alpha^{1-\frac{2}{D}} \sum_{i=1}^d  \abs{x_i}^{\frac{2}{D}}\Big]\\
			&\overeq{(a)} \argmin_{x\in \Lmin} \Big[ -\sum_{i=1}^d  \abs{x_i}^{\frac{2}{D}}\Big]\\
			&= \argmax_{x\in \Lmin} \Big[\sum_{i=1}^d  \abs{x_i}^{\frac{2}{D}}\Big]\\
			&= \argmax_{x\in \Lmin} \norm{x}_{\ell^{\frac{2}{D}}}.
		\end{split}
	\end{equation*}
    The claim now follows by definition of $\ga$.	
\end{proof}

%% file: parts/basic_properties_Dtwo.tex
\subsection{Basic properties of $\arsinh$ and $\Ha$}

\subsubsection{Proof of Lemma \ref{lemma:basic_properties_arsinh}}\label{section:basic_properties_arsinh}

\begin{proof}[Proof of \cref{lemma:basic_properties_arsinh}]
	\emph{(i)} We have
	\begin{equation*}
		\arsinh \left(\frac{t}{2} \right) 
		= \log\Big( \frac{t}{2} + \sqrt{\frac{t^2}{4}+1} \Big) = \log(t) + \Delta(t),
	\end{equation*}
	where
	\begin{equation*}
		\Delta(t) = \log\Big( \frac{1}{2} \big( 1 + \sqrt{1 + \frac{4}{t^2}} \big)\Big).
	\end{equation*}
	Using the concavity of the square root and of the logarithm, 
	i.e., 
	$\sqrt{1+\eps} \le 1 + \frac{\eps}{2}$ and $\log(1+\eps) \le \eps$, and the monotonicity of the logarithm, we infer that
	\begin{equation*}
		\Delta(t) \le \log\Big( \frac{1}{2}\big( 2 + \frac{2}{t^2}\big)\Big) \le \frac{1}{t^2}.
	\end{equation*}
	
	\emph{(ii)} Switching the roles of $s$ and $t$, if necessary, or their signs, we may assume that $0<s<t$. In this case we have 
	$\sqrt{1+\frac{1}{t^2}} \le \sqrt{1+\frac{1}{s^2}}$ and thus
	\begin{equation*}
		\arsinh(t) -\arsinh(s)= \log\big( t +\sqrt{t^2+1} \big) - \log\big( s+\sqrt{s^2+1} \big)
		= \log\Big( \frac{t\sqrt{1+\frac{1}{t^2}}}{s\sqrt{1+\frac{1}{s^2}}} \Big)\le \log\Big(\frac{t}{s}\Big).
	\end{equation*}
	
	\emph{(iii)} The map is smooth and its first and second derivatives are
	\begin{equation*}
		\frac{\dif}{\dif t} \big( t\arsinh(t)\big) 
		= \arsinh(t) + \frac{t}{\sqrt{t^2+1}}
	\end{equation*}
	and
	\begin{equation*}
		\frac{\dif^2}{\dif t^2} \big( t\arsinh(t)\big) 
		= \frac{1}{\sqrt{t^2+1}} + \frac{\sqrt{t^2+1} 
		- \frac{t^2}{\sqrt{t^2+1}}}{t^2+1}
		= \frac{1}{\sqrt{t^2+1}}+\frac{1}{(t^2+1)^{\frac{3}{2}}}>0.
	\end{equation*}
\end{proof}

\subsubsection{Proof of Lemma \ref{lemma:strong_convexity}}
\label{section:proof_strong_convexity}

\begin{proof}[Proof of Lemma \ref{lemma:strong_convexity}]
	We trace the steps in \cite[Lemma 4]{ghai2019exponentiated}.

	We have 
	\begin{equation*}
		\frac{\partial^2}{\partial x_i^2} \Ha(x) = \frac{1}{\sqrt{x_i^2 + (2\alpha)^2}}.
	\end{equation*}

	Then, by the Cauchy-Schwarz inequality, we have
	\begin{equation*}
		\begin{split}
			&\langle \nabla^2 \Ha(x)n,n\rangle 
			= \sum_{i \in \supp (n)} \frac{n_i^2}{\sqrt{x_i^2 + (2\alpha)^2}}
			= \sum_{i \in \supp (n)} \frac{n_i^2}{\sqrt{x_i^2 + (2\alpha)^2}} 
			\cdot \frac{\sum_{i \in \supp (n)} \sqrt{x_i^2 + (2\alpha)^2}}{\sum_{i \in \supp (n)} \sqrt{x_i^2 + (2\alpha)^2}}\\
			&\ge \frac{1}{\sum_{i \in \supp (n)} \sqrt{x_i^2 + (2\alpha)^2}} 
			\left( \sum_{i=1}^d \frac{\abs{n_i}}{\sqrt[4]{x_i^2 + (2\alpha)^2}}\sqrt[4]{x_i^2 + (2\alpha)^2}\right)^2
			= \frac{\norm{n}_{\ell^1}^2}{\sum_{i \in \supp (n)} \sqrt{x_i^2 + (2\alpha)^2}}.
		\end{split}
	\end{equation*}
	The claim now follows from 
	$\sum_{i \in \supp (n)} \sqrt{x_i^2 + (2\alpha)^2} \le \norm{x}_{\ell^1} + 2\alpha \vert \supp (n) \vert $.
\end{proof}

%% file: parts/basic_properties_Dgreatertwo.tex
\subsection{Basic properties of $\hh$, $q_D$, and $\QQ$} \label{sec:basic_properties_Dge3}

Before we start with the proofs,
we collect the following facts about the derivatives of the function $\hh$.
A direct computation shows
that the first and second derivatives of $\hh$ are given as follows.
\begin{lemma}\label[lemma]{lemma:basic_properties_q_h_g_1}
	Let $D\in \N$ with $D\ge 3$ and $\gamma:=\frac{D-2}{D}$.
	We have for all $z\in (-1,1)$ that
	\begin{equation} \label{eqn615b}
		\hh'(z) = \frac{1}{\gamma} \Big( (1-z)^{-\frac{1}{\gamma}-1} + (1+z)^{-\frac{1}{\gamma}-1}\Big)
	\end{equation}
	and 
	\begin{equation} \label{eqn615a}
		\hh''(z) = \frac{1}{\gamma}\Big(\frac{1}{\gamma}+1\Big) \Big( (1-z)^{-\frac{1}{\gamma}-2} - (1+z)^{-\frac{1}{\gamma}-2}\Big).
	\end{equation}
\end{lemma}

Using \cref{lemma:WindHansen},
we can establish a bound for the asymptotic behavior of the inverse function of $\hh$,
which will also be useful in our proofs.
\begin{lemma} \label[lemma]{lemma:bounds_on_h_D_2}
	Let $D\in \N$ with $D\ge 3$ and $\gamma:=\frac{D-2}{D}$.
	Then,
	for all $u > 0$, we have
	\begin{equation*} 
		\frac{\gamma}{1+(u+1)^{1+\gamma}}
		\le \big(\hh^{-1}\big)'(u) 
		\le \gamma \cdot \min\Big\{\frac{1}{2},\frac{1}{u^{1+\gamma}}\Big\}.
	\end{equation*}
\end{lemma}
\begin{proof}
	It follows from \cref{eqn615b} that
	\begin{equation} \label{eqn6002}
		\frac{1}{\gamma}(1-z)^{-\frac{1}{\gamma}-1}
		\le \hh'(z) 
		\le \frac{1}{\gamma} \Big( (1-z)^{-\frac{1}{\gamma}-1} + 1\Big).
	\end{equation}
	We have
	\begin{equation*}
		\big(\hh^{-1}\big)'(u) 
		= \frac{1}{\hh'\big( \hh^{-1}(u)\big)}.
	\end{equation*}
	Using \cref{lemma:WindHansen} 
	and that the map $\hh'$ is increasing on $(0,\infty)$
	at $(a)$,
	and \cref{eqn6002} at $(b)$, we infer that
	\begin{equation*}
		\big(\hh^{-1}\big)'(u) 
		\overgeq{(a)}\frac{1}{\hh'\big( 1- (u+1)^{-\gamma}\big)}
		\overgeq{(b)} \frac{ \gamma}{1 + \big( u+1\big)^{-\gamma  \big( -\frac{1}{\gamma}-1\big)}}
		=\frac{\gamma}{1+(u+1)^{1+\gamma}}.
	\end{equation*}
	Analogously, we deduce that
	\begin{equation*}
		\big(\hh^{-1}\big)'(u) 
		\le \frac{1}{\hh'\big( 1-u^{-\gamma}\big)}
		\le \frac{\gamma}{u^{1+\gamma}}.
	\end{equation*}
	Since $\hh'$ is increasing, we also have $\hh'(u)\ge \hh'(0)= \frac{2}{\gamma}$ for $u\ge 0$.
\end{proof}

\subsubsection{Proof of \cref{lemma:basic_properties_q_h_g_2}}
\begin{proof}[Proof of \cref{lemma:basic_properties_q_h_g_2}]
	All the functions are smooth as compositions or inverses of smooth functions. The symmetry properties can be checked directly from the definitions.
	
	By \cref{lemma:basic_properties_q_h_g_1}, 
	we have $\hh'>0$ on $(-1,1)$ and $\hh''>0$ on $(0,1)$. 
	Thus, $\hh$ is convex.
	Furthermore, since $\hh$ is convex and increasing, we have for $u\in \R$ and $v>0$ that
	\begin{equation*}
		(\hh^{-1})'(u) = \frac{1}{\hh'\circ \hh^{-1}(u)} >0,
		\quad \text{and} \quad
		(\hh^{-1})''(v) = -\frac{\hh''\circ \hh^{-1}(v)}{(\hh'\circ \hh^{-1}(v))^3}<0.
	\end{equation*}
	
	We have
	\begin{equation*}
		\qq'(u) = \hh^{-1}(u) \qquad \text{and} \qquad
		\qq''(u) = \frac{1}{\hh'\circ \hh^{-1}(u)}
	\end{equation*}
	and $\hh'>0$.
	Hence, it follows that $\qq''>0$
	and thus $\qq$ is convex.
	Furthermore, $\qq' = \hh^{-1}>0$ on $(0,\infty)$
	which implies that $\qq$ is increasing.
	This completes the proof of the lemma.
\end{proof}

\subsubsection{Proof of \cref{lemma:almost_convex}}
\begin{proof}[Proof of \cref{lemma:almost_convex}]
	Recall the differentiation rules
	\begin{equation*}
		(fg)'' = f''g + 2 f'g' + f g'',\quad
		\text{and} \quad
		(g^{-1})'' = \Big( \frac{1}{g'\circ g^{-1}}\Big)' = -\frac{g''\circ g^{-1}}{(g'\circ g^{-1})^3}.
	\end{equation*}
	Therefore, with $f (t):=t $ and $g (t):=\hh^{-1}(t)$, we have
	\begin{equation*}
		\frac{\dif^2}{\dif t^2} \left( t \hh^{-1}(t) \right)
		= \frac{2}{\hh'(\hh^{-1}(t))} - \frac{ t\hh''(\hh^{-1}(t))}{[\hh^{'}(\hh^{-1}(t))]^3}
	\end{equation*}
	for all $t\in (0,\infty)$. Let $u\in (0,1)$ such that $\hh(u)=t$. We obtain
	\begin{equation} \label{eqn621}
		\frac{\dif^2}{\dif t^2} \left( t \hh^{-1}(t) \right)
		= \frac{2 (\hh'(u))^2 - \hh(u)\hh''(u)}{(\hh'(u))^3}.
	\end{equation}
	Let $\eta:= \frac{1}{\gamma}$. 
	Using \cref{lemma:basic_properties_q_h_g_1},
	see \cref{eqn615b} and \cref{eqn615a},
	we obtain
	\begin{equation} \label{eqn2100}
		2 (\hh'(u))^2 
		= 2\eta^2
		\Big( (1-u)^{-\eta-1} + (1+u)^{-\eta-1}\Big)^2
		\ge 2\eta^2
		(1-u)^{-2\eta-2}
	\end{equation}
	and 
	\begin{equation} \label{eqn2101}
		\begin{split}
			&\hh(u)\hh''(u) \\
			=& \eta(\eta+1)
			\Big( (1-u)^{-\eta} - (1+u)^{-\eta}\Big)
			\Big( (1-u)^{-\eta-2} - (1+u)^{-\eta-2}\Big)\\
			\le& \eta( \eta+1)
			(1-u)^{-\eta} 
			(1-u)^{-\eta-2}\\
			=& \eta(\eta+1)
			(1-u)^{-2\eta-2}.	
		\end{split}
	\end{equation}
	Since $2\eta^2>\eta(\eta+1)$
	due to $\eta = \frac{D}{D-2} >1 $, 
	the inequalities \eqref{eqn2100} and \eqref{eqn2101} imply 
	that $2 (\hh'(u))^2  >\hh(u)\hh''(u)$. 
	Furthermore, $\hh'>0$ by \cref{lemma:basic_properties_q_h_g_1}, 
	see \cref{eqn615b}. 
	Inserting all into \eqref{eqn621}, we deduce the claim.
\end{proof}

\subsubsection{Proof of \cref{lemma:basic_inequalities_deep_unique}}

\begin{proof}[Proof of \cref{lemma:basic_inequalities_deep_unique}]
Note that statement (i),
which is the inequality
\begin{equation}
		\hh'(z) \le \frac{2}{\gamma}(1-z)^{-\frac{1}{\gamma}-1}
\end{equation}
for all $z \in [0,1)$, follows directly from \cref{lemma:basic_properties_q_h_g_1},
see \cref{eqn615b}.
It remains to prove statement (ii).
	Assume that $u\ge v$. Using that $\hh^{-1}$ is increasing at $(a)$, the mean value theorem with some $\xi \in (v,u)$ at $(b)$, and \cref{lemma:bounds_on_h_D_2} at $(c)$, we obtain
	\begin{equation*}
		\begin{split}
			\abs{\hh^{-1}(u)-\hh^{-1}(v)} 
			&\overeq{(a)} \hh^{-1}(u)-\hh^{-1}(v)
			\overeq{(b)}\big( \hh^{-1}\big)'(\xi)   (u-v)
			\overleq{(c)} \gamma  \xi^{-1-\gamma}  (u-v)\\
			&\le \gamma  v^{-1-\gamma}  (u-v) 
			= \frac{\gamma}{\big( \min\{u,v\}\big)^{1+\gamma}}\abs{u-v}.
		\end{split}
	\end{equation*}
	The case $u\le v$ is treated analogously.
\end{proof}

\subsubsection{Proof of \cref{lemma:basic_inequalities_deep_non-unique}}
\label{appendix:basic_inequalities_deep_non-unique}
\begin{proof}[Proof of \cref{lemma:basic_inequalities_deep_non-unique}]
	\noindent 
	\textbf{Proof of (i)}
	We observe that
	\begin{equation*}
		\hh^{-1}(u) - \gD'(u)
		=
		\hh^{-1} (u) - 1 + u^{\frac{2}{D}-1}
		=
		\hh^{-1} (u) - 1 + u^{-\gamma}.
	\end{equation*}
	From \cref{lemma:WindHansen}
	we infer 
	\begin{equation*}
		0\le h_D^{-1}(u)-g_D'(u) \le u^{-\gamma}-(u+1)^{-\gamma}.
	\end{equation*}
	Using the mean value theorem, we deduce the claim.\\

	\noindent
   \textbf{Proof of (ii)}
	Let $u\ge 1$. 
	By the mean value theorem, 
	it holds for some $\xi \in (1,1+\frac{1}{u})\subset (1,2)$ that
	\begin{equation*}
		1+ (u+1)^{1+\gamma} 
		= 1+u^{1+\gamma}\Big( 1+ \frac{1}{u}\Big)^{1+\gamma} 
		= 1+u^{1+\gamma}\Big( 1 + (1+\gamma)  \frac{\xi^{\gamma}}{u}\Big) 
		= u^{1+\gamma} \big( 1+ \delta(u)\big),
	\end{equation*}
	where 
	\begin{equation*}
		\delta(u)
		:= 
		 \frac{1}{u^{1+\gamma}} +(1+\gamma)  \frac{\xi^{\gamma}}{u}  
		\overleq{\gamma \le 1} 
		 \frac{1}{u^{2}} +2 \frac{\xi}{u}  
		 \overleq{\xi \le 2}
		\frac{5}{u}.
	\end{equation*}
	Thus, it follows from \cref{lemma:bounds_on_h_D_2} that
	\begin{equation} \label{eqn6003}
		\frac{\gamma}
		{\left(1+\frac{5}{u}\right)u^{1+\gamma}}
		\le \big(\hh^{-1}\big)'(u) 
		\le \frac{\gamma}{u^{1+\gamma}}.
	\end{equation}
	This proves inequality \eqref{equ:DominikInternfinal1}.
	To show inequality \eqref{eqn:bounds_on_h_D_5},
	we infer from \cref{eqn6003} that 
	\begin{equation*}
		0 \le 
		\frac{\gamma}{u^{1+\gamma}} - \big(\hh^{-1}\big)'(u) 
		\le \frac{\gamma}{u^{1+\gamma}} \Big( 1 - \frac{1}{1+\frac{5}{u}}\Big) 
		\le \frac{5\gamma}{u^{2+\gamma}}.
	\end{equation*}

	\noindent
	\textbf{Proof of (iii)}
	By standard differentiation rules, we have
	\begin{equation}\label{eqn1400-2}
		(\hh^{-1})'' 
		= \Big( \frac{1}{\hh'\circ \hh^{-1}}\Big)' 
		= -\frac{\hh''\circ \hh^{-1}}{(\hh'\circ \hh^{-1})^3}
		= - \big(\hh''\circ \hh^{-1}\big) \cdot \left( \left( \hh^{-1} \right)' \right)^3.
	\end{equation}
	Using \cref{lemma:bounds_on_h_D_2}, we infer that
	\begin{equation}\label{eqn1401-2}
		\big(\hh^{-1}\big)^3 (u) \le 
		\frac{\gamma^3}{u^{3+3\gamma}}.
	\end{equation}
	By \cref{lemma:basic_properties_q_h_g_1}, we have for all $z\in (-1,1)$ that
	\begin{equation*}
		\hh''(z)\le \frac{1}{\gamma}\Big(\frac{1}{\gamma}+1\Big) 
		(1-z)^{-\frac{1}{\gamma}-2}.
	\end{equation*}
	Hence, using \cref{lemma:WindHansen} at $(a)$, we obtain
	\begin{align}
			\hh''\circ \hh^{-1}(u) 
			&\le \frac{1}{\gamma}\Big(\frac{1}{\gamma}+1\Big)  (1-\hh^{-1}(u))^{-\frac{1}{\gamma}-2}
			\overleq{(a)} \frac{1}{\gamma}
			\Big(\frac{1}{\gamma}+1\Big)  \left(1-\left(1-\left( u+1 \right)^{-\gamma} \right) \right)^{-\frac{1}{\gamma}-2} 
			\nonumber \\
			&= \frac{1}{\gamma}\Big(\frac{1}{\gamma}+1\Big) \cdot (u+1)^{1+2\gamma}
			= \frac{1}{\gamma}\Big(\frac{1}{\gamma}+1\Big) 
			\Big(1+\frac{1}{u}\Big)^{1+2\gamma}
			u^{1+2\gamma}.\label{eqn1402-2}
	\end{align}
	Inserting \eqref{eqn1401-2} and \eqref{eqn1402-2} into \eqref{eqn1400-2}, 
	it follows that
	\begin{align*}
		(\hh^{-1})'' (u)
		\ge 
		&
		-\left( \frac{1}{\gamma} +1 \right)
		\left( 1+\frac{1}{u} \right)^{1+2 \gamma}
		\gamma^2 u^{-2-\gamma}\\
		\overgeq{\gamma \le 1} 
		&-2
		\left( 1+\frac{1}{u} \right)^{1+2 \gamma}
		 u^{-2-\gamma}\\
		\overgeq{u\ge 1}
		&-16 u^{-2-\gamma}.
	\end{align*}
	Since $\hh^{-1}$ is concave, we also have $(\hh^{-1})''\le 0$.
	Thus, we have shown that for $u \ge 1$ it holds that
	\begin{equation*}
	0
	\le 
	(\hh^{-1})'' (u)
	\le
	16 u^{-2-\gamma}.
	\end{equation*}
	This completes the proof of statement (iii).
\end{proof}

\subsubsection{Proof of \cref{lemma:strong_convexity_deep}}\label{appendix:strong_convexity_deep}
\begin{proof}[Proof of \cref{lemma:strong_convexity_deep}]
	For all $i\in [d]$ define
	\begin{equation*}
		\zeta_i := \Big[\big(\hh^{-1}\big)'\Big(\frac{\abs{x_i}}{\alpha}\Big)\Big]^{-1}.
	\end{equation*}
	Using that $\big(\hh^{-1}\big)'$ is even at $(a)$ and the Cauchy-Schwarz inequality at $(b)$, we obtain
	\begin{equation*}\begin{split}
			\langle n, \nabla^2 \QQ(x)n\rangle
			&=\sum_{i\in \supp(n)}\frac{ n_i^2 }{\alpha}\big(\hh^{-1}\big)'\Big(\frac{x_i}{\alpha}\Big)
			\overeq{(a)}\sum_{i\in \supp(n)}  \frac{n_i^2}{\alpha}\big(\hh^{-1}\big)'\Big(\frac{\abs{x_i}}{\alpha}\Big)\\
			&= \frac{1}{\alpha} \sum_{i\in \supp(n)}   \frac{n_i^2}{\zeta_i}
			= \sum_{i\in \supp(n)} \frac{n_i^2}{ \alpha \zeta_i} \cdot \frac{ \sum_{i\in \supp(n)} \zeta_i}{\sum_{i\in \supp(n)} \zeta_i}\\
			&\overgeq{(b)} \frac{1}{ \alpha \sum_{i\in \supp(n)} \zeta_i} \cdot \Big(\sum_{i\in \supp(n)} \frac{\abs{n_i}}{\sqrt{\zeta_i}} \sqrt{\zeta_i}\Big)^2\\
			&=  \frac{ \norm{n}_{\ell^1}^2  }{\alpha \sum_{i\in \supp(n)} \zeta_i}.
	\end{split}\end{equation*}
	
	We first prove \eqref{eqn:strong_convexity_deep_1}. Using \cref{lemma:bounds_on_h_D_2} at $(a)$, and $(a+b)^{1+\gamma} \le 2^{\gamma}\big(a^{1+\gamma}+b^{1+\gamma}\big)$ together with $2^{\gamma}\le 2$ at $(b)$, we infer that
	\begin{equation*}
		\zeta_i = \Big[\big(\hh^{-1}\big)'\Big(\frac{\abs{x_i}}{\alpha}\Big)\Big]^{-1}
		\overleq{(a)} \frac{1}{\gamma} \Big[1 + \Big( \frac{\abs{x_i}}{\alpha}+1\Big)^{1+ \gamma}\Big]
		\overleq{(b)} \frac{1}{\gamma}\Big[3 + 2\Big(\frac{\abs{x_i}}{\alpha}\Big)^{1+\gamma} \Big].
	\end{equation*}
	Hence, we obtain that
	\begin{equation*}
		\alpha \sum_{i\in \supp(n)} \zeta_i 
		\le \frac{1}{\gamma \alpha^{\gamma} } 
		\big( 3\abs{\supp(n)} \alpha^{1+\gamma} + 2\norm{x}_{\ell^{1+\gamma}}^{1+\gamma}\big) .
	\end{equation*}
	Inserting this into the above inequality yields that
	\begin{equation*}
		\langle n, \nabla^2 \QQ(x)n\rangle
		\ge 
		\frac{\gamma \norm{n}_{\ell^1}^2 \alpha^{\gamma}}{3\abs{\supp(n)} \alpha^{1+\gamma} + 2\norm{x}_{\ell^{1+\gamma}}^{1+\gamma}}.
	\end{equation*}
	This proves inequality \eqref{eqn:strong_convexity_deep_1}.
	It remains to prove \eqref{eqn:strong_convexity_deep_2}.
	Recall that we assume that $\alpha<\abs{x_i}$ for all $i\in \SP$. 
	Then inequality \eqref{equ:DominikInternfinal1} implies that
	\begin{equation*}\begin{split}
			\zeta_i =
			\Big[\big(\hh^{-1}\big)'\Big(\frac{\abs{x_i}}{\alpha}\Big)\Big]^{-1}
			\le \frac{ \abs{x_i}^{1+\gamma}}{ \gamma \alpha^{1+\gamma}}
			\left( 1+  \frac{5 \alpha}{\min_{i\in \SP}\abs{x_i}} \right)
	\end{split}\end{equation*}
	for all $i\in \SP$. This implies that
	\begin{equation*}
		\langle n, \nabla^2 \QQ(x)n\rangle
		\ge
		 \frac{\norm{n}_{\ell^1}^2 \gamma \alpha^{\gamma}}
		 {\norm{x}_{\ell^{1+{\gamma}}}^{1+\gamma} \left( 1+\frac{5\alpha}{\min_{i\in \SP}\abs{x_i}}  \right) },
	\end{equation*}
	which completes the proof.
\end{proof}

\subsubsection{Proof of \cref{lemma:technical_1}}

\begin{lemma}\label[lemma]{lemma:technical_1}
	For all $0<\gamma<1$ it holds that
	\begin{equation*}
		\Big(1-\frac{\gamma}{4}\Big)^{-\frac{\gamma+1}{\gamma}} \le 2.
	\end{equation*}
\end{lemma}
\begin{proof}
	For $t>0$ let $f(t):= t^{\frac{\gamma}{1+\gamma}}$. 
	Then we have 
	$f'(t) = t^{-\frac{1}{\gamma+1}} \gamma /(\gamma+1)$ 
	and thus there exists $\xi\in (\frac{1}{2},1)$
	such that
	\begin{equation*}
		1-\Big(\frac{1}{2}\Big)^{\frac{\gamma}{\gamma+1}}  
		= f \left(1\right)
		-f\left(\frac{1}{2}\right) 
		= f'(\xi) \cdot \frac{1}{2}
		=
		\frac{\gamma}{2(\gamma+1) \xi^{\frac{1}{\gamma+1}}}
		\ge \frac{\gamma}{ 2 \left( \gamma+1 \right)} 
		\ge \frac{\gamma}{4}.
	\end{equation*}
	Rearranging terms, we obtain
	\begin{equation*}
		\Big(1-\frac{\gamma}{4}\Big)^{\frac{\gamma+1}{\gamma}} \ge \frac{1}{2},
	\end{equation*}
	which implies the claim.
\end{proof}